\documentclass[final,12pt]{alt2024} 

\usepackage[utf8]{inputenc} 
\usepackage[T1]{fontenc}    
\usepackage{hyperref}       
\usepackage{url}            
\usepackage{booktabs}       
\usepackage{amsfonts}       
\usepackage{nicefrac}       
\usepackage{microtype}      
\usepackage{xcolor}         

\usepackage{dsfont}
\usepackage{enumitem}
\usepackage[toc,page,header]{appendix}
\usepackage{lipsum,wrapfig}

\usepackage[english]{babel}
\usepackage{amsmath, amssymb}
\usepackage{graphicx}
\usepackage{latexsym}
\usepackage{graphicx}
\usepackage{cases}
\usepackage[mathscr]{euscript}
\usepackage{xcolor}
\usepackage{colortbl}
\usepackage{thmtools}
\usepackage{thm-restate}

\usepackage{mathtools}
\usepackage{subcaption}

\usepackage{multirow}
\usepackage{wrapfig}

\usepackage{pifont}

\newtheorem{assumption}{Assumption}

\let\vec\undefined
\usepackage{MnSymbol}
\DeclareMathAlphabet\mathbb{U}{msb}{m}{n}
\usepackage{xpatch}

\definecolor{Gray}{gray}{0.85}
\newcolumntype{g}{>{\columncolor{Gray}}c}

\hypersetup{
  colorlinks   = true, 
  urlcolor     = blue, 
  linkcolor    = blue, 
  citecolor   = blue 
}

\def\Rset{\mathbb{R}}

\DeclareMathOperator*{\E}{\mathbb{E}}

\DeclareMathOperator*{\argmax}{\rm argmax}

\DeclareMathOperator{\sign}{sign}

\newcommand{\nrm}[1]{{\left\vert\kern-0.25ex\left\vert\kern-0.25ex\left\vert #1 
    \right\vert\kern-0.25ex\right\vert\kern-0.25ex\right\vert}}

\DeclarePairedDelimiter{\abs}{\lvert}{\rvert} 
\DeclarePairedDelimiter{\bracket}{[}{]}
\DeclarePairedDelimiter{\curl}{\{}{\}}
\DeclarePairedDelimiter{\paren}{(}{)}
\DeclarePairedDelimiter{\norm}{\|}{\|}

\newcommand{\sA}{{\mathscr A}}

\newcommand{\sC}{{\mathscr C}}
\newcommand{\sD}{{\mathscr D}}
\newcommand{\sE}{{\mathscr E}}
\newcommand{\sF}{{\mathscr F}}
\newcommand{\sG}{{\mathscr G}}
\newcommand{\sH}{{\mathscr H}}

\newcommand{\sM}{{\mathscr M}}

\newcommand{\sR}{{\mathscr R}}

\newcommand{\sX}{{\mathscr X}}
\newcommand{\sY}{{\mathscr Y}}

\newcommand{\bb}{{\mathbf b}}

\newcommand{\bw}{{\mathbf w}}

\newcommand{\sfB}{{\mathsf B}}

\newcommand{\sfL}{{\mathsf L}}

\newcommand{\wt}{\widetilde}

\newcommand{\e}{\epsilon}

\newcommand{\ignore}[1]{}

\newcommand{\hh}{{\sf h}}

\newcommand{\1}{\mathds{1}}
\newcommand{\labs}{{\sfL_{\rm{abs}}}}
\newcommand{\labsc}{{\wt \sfL_{\rm{abs}}}}

\title[Predictor-Rejector Multi-Class Abstention]
      {Predictor-Rejector Multi-Class Abstention:\\
      Theoretical Analysis and Algorithms}

\usepackage{times}

\altauthor{%
 \Name{Anqi Mao} \Email{aqmao@cims.nyu.edu}\\
 \addr Courant Institute of Mathematical Sciences, New York%
 \AND
 \Name{Mehryar Mohri} \Email{mohri@google.com}\\
 \addr Google Research and Courant Institute of Mathematical Sciences, New York%
 \AND
 \Name{Yutao Zhong} \Email{yutao@cims.nyu.edu}\\
 \addr Courant Institute of Mathematical Sciences, New York%
}

\begin{document}

\maketitle

\begin{abstract}

We study the key framework of learning with abstention in the
multi-class classification setting. In this setting, the learner can
choose to abstain from making a prediction with some pre-defined
cost. We present a series of new theoretical and algorithmic results
for this learning problem in the predictor-rejector framework.
We introduce several new families of surrogate losses for which we
prove strong non-asymptotic and hypothesis set-specific consistency
guarantees, thereby resolving positively two existing open questions. These guarantees provide upper bounds on the estimation error
of the abstention loss function in terms of that of the surrogate
loss. We analyze both a single-stage setting where the predictor and
rejector are learned simultaneously and a two-stage setting crucial in
applications, where the predictor is learned in a first stage using a
standard surrogate loss such as cross-entropy.
These guarantees suggest new multi-class abstention algorithms based
on minimizing these surrogate losses. We also report the results of
extensive experiments comparing these algorithms to the current
state-of-the-art algorithms on CIFAR-10, CIFAR-100 and SVHN
datasets. Our results demonstrate empirically the benefit of our new
surrogate losses and show the remarkable performance of our broadly
applicable two-stage abstention algorithm.

\end{abstract}


\section{Introduction}
\label{sec:intro}

The problem of learning with abstention has become increasingly
crucial in various applications. In natural language generation or
question-answering, it is not always possible to provide accurate or
factual responses. Therefore, it is of utmost importance to develop
the ability to abstain from generating a response in such cases to
prevent the occurrence of misleading or incorrect information, often
referred to as \emph{hallucinations} \citep{WeiEtAl2022,
  Filippova2020, maynez2020}.  Abstention plays a critical role in
several other application areas. For instance, in the field of
autonomous vehicle control, an incorrect prediction can pose a
significant threat to human lives. Similarly, in decision-making
systems, incorrect choices can have severe ethical implications.
These scenarios highlight the necessity of incorporating abstention
mechanisms to ensure the safety, reliability, and ethical soundness of
the learning process.

We can distinguish several broad methods for learning with abstention
in the literature: \emph{confidence-based methods}, which consist of
abstaining when the score returned by a pre-trained model falls below
some threshold
\citep{Chow1957,chow1970optimum,bartlett2008classification,
  yuan2010classification,WegkampYuan2011,ramaswamy2018consistent,NiCHS19}; \emph{selective
classification}, which analyzes a set-up with a \emph{predictor} and a
\emph{selector} and defines a selection risk or loss normalized by the expected selection or coverage
\citep{el2010foundations,wiener2011agnostic,el2012active,
  wiener2015agnostic,geifman2017selective,geifman2019selectivenet}; a
\emph{predictor-rejector formulation}, which is based on learning both
a \emph{predictor} and a \emph{rejector}, each from a different family
of functions, and that takes into account explicitly the 
abstention cost $c$
\citep{CortesDeSalvoMohri2016,CortesDeSalvoMohri2016bis,CortesDeSalvoMohri2023,cheng2023regression,MohriAndorChoiCollinsMaoZhong2024learning,li2024no}; and a more
recent \emph{score-based formulation} that consists of augmenting the
multi-class categories with a rejection label and of abstaining when
the score assigned to the rejection label is the highest
\citep{mozannar2020consistent,caogeneralizing,MaoMohriZhong2024score}.  Another problem closely related to
abstention is that of \emph{deferring} to an alternative model, or
even to a human in some instances. This can also be considered as a
special case of the general abstention scenario and tackled
in a similar way
\citep{madras2018predict,raghu2019algorithmic,mozannar2020consistent,
  okati2021differentiable,wilder2021learning,verma2022calibrated,
  narasimhanpost,verma2023learning,mao2023two,cao2023defense,MaoMohriZhong2024deferral,chen2024learning,mao2024regression}.

We will be particularly interested in the predictor-rejector
formulation, which explicitly models the cost of abstention. The
selective classification of \citet{el2010foundations} is also
interesting, but it does not explicitly factor in the cost $c$ and is
based on a distinct objective. Confidence-based methods are also very
natural and straightforward, but they may fail when the predictor is
not calibrated, a property that often does not hold. Additionally,
they have been shown to be suboptimal when the predictor differs from
the Bayes classifier \citep{CortesDeSalvoMohri2016}. The score-based
formulation \citep{mozannar2020consistent} admits very common
properties and also explicitly takes into account the rejection cost
$c$.  We will compare the predictor-rejector formulation with the
score-based one. We will show via an example that the
predictor-rejector is more natural in some instances and will also
compare the two formulations in our experiments. We further elaborate on the difference
between the two formulations in Appendix~\ref{app:difference}.

How should the problem of multi-class classification with abstention
be formulated and when is it appropriate to abstain?
The extension of the results of \citet{CortesDeSalvoMohri2016} to
\emph{multi-class classification} was found to be very challenging by
\citet{NiCHS19}. In fact, these authors left the following as an open question: \emph{can we define Bayes-consistent surrogate losses for the
predictor-rejector abstention formulation in the multi-class setting?}
This paper deals precisely with this topic: we present a series of new
theoretical, algorithmic, and empirical results for multi-class
learning with abstention in predictor-rejector formulation and, in
particular, resolve this open question in a strongly positive way.

For the score-based formulation, a surrogate loss function based on
cross-entropy was introduced by \citet{mozannar2020consistent}, which
was proven to be Bayes-consistent. Building upon this work,
\citet{caogeneralizing} presented a more comprehensive collection of
Bayes-consistent surrogate losses for the score-based
formulation. These surrogate losses can be constructed using any
consistent loss function for the standard multi-class classification
problem. More recently, \citet{MaoMohriZhong2024score} gave an extensive analysis of surrogate losses for the score-based formulation supported by \emph{$\sH$-consistency bounds}. In a recent study, \citet{pmlr-v206-mozannar23a}
demonstrated that existing score-based surrogate losses for abstention
are not \emph{realizable consistent} with respect to the abstention
loss, as defined by \citet{long2013consistency}. Instead, they
proposed a novel surrogate loss that achieves realizable $(\sH,
\sR)$-consistency, provided that the sets of predictors $\sH$ and
rejectors $\sR$ are \emph{closed under scaling}.  However, the authors
expressed uncertainty regarding the Bayes-consistency of their
proposed surrogate losses and left open the question of
\emph{identifying abstention surrogate losses that are both consistent
and realizable $(\sH, \sR)$-consistent when $\sH$ and $\sR$ satisfy
the scaling closure property}.  We address this open question by
demonstrating that our newly proposed surrogate losses benefit from
both Bayes-consistency and realizable consistency. We give a more
comprehensive discussion of related work in
Appendix~\ref{app:related-work}.

We show in Section~\ref{sec:score-example} that in some instances
the optimal solution cannot be derived in the score-based
formulation, unless we resort to more complex scoring functions. In
contrast, the solution can be straightforwardly derived in the
predictor-rejector formulation.
In Section~\ref{sec:general}, we present and analyze a new family of
surrogate loss functions for multi-class abstention in the
predictor-rejector formulation, first in the \emph{single-stage
setting}, where the predictor $h$ and the rejector $r$ are selected
simultaneously, next in a \emph{two-stage setting}, where first the
predictor $h$ is chosen and fixed and subsequently the rejector $r$ is
determined. The two-stage setting is crucial in many applications
since the predictor $h$ is often already learned after a costly
training of several hours or days. Re-training to ensure a
simultaneous learning of $h$ and $r$ is then inconceivable due to its
prohibitive cost.

In the \emph{single-stage setting} (Section~\ref{sec:single-stage}),
we first give a negative result, ruling out abstention surrogate
losses that do not
verify a technical condition. Next, we present several
positive results for abstention surrogate
losses verifying that condition, for which we prove non-asymptotic \emph{$(\sH,\sR)$-consistency bounds} \citep{awasthi2022Hconsistency,awasthi2022multi}
that are stronger than Bayes-consistency.
Next, in Section~\ref{sec:two-stage}, we also prove \emph{$(\sH,
\sR)$-consistency bounds} for the \emph{two-stage setting}. Minimizing
these new surrogate losses directly result in new algorithms for
multi-class abstention.

In Section~\ref{sec:general-realizable}, we prove realizable
consistency guarantees for both single-stage and two-stage
predictor-rejector surrogate losses.\ignore{This property underscores
  the advantages of our predictor-rejector formulation and is
  supported by the empirical success of the newly proposed surrogate
  losses. We further discuss in detail the difference between the
  predictor-rejector formulation and the score-based formulation,
  which underscores our work's innovation and significant
  contribution.}
In Section~\ref{sec:experiments}, we empirically show that our two-stage predictor-rejector surrogate loss
consistently outperforms state-of-the-art
scored-based surrogate losses\ignore{on CIFAR-10, CIFAR-100 and
SVHN datasets}, while our single-stage one achieves comparable results. Our main contributions are summarized below: 
\begin{itemize}
\itemsep-0.1em 
\item Counterexample for score-based abstention formulation.

\item Negative results for single-stage predictor-rejector surrogate losses.

\item New families of single-stage predictor-rejector surrogate
      losses for which we prove strong non-asymptotic and hypothesis
      set-specific consistency guarantees, thereby resolving
      positively an open question mentioned by \citet{NiCHS19}.

\item Two-stage predictor-rejector formulations and their
      $\sH$-consistency bounds guarantees.

\item Realizable consistency guarantees for both single-stage and
      two-stage surrogate losses, which resolve positively the recent
      open question posed by \citet{pmlr-v206-mozannar23a}.

\item Experiments on CIFAR-10, CIFAR-100 and SVHN
datasets empirically demonstrating the usefulness of our
      proposed surrogate losses.
 
\end{itemize}

\section{Preliminaries}
\label{sec:preliminaries}

We first introduce some preliminary concepts and definitions,
including the description of the predictor-rejector formulation and
background on $\sH$-consistency bounds.
We examine the standard multi-class classification scenario with an
input space $\sX$ and a set of $n \geq 2$ classes or labels $\sY =
\curl*{1, \ldots, n}$. We will denote by $\sD$ a distribution over
$\sX \times \sY$ and by $p(x, y) = \sD(Y = y \mid X = x)$ the
conditional probability of $Y = y$ given $X = x$. We will also adopt
the shorthand $p(x) = \paren*{p(x, 1), \ldots, p(x, n)}$ to denote the
vector of conditional probabilities, given $x \in \sX$. We study more
specifically the learning scenario of multi-class classification with
abstention, within the predictor-rejector formulation introduced by
\citet{CortesDeSalvoMohri2016bis}.

\textbf{Predictor-rejector formulation.}
In this formulation of the abstention problem, we have a hypothesis
set $\sH$ of prediction functions mapping from $\sX \times \sY$ to
$\Rset$, along with a family $\sR$ of abstention functions, or
\emph{rejectors}, which map from $\sX$ to $\Rset$. For an input $x \in
\sX$ and a hypothesis $h \in \sH$, the predicted label $\hh(x)$ is
defined as the one with the highest score, $\hh(x) = \argmax_{y\in
  \sY}h(x, y)$, using an arbitrary yet fixed deterministic strategy to
break ties. A rejector $r \in \sR$ is used to abstain from predicting
on input $x$ if $r(x)$ is non-positive, $r(x) \leq 0$, at a cost $c(x)
\in [0, 1]$. The \emph{predictor-rejector abstention loss} $\labs$ for
this formulation is thus defined for any $(h, r) \in \sH \times \sR$
and $(x, y) \in \sX \times \sY$ as
\begin{equation}
\label{eq:abs}
\labs(h, r, x, y)
= \1_{\hh(x) \neq y} \1_{r(x)> 0} + c(x) \1_{r(x) \leq 0}.
\end{equation}
When the learner does not abstain ($r(x) > 0$), it incurs the familiar
zero-one classification loss. Otherwise, it abstains ($r(x) \leq 0$)
at the expense of a cost $c(x)$. The learning problem consists of
using a finite sample of size $m$ drawn i.i.d.\ from $\sD$ to select a
predictor $h$ and rejector $r$ with small expected $\labs$ loss,
$\E_{(x, y) \sim \sD}[\labs(h, r, x, y)]$.
For simplification, in the following, we assume a
constant cost function $c \in (0, 1)$. This assumption is not
necessary however, and all $(\sH,
\sR)$-consistency bounds results in Sections~\ref{sec:single-stage} and \ref{sec:two-stage} extend straightforwardly to
general cost functions.

Optimizing the predictor-rejector abstention loss is intractable for
most hypothesis sets. Thus, learning algorithms in this context must
rely on a surrogate loss $\sfL$ for $\labs$. In the subsequent
sections, we study the consistency properties of several surrogate
losses. Given a surrogate loss function $\sfL$ in the
predictor-rejector framework, we denote by $\sE_{\sfL}(h, r) = \E_{(x,
  y)\sim \sD}\bracket*{\sfL(h, r, x, y)}$ the $\sfL$-expected loss of
a pair $(h, r) \in \sH \times \sR$ and by $\sE_{\sfL}^*(\sH, \sR) =
\inf_{h \in \sH, r \in \sR} \sE_{\sfL}(h, r)$ its infimum over $\sH
\times \sR$.

\textbf{$(\sH, \sR)$-consistency bounds.} We will prove $(\sH,
\sR)$-consistency bounds for several surrogate losses in the
predictor-rejector framework, which extend to the predictor-rejector
framework the $\sH$-consistency bounds of
\citet{awasthi2021calibration,awasthi2021finer,awasthi2022Hconsistency,awasthi2022multi,AwasthiMaoMohriZhong2023theoretically,awasthi2024dc,MaoMohriZhong2023cross,MaoMohriZhong2023ranking,MaoMohriZhong2023rankingabs,zheng2023revisiting,MaoMohriZhong2023characterization,MaoMohriZhong2023structured,mao2024top,mao2024h}.
These are inequalities upper-bounding the predictor-rejector
abstention estimation loss $\labs$ of a hypothesis $h \in \sH$ and
rejector $r \in \sR$ with respect to their surrogate estimation
loss. They admit the following form:
\[
\sE_{\labs}(h, r) - \sE_{\labs}^*(\sH, \sR) \leq f\paren{\sE_{\sfL}(h,
  r) - \sE_{\sfL}^*(\sH, \sR)},
\]
where $f$ is a non-decreasing
function. Thus, the estimation error $(\sE_{\labs}(h, r) -
\sE_{\labs}^*(\sH,\sR))$ is then bounded by $f(\e)$ if the surrogate
estimation loss $(\sE_{\sfL}(h, r) - \sE_{\sfL}^*(\sH,\sR))$ is reduced
to $\e$.
An important term that appears in these bounds is the
\emph{minimizability gap}, which is defined by
$\sM_\sfL(\sH, \sR) = \sE_{\sfL}^*(\sH, \sR) -
\E_x\bracket*{\inf_{h \in \sH, r\in \sR} \E_y \bracket{\sfL(h, r,
    X, y) \mid X = x}}$.
When the loss function
$\sfL$ depends solely on $h(x, \cdot)$ and $r(x)$, which holds for
most loss functions used in applications, and when $\sH$ and $\sR$
include all measurable functions, the minimizability gap is null
\citep[lemma~2.5]{steinwart2007compare}. However, it is generally
non-zero for restricted hypothesis sets $\sH$ and $\sR$. The
minimizability gap can be upper-bounded by the approximation error
$\sA_\sfL(\sH, \sR) = \sE_{\sfL}^*(\sH, \sR) -
\E_x\bracket[\big]{\inf_{h, r} \E_y \bracket{\sfL(h, r, X, y) \mid X =
    x}}$, where the infimum is taken over all measurable
functions. But, the minimizability gap is a finer quantity and leads
to more favorable guarantees.

\section{Counterexample for score-based abstention losses}
\label{sec:score-example}

We first discuss a natural example here that can be tackled
straightforwardly in the predictor-rejector formulation but for which
the same solution cannot be derived in the score-based formulation
setting, unless we resort to more complex functions. This natural
example motivates our study of surrogate losses for the
predictor-rejector formulation in Section~\ref{sec:general}. We begin by
discussing the score-based formulation and subsequently
highlight its relationship with the predictor-rejector formulation.

\textbf{Score-based abstention formulation.}
In this version of the abstention problem, the label set $\sY$ is
expanded by adding an extra category $(n + 1)$, which represents
abstention. We indicate the augmented set as $\wt \sY = \curl*{1, \ldots, n,
  n + 1}$ and consider a hypothesis set $\wt \sH$ comprising functions
that map from $\sX \times \wt \sY$ to $\Rset$.
The label assigned to an input $x \in \sX$ by $\wt h \in \wt \sH$ is
denoted as $\wt \hh(x)$. It is defined as $\wt \hh(x) = n + 1$ if $\wt
h(x, n + 1) \geq \max_{y \in \sY} \wt h(x, y)$; otherwise, $\wt
\hh(x)$ is determined as an element in $\sY$ with the highest score,
$\wt \hh(x) = \argmax_{y \in \sY} \wt h(x, y)$, using an arbitrary yet
fixed deterministic strategy to break ties. When $\wt \hh(x) = n + 1$,
the learner chooses to abstain from predicting for $x$ and incurs a
cost $c$. In contrast, it predicts the label $y = \wt \hh(x)$ if
otherwise. The \emph{score-based abstention loss} $\labsc$ for this
formulation is defined for any $\wt h \in \wt \sH$ and $(x, y) \in \sX
\times \sY$ as follows:
\begin{equation}
\label{eq:abs-score}
\labsc(\wt h, x, y)
= \1_{\wt \hh(x)\neq y}\1_{\wt \hh(x)\neq n + 1} + c \, \1_{\wt \hh(x) = n + 1}.
\end{equation}
Thus, as in the predictor-rejector context, when the learner does not
abstain ($\wt \hh(x) \neq n + 1$), it reduces to the familiar
zero-one classification loss. Conversely, when it abstains ($\wt
\hh(x) = n + 1$), it incurs the cost $c$. With a finite sample
drawn i.i.d.\ from $\sD$, the learning problem involves choosing a
hypothesis $\wt h$ within $\wt \sH$ that yields a minimal expected
score-based abstention loss, $\E_{(x, y) \sim \sD}[\labsc(\wt h, x,
  y)]$.

\textbf{Relationship between the two formulations.}
In the score-based formulation, rejection is defined by the condition
$\wt h(x, n + 1) - \max_{y \in \sY} \wt h(x, y) \geq 0$. Thus, a
predictor-rejector formulation with $(h, r) \in \sH \times \sR$ can be
equivalently formulated as a score-based problem with $\wt h$ defined
by $\wt h(x, y) = h(x, y)$ for $y \in \sY$ and $\wt h(x, n + 1) =
\max_{y \in \sY} h(x, y) - r(x)$: $\labsc(\wt h, x, y) = \labs(h, r,
x, y)$ for all $(x, y) \times \sX \times \sY$. Note, however, that
function $\wt h(\cdot, n + 1)$ is in general more complex.  As an
example, while $h$ and $r$ may both be in a family of linear
functions, in general, $\wt h(\cdot, n + 1)$ defined in this way is no
more linear. Thus, a score-based formulation might require working
with more complex functions. We further elaborate on the difference
between the two formulations in Appendix~\ref{app:difference}.
\ignore{Similarly, a score-based formulation
with $\wt h$ can be equivalently viewed as a predictor-rejector formulation
with $h$ being the restriction of $\wt h$ to $\sX \times \sY$ and
$r(x) = \max_{y \in \sY} \wt h(x, y) - \wt h(x, n + 1)$.}

\setlength{\intextsep}{0pt}
\setlength{\columnsep}{7pt}
\begin{wrapfigure}{Rt}{0.32\textwidth}
\begin{center}
\vskip -.2in
  \includegraphics[scale=0.28]{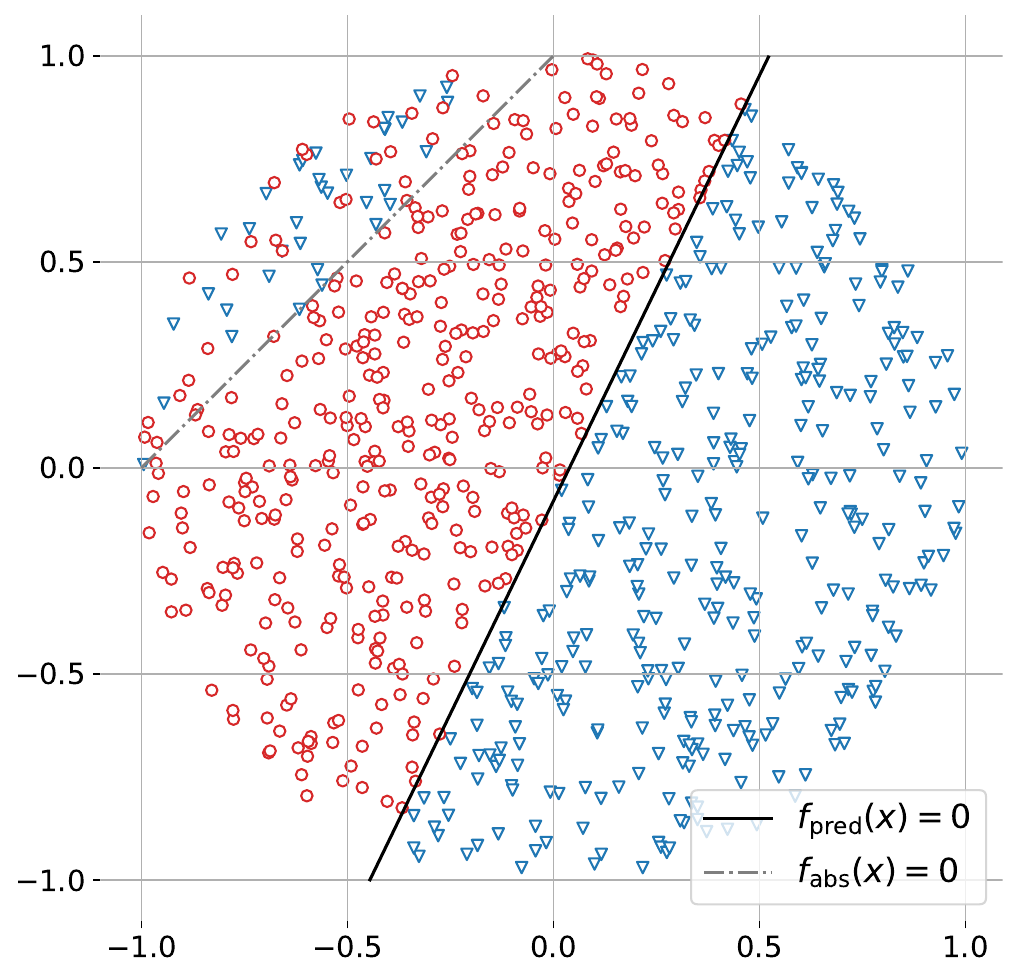}
\vskip -0.1in
\caption{Counterexample for score-based abstention losses.}
\vskip -0.15in
\label{fig:example}
\end{center}
\end{wrapfigure}

\textbf{Counterexample setting.} Let $\sY = \curl*{1, 2}$ and let $x$
follow the uniform distribution on the unit ball $\sfB_2(1) =
\curl*{x\colon \norm*{x}_2 \leq 1}$. We will consider the linear
models $f\in\sF_{\mathrm{lin}} = \curl*{x \rightarrow \bw\cdot x + \bb
  \mid \|\bw\|_2 = 1}$.  We set the label of a point $x$ as follows:
fix two linear functions $f_{\rm{abs}}(x) = \bw_{\rm{abs}}\cdot
x+\bb_{\rm{abs}}$ and $f_{\rm{pred}}(x) = \bw_{\rm{pred}}\cdot
x + \bb_{\rm{pred}}$ in $\sF_{\mathrm{lin}}$ , if $f_{\rm{abs}}(x)\leq
0$, then set $y = 1$ with probability $\frac12$ and $y = 2$ with
probability $\frac12$; if $f_{\rm{abs}}(x)> 0 \text{ and
}f_{\rm{pred}}(x) > 0$, then set $y = 1$; if $f_{\rm{abs}}(x) > 0 \text{
  and }f_{\rm{pred}}(x)\leq 0$, then set $y = 2$; see
Figure~\ref{fig:example}.

We denote by $\sH_{\mathrm{lin}}$ the hypothesis set of linear scoring
functions $h$ with two labels: $h(\cdot, 1)$ and $h(\cdot, 2)$
are in $\sF_{\rm{lin}}$ with the natural constraint $h(\cdot,1) +
h(\cdot,2) = 0$ as in \citet{lee2004multicategory}. We also denote by
$\wt \sH_{\mathrm{lin}}$ the linear hypothesis set of functions $\wt
h$ with three scores $\wt h(\cdot,1)$, $\wt h(\cdot, 2)$, $\wt
h(\cdot, 3)$ in $\sF_{\rm{lin}}$ with the same constraint $\wt
h(\cdot, 1) + \wt h(\cdot, 2)=0$ while $\wt h(\cdot, 3)$ is
independent of $\wt h(\cdot,1)$ and $\wt h(\cdot, 2)$. Note that here the
constraint is imposed only to simplify the analysis and is not
necessary for the counter-example to hold.  Thus, for any cost $c \in
\left[0,\frac12\right)$, the Bayes solution in this setting consists
  of abstaining on $\curl*{x\in \sfB_2(1):f_{\rm{abs}}(x)\leq 0}$ and
  otherwise making a prediction according to the decision surface
  $f_{\rm{pred}}(x) = 0$.

In the predictor-rejector formulation, the learner seeks to select a hypothesis $h$ in $\sH_{\rm{lin}}$
and a rejector $r$ in $\sF_{\rm{lin}}$ with small expected predictor-rejector
loss, $\E_{(x, y) \sim \sD}[\labs(h, r, x, y)]$. In the score-based abstention formulation, the learner seeks to select a hypothesis $\wt h$ in $\wt
\sH_{\rm{lin}}$ with small expected score-based abstention loss, $\E_{(x, y) \sim
  \sD}[\labsc(\wt h, x, y)]$. We will show that, in the predictor-rejector formulation, it is straightforward to find the Bayes solution but, in the score-based formulation, the same solution cannot be achieved, unless a more complex family of functions is adopted for $\wt h(\cdot, 3)$.

\textbf{Predictor-rejector formulation succeeds.} For the predictor-rejector abstention loss $\labs$, it is straightforward to see that for any cost $c\in \left[0,\frac12\right)$, the best-in-class predictor and rejector $h^*_{\sH_{\rm{lin}}}$ and $r^*_{\sF_{\rm{lin}}}$ can be expressed as follows:
$
h^*_{\sH_{\rm{lin}}}(\cdot, 1) = f_{\rm{pred}}(\cdot), \,h^*_{\sH_{\rm{lin}}}(\cdot, 2)  = -f_{\rm{pred}}(\cdot), \,r^*_{\sF_{\rm{lin}}}  = f_{\rm{abs}}
$.
Moreover, it is clear that $h^*_{\sH_{\rm{lin}}}$ and
$r^*_{\sF_{\rm{lin}}}$ match the Bayes solution.

\textbf{Score-based abstention formulation fails.} 
For the score-based abstention loss $\labsc$, for any cost $c\in \left[0,\frac12\right)$, the best-in-class classifier $\wt h^*_{\wt \sH_{\rm{lin}}}\!\!$ has the following form:
$
\wt h^*_{\wt \sH_{\rm{lin}}}(\cdot, 1) =  f_1(\cdot),\,
\wt h^*_{\wt \sH_{\rm{lin}}}(\cdot, 2) = -f_1(\cdot),\,
\wt h^*_{\wt \sH_{\rm{lin}}}(\cdot,3) = f_2(\cdot),
$
for some $f_1,f_2\in \sF_{\rm{lin}}$. Thus, $\wt h^*_{\wt
  \sH_{\rm{lin}}}\!\!$ abstains from making a prediction on
$\curl*{x\in \sfB_2(1):f_2(x)\leq \abs*{f_1(x)}}$ and otherwise predicts
according to the decision surface $f_1(x) = 0$. To match the Bayes
solution, $f_1$ must equal $f_{\rm{pred}}$ and $f_2$ must
satisfy the following condition:
$
  f_2(x) \geq \abs*{f_{\rm{pred}}(x)}
  \Leftrightarrow f_{\rm{abs}}(x) \leq 0,
$
which does not hold. Therefore,
unless we resort to more complex functions for $\wt h(\cdot,3)$, the
score-based formulation cannot result in the Bayes solution.

\section{Predictor-rejector surrogate losses}
\label{sec:general}

In this section, we present and analyze a new family of
surrogate loss functions $\sfL$ for $\labs$, first in 
the \emph{single-stage setting}, where the predictor $h$ and
the rejector $r$ are selected simultaneously, next in 
a \emph{two-stage setting}, where first the predictor $h$ 
is chosen and fixed and subsequently the rejector $r$ is determined. We give $(\sH,\sR)$-consistency bounds and guarantees for both settings. 

\subsection{Single-stage predictor-rejector surrogate losses}
\label{sec:single-stage}

In view of the expression of the predictor-rejector abstention loss $\labs(h, r, x, y)
= \1_{\hh(x) \neq y} \1_{r(x)> 0} + c\1_{r(x) \leq 0}$, if $\ell$ is a surrogate loss for the zero-one
multi-class classification loss over the set of labels $\sY$, then, $\sfL$ defined as follows is a natural surrogate loss for $\labs$: for all $(x, y) \in \sX \times \sY$,
\begin{equation}
\label{eq:sur-general}
\mspace{-0mu}
\sfL(h, r, x, y)
= \ell(h, x, y)\Phi\paren*{-\alpha r(x)} + \Psi(c) \Phi\paren*{\beta r(x)},
\mspace{-5mu}
\end{equation}
where $\Psi$ is a non-decreasing
function, $\Phi$ is a non-increasing auxiliary
function upper bounding $t \mapsto \1_{t \leq 0}$ and $\alpha$, $\beta$ are positive constants. The formulation \eqref{eq:sur-general} of
$\sfL$ is a multi-class generalization of the binary abstention
surrogate loss proposed in \citep{CortesDeSalvoMohri2016bis,CortesDeSalvoMohri2016}, where the binary margin-based loss $\wt\Phi(yh(x))$ and $\Psi(t) =t$ are used instead:
\begin{equation}
\label{eq:sur-general-binary}
\sfL_{\mathrm{bin}}(h, r, x, y) =
\wt\Phi(yh(x))\Phi\paren*{-\alpha r(x)} + c \Phi\paren*{\beta r(x)}.
\mspace{-1mu}
\end{equation}
Minimizing $\sfL_{\mathrm{bin}}$ with a regularization term was shown to achieve state-of-the-art results in the binary case with margin-based losses $\wt\Phi$ such as the exponential loss $\wt\Phi_{\rm{exp}}(t) = \exp(-t)$ and the hinge loss $\wt\Phi_{\rm{hinge}}(t) = \max\curl*{1 - t,0}$ \citep{CortesDeSalvoMohri2016bis,CortesDeSalvoMohri2016}. However, we will show below that its multi-class generalization $\sfL$ imposes a more stringent condition on the choice of the surrogate loss $\ell$, which rules out for example the multi-class exponential loss. We
will show, however, that several other loss
functions do satisfy that condition, for 
example the multi-class hinge loss.
In the following, for simplicity, we consider $\Phi(t) = \exp(-t)$ as in \citep{CortesDeSalvoMohri2016bis} for the main analysis, though a similar analysis can be given for other functions $\Phi$. We first present a negative result, ruling out 
surrogate losses $\sfL$ that are based on a loss $\ell$ that does not verify a certain
condition. 

As with \citep{awasthi2022multi}, we assume that the hypothesis sets are \emph{symmetric and complete}.
We say that a hypothesis set $\sG$ is \emph{symmetric} if there exists a family
$\sF$ of functions $f$ mapping from $\sX$ to $\Rset$ such that
$\curl*{\bracket*{g(x,1),\ldots,g(x,n)}\colon g\in
  \sG} = \curl*{\bracket*{f_1(x),\ldots, f_n(x)}\colon f_1, \ldots, f_n\in
  \sF}$, for any $x \in \sX$. We say that a hypothesis set $\sH$ is \emph{complete} if the set
of scores it generates spans $\Rset$, that is, $\curl*{g(x,y)\colon
  g\in \sG} = \Rset$, for any $(x, y)\in \sX \times \sY$. The hypothesis sets widely used in practice including linear models and multilayer feedforward neural networks are all symmetric and complete.

\begin{restatable}[\textbf{Negative result for single-stage surrogates}]{theorem}{NegativeBound}
\label{Thm:negative-bound}
Assume that $\sH$ is symmetric and complete, and that $\sR$ is
complete. If there exists $x \in \sX$ such that $\inf_{h \in \sH}
\E_y\bracket*{\ell(h,X, y) \mid X = x}\neq \frac{\beta \Psi \paren*{1
    - \max_{y\in \sY}p(x, y)}}{\alpha}$, then, there does not exist a
non-decreasing function $\Gamma\colon \Rset_{+} \to \Rset_{+}$ with the property
$\lim_{t\to 0^{+}}\Gamma(t) = 0$ such that the following $(\sH,
\sR)$-consistency bound holds: for all $h \in \sH$, $r \in \sR$, and
any distribution,
\ifdim\columnwidth=\textwidth
\begin{equation*}
\sE_{\labs}(h, r) - \sE_{\labs}^*(\sH, \sR) + \sM_{\labs}(\sH, \sR)
\leq \Gamma\paren*{\sE_{\sfL}(h, r) - \sE_{\sfL}^*(\sH, \sR) +
  \sM_{\sfL}(\sH, \sR)}.
\end{equation*}
\else
\begin{multline*}
\sE_{\labs}(h, r) - \sE_{\labs}^*(\sH, \sR) + \sM_{\labs}(\sH, \sR)
\\ \leq \Gamma\paren*{\sE_{\sfL}(h, r) - \sE_{\sfL}^*(\sH, \sR) +
  \sM_{\sfL}(\sH, \sR)}.
\end{multline*}
\fi
\end{restatable}
The proof (Appendix~\ref{app:general-negativ}) proceeds by
contradiction. Assuming that the $(\sH, \sR)$-consistency bound is
valid would entail that the pointwise best-in-class predictor and
best-in-class rejector for the single-stage surrogate loss align with
those of the abstention loss, which can be characterized by
Lemma~\ref{lemma:calibration_gap_general} in
Appendix~\ref{app:general}. Incorporating those explicit forms into
the analysis of the conditional risk of the surrogate loss leads to a
contradiction upon examination of its derivatives.

In view of Theorem~\ref{Thm:negative-bound}, to find a surrogate loss $\sfL$
that admits a meaningful $(\sH, \sR)$-consistency bound, we need to
consider multi-class surrogate losses $\ell$ for which the following
condition holds for any $x\in \sX$, for some $\Psi$ and pair
$(\alpha, \beta) \in \Rset_{\plus}^2$:
\begin{align*}
  \inf_{h \in \sH} \E_y\bracket*{\ell(h,X, y) \mid X = x}
  = \frac{\beta\Psi\paren*{1 - \max_{y\in \sY}p(x, y)}}{\alpha}.
\end{align*}
In the binary case, this condition is easily verifiable, as $\max_{y\in \sY}p(x, y)$ uniquely determines the other probabilities. However, in the multi-class scenario, a fixed $\max_{y\in \sY}p(x, y)$ still allows for variation in the other probabilities within $\inf_{h \in \sH} \E_y\left[\ell(h,X, y) \mid X = x\right]$. This leads to the difficulties in extending the binary framework to the multi-class classification.

Nevertheless, we will show that this necessary condition is satisfied by three common
multi-class surrogate losses $\ell$. Furthermore, we will prove $(\sH,
\sR)$-consistency bounds for the predictor-rejector surrogate losses
$\sfL$ based on any of these three choices of $\ell$ defined for all
$h \in \sH$ and $(x, y)$ as follows:

(i) The \emph{mean absolute error loss} \citep{ghosh2017robust}: $\ell_{\rm{mae}}(h, x, y) = 1
- \frac{e^{h(x, y)}}{\sum_{y'\in \sY}e^{h(x, y')}}$;

(ii) The \emph{constrained $\rho$-hinge loss}:
$\ell_{\rho-\mathrm{hinge}}(h, x, y) = \sum_{y'\neq
  y}\Phi_{\rho-\rm{hinge}}\paren*{-h(x, y')}$, $\rho > 0$, with
$\Phi_{\rho-\rm{hinge}}(t) = \max\curl[\big]{0, 1 - \frac{t}{\rho}}$
the $\rho$-hinge loss, and the constraint $\sum_{y\in \sY} h(x, y) =
0$.

(iii) The \emph{$\rho$-Margin loss}:
$\ell_{\rho}(h, x, y) = \Phi_{\rho}\paren*{\rho_h(x, y)}$, with
$\rho_h(x, y) = h(x, y) - \max_{y' \neq y} h(x, y')$ the confidence
margin and $\Phi_{\rho}(t) = \min\curl[\big]{\max\curl[\big]{0,1 -
    \frac{t}{\rho}},1}, \rho>0$ the $\rho$-margin loss.
    
\begin{restatable}[\textbf{$(\sH, \sR)$-consistency bounds for
      single-stage surrogates}]{theorem}{SpecificLossBound}
\label{Thm:spcific-loss-bound}
Assume that $\sH$ is symmetric and complete and $\sR$ is
complete. Then, for $\alpha=\beta$, and $\ell = \ell_{\rm{mae}}$, or
$\ell = \ell_{\rho}$ with $\Psi(t) = t$, or $\ell =
\ell_{\rho-\mathrm{hinge}}$ with $\Psi(t) = nt$, the following $(\sH,
\sR)$-consistency bound holds for all $h \in \sH$, $r \in \sR$ and any distribution:
\ifdim\columnwidth=\textwidth
\begin{equation*}
\sE_{\labs}(h, r) - \sE_{\labs}^*(\sH, \sR) + \sM_{\labs}(\sH, \sR)
\leq \Gamma\paren*{\sE_{\sfL}(h, r) - \sE_{\sfL}^*(\sH, \sR) +
  \sM_{\sfL}(\sH, \sR)},
\end{equation*}
\else
\begin{multline*}
\sE_{\labs}(h, r) - \sE_{\labs}^*(\sH, \sR) + \sM_{\labs}(\sH, \sR) \\
\leq \Gamma\paren*{\sE_{\sfL}(h, r)-\sE_{\sfL}^*(\sH, \sR) +\sM_{\sfL}(\sH, \sR)},
\end{multline*}
\fi
where $\Gamma(t) = \max\curl*{2n\sqrt{t}, n\,t}$ for $\ell=
\ell_{\rm{mae}}$; $\Gamma(t) = \max\curl*{2\sqrt{t}, t}$ for $\ell=
\ell_{\rho}$; and $\Gamma(t) = \max\curl*{2\sqrt{n t}, t}$ for $\ell=
\ell_{\rho-\rm{hinge}}$.
\end{restatable}
The theorem provides
strong guarantees for the predictor-rejector surrogate losses we
described in the single-stage setting. The technique used in the proof (Appendix~\ref{app:general-positive-single-stage})
is novel and requires careful analysis of various cases involving the pointwise
best-in-class predictor and rejector. This analysis is
challenging and needs to take into account the \emph{conditional risk} and \emph{calibration gap} (see Appendix~\ref{app:general}) of specific loss functions. The approach is
substantially different from the standard scenarios examined in
\citep{awasthi2022multi}, due to the simultaneous minimization of both
the predictor and rejector in the abstention setting. Discussions on Theorem~\ref{Thm:spcific-loss-bound} are given in Remark~\ref{remark:spcific-loss-bound} in Appendix~\ref{app:remark}. The following is a direct
consequence of Theorem~\ref{Thm:spcific-loss-bound} when $\sH$ and $\sR$ include all measurable functions,
since the minimizability gaps $\sM_{\labs}$ and $\sM_{\sfL}$ are
then zero.

\begin{restatable}[\textbf{Excess error bounds for single-stage surrogates}]{corollary}{SpecificLossBoundCor}
\label{cor:spcific-loss-bound}
For $\alpha=\beta$, and $\ell= \ell_{\rm{mae}}$, or $\ell = \ell_{\rho}$
with $\Psi(t)=t$, or
$\ell = \ell_{\rho-\mathrm{hinge}}$ with $\Psi(t)=nt$,
the following excess error bound holds for all $h\in \sH_{\rm{all}}$,
$r\in \sR_{\rm{all}}$ and any distribution:
\begin{equation*}
\sE_{\labs}(h, r) - \sE_{\labs}^*(\sH_{\rm{all}}, \sR_{\rm{all}}) 
 \leq \Gamma\paren*{\sE_{\sfL}(h, r)-\sE_{\sfL}^*(\sH_{\rm{all}}, \sR_{\rm{all}})},
\end{equation*}
where $\Gamma$ has the same form as in Theorem~\ref{Thm:spcific-loss-bound}.
\end{restatable}

The corollary resolves in a positive way the open question mentioned
by \citet{NiCHS19}. In fact it provides a stronger result since it
gives non-asymptotic excess error bounds for the three abstention
surrogate losses previously described. These are stronger guarantees
than Bayes-consistency of these loss functions, which follow
immediately by taking the limit.

It should be noted that our novel single-stage predictor-rejector
surrogate losses might present some challenges for optimization. This
is due to several factors: the difficulty of optimizing the mean
absolute error loss \citep{zhang2018generalized} (also see
Section~\ref{sec:experiments}), the restriction imposed by the
constrained hinge loss, which is incompatible with the standard use of
the softmax function in neural network hypotheses, and the
non-convexity of the $\rho$-margin loss. However, our primary
objective has been a theoretical analysis and the significance of
these surrogate losses lies in their novelty and strong guarantees. As
shown in Corollary~\ref{cor:spcific-loss-bound}, they are the first
Bayes-consistent surrogate losses within the predictor-rejector
formulation for multi-class abstention, addressing an open question in
the literature \citep{NiCHS19}.

\subsection{Two-stage predictor-rejector surrogate losses}
\label{sec:two-stage}

Here, we explore a two-stage algorithmic approach, for which we
introduce surrogate losses with more flexible choices of $\ell$ that
admit better optimization properties, and establish $(\sH,
\sR)$-consistency bounds for them.  This is a key scenario since, in
practice, often a large pre-trained prediction model is already
available (first stage), and retraining it would be prohibitively
expensive. The problem then consists of leaving the first stage
prediction model unchanged and of subsequently learning a useful
rejection model (second stage).

Let $\ell$ be a surrogate loss for standard multi-class classification
and $\Phi$ a function like the exponential function that determines a
margin-based loss $\Phi(yr(x))$ in binary classification, with $y \in
\curl*{-1, +1}$ for a function $r$. We propose a two-stage algorithmic
approach and a surrogate loss to minimize in the second stage: first,
find a predictor $h$ by minimizing $\ell$; second, with $h$ fixed,
find $r$ by minimizing $\ell_{\Phi, h}$, a surrogate loss function of
$r$ for all $(x, y)$, defined by
\begin{align}
\label{eq:ell-Phi-h}
\ell_{\Phi,h}\paren*{r, x, y}  
= \1_{\hh(x) \neq y} \Phi\paren*{-r(x)} + c \Phi\paren*{r(x)},
\end{align}
where $t \mapsto \Phi(t)$ is a non-increasing auxiliary function upper bounding $\1_{t \leq 0}$.
This algorithmic approach is straightforward since the first stage
involves the familiar task of finding a predictor using a standard
surrogate loss, such as logistic loss (or cross-entropy with softmax). The second stage is also relatively simple as $h$ remains
fixed and the form of $\ell_{\Phi, h}$ is uncomplicated, with $\Phi$
possibly being the logistic or exponential loss.
It is important to underscore that the judicious selection of the
indicator function in the initial term of \eqref{eq:ell-Phi-h} plays a
crucial role in guaranteeing that the two-stage surrogate loss
benefits from $(\sH, \sR)$-consistency bounds.  If a surrogate loss is
used in the first stage, this may not necessarily hold.

Note that in \eqref{eq:ell-Phi-h}, $h$ is fixed and only $r$ is learned by minimizing the surrogate loss corresponding to that $h$, while in contrast both $h$ and $r$ are jointly learned in the abstention loss \eqref{eq:abs}. We denote by $\ell_{\mathrm{abs}, h}$ the two-stage version of the abstention loss \eqref{eq:abs} with a fixed predictor $h$, defined as: for any $r \in \sR$, $x \in \sX$ and $y \in \sY$:
\begin{equation}
\label{eq:two-stage-abstention-loss}
\ell_{\mathrm{abs}, h} (r, x, y) = \1_{\hh(x) \neq y} \1_{r(x)> 0} + c \1_{r(x) \leq 0}.
\end{equation}
In other words, both $\ell_{\Phi,h}$ and $\ell_{\mathrm{abs}, h}$ are loss functions of the abstention function $r$, while $\labs$ is a loss function of the pair $(h, r) \in (\sH, \sR)$.

Define the binary zero-one classification loss as
$\ell_{0-1}^{\rm{binary}}(r, x, y) = \1_{y\neq \sign(r(x))}$, where
$\sign(t) = \1_{t>0} - \1_{t\leq 0}$.  As with the single-stage
surrogate losses, the two-stage surrogate losses benefit from strong
consistency guarantees as well. We first show that in the second stage
where a predictor $h$ is fixed, the surrogate loss function
$\ell_{\Phi,h}$ benefits from $\sR$-consistency bounds with respect to
$\ell_{\mathrm{abs}, h}$ if, $\Phi$ admits an $\sR$-consistency bound
with respect to the binary zero-one loss $\ell_{0-1}^{\rm{binary}}$.
\begin{restatable}[\textbf{$\sR$-consistency bounds for second-stage surrogates}]{theorem}{BoundGenralSecondStep}
\label{Thm:bound-general-second-step}
Fix a predictor $h$. Assume that $\Phi$ admits an $\sR$-consistency
bound with respect to $\ell_{0-1}^{\rm{binary}}$.  Thus, there exists
a non-decreasing concave function $\Gamma$ such that, for all $r \in
\sR$,
\begin{align*}
\sE_{\ell_{0-1}^{\rm{binary}}}(r) - \sE_{\ell_{0-1}^{\rm{binary}}}^*(\sR) + \sM_{\ell_{0-1}^{\rm{binary}}}(\sR)
& \leq \Gamma\paren*{\sE_{\Phi}(r)-\sE_{\Phi}^*(\sR) +\sM_{\Phi}(\sR)}.
\end{align*}
Then, the following $\sR$-consistency bound holds for all $r\in \sR$ and any distribution:
\begin{equation*}
\sE_{\ell_{\mathrm{abs}, h} }(r)-\sE_{\ell_{\mathrm{abs}, h} }^*(\sR) +\sM_{\ell_{\mathrm{abs}, h} }(\sR) \leq \Gamma\paren*{\paren*{\sE_{\ell_{\Phi, h} }(r)-\sE_{\ell_{\Phi, h} }^*(\sR) +\sM_{\ell_{\Phi, h}}(\sR)}/c}.
\end{equation*}
\end{restatable}
The proof (Appendix~\ref{app:general-positive-second-stage}) consists
of analyzing the calibration gap of the abstention loss and
second-stage surrogate loss, for a fixed predictor $h$.  The
calibration gap here is more complex than that in the standard setting
as it takes into account the conditional probability, the error of
that fixed predictor and the cost, and thus requires a completely
different analysis. To establish $\sR$-consistency bounds, we need to
upper bound the calibration gap of the abstention loss by that of the
surrogate loss. However, directly working with them is rather
difficult due to their complex forms. Instead, a key observation is
that both forms share structural similarities with the calibration
gaps in the standard classification.  Motivated by the above
observation, we construct an appropriate conditional distribution to
transform the two calibration gaps into standard ones. Then, by
applying Lemma~\ref{lemma:aux} in Appendix~\ref{app:general}, we
manage to leverage the $\sR$-consistency bound of $\Phi$ with respect
to the binary zero-one classification loss to upper bound the target
calibration gap by that of the surrogate calibration gap.

We further discuss Theorem~\ref{Thm:bound-general-second-step} in
Remark~\ref{remark:bound-general-second-step}
(Appendix~\ref{app:remark}). In the special case where $\sH$ and $\sR$
are the family of all measurable functions, all the minimizability gap
terms in Theorem~\ref{Thm:bound-general-second-step} vanish. Thus, we
obtain the following corollary.
\begin{corollary}
\label{cor:tsr}
Fix a predictor $h$. Assume that
$\Phi$ admits an excess error bound with respect to  $\ell_{0-1}^{\rm{binary}}$.  Thus,
there exists a non-decreasing concave functions $\Gamma$
such that, for all $r \in \sR_{\rm{all}}$,
\begin{align*}
\sE_{\ell_{0-1}^{\rm{binary}}}(r) - \sE_{\ell_{0-1}^{\rm{binary}}}^*(\sR_{\rm{all}})
& \leq \Gamma\paren*{\sE_{\Phi}(r)-\sE_{\Phi}^*(\sR_{\rm{all}})}.
\end{align*}
Then, the following excess error bound holds for all $r\in \sR_{\rm{all}}$ and any distribution:
\begin{equation*}
\sE_{\ell_{\mathrm{abs}, h} }(r)-\sE_{\ell_{\mathrm{abs}, h} }^*(\sR_{\rm{all}}) \leq \Gamma\paren*{\paren*{\sE_{\ell_{\Phi, h} }(r)-\sE_{\ell_{\Phi, h} }^*(\sR_{\rm{all}})} / c}.
\end{equation*}
\end{corollary}
See Remark~\ref{remark:tsr} (Appendix~\ref{app:remark}) for a
brief discussion of Corollary~\ref{cor:tsr}.

We now present $(\sH, \sR)$-consistency bounds for the whole two-stage
approach with respect to the abstention loss function $\labs$. Let
$\ell_{0-1}$ be the multi-class zero-one loss: $\ell_{0-1}(h, x, y) =
1_{\hh(x) \neq y}$. We will consider hypothesis sets $\sR$ that are
\emph{regular for abstention}, that is such that for any $x\in
\sX$, there exist $f, g \in \sR$ with $f(x)>0$ and $g(x)\leq 0$.
If $\sR$ is regular for abstention, then, for any $x$, there is an
option to accept and an option to reject.  \ignore{The next result
  shows that for those margin-based loss functions $\Phi$, their
  corresponding two-stage abstention surrogate losses admit
  $(\sH,\sR)$-consistency bounds with respect to the abstention loss
  $\labs$ as well.}
\begin{restatable}[\textbf{$(\sH, \sR)$-consistency bounds for two-stage
      approach}]{theorem}{BoundGenralTwoStep}
\label{Thm:bound-general-two-step} 
Suppose that $\sR$ is regular.  Assume that $\ell$ admits an
$\sH$-consistency bound with respect to $ \ell_{0-1}$ and that $\Phi$
admits an $\sR$-consistency bound with respect to
$\ell_{0-1}^{\rm{binary}}$.  Thus, there are non-decreasing concave
functions $\Gamma_1$ and $\Gamma_2$ such that, for all $h\in \sH$ and
$r \in \sR$,
\begin{align*}
\sE_{\ell_{0-1}}(h) - \sE_{\ell_{0-1}}^*(\sH) + \sM_{\ell_{0-1}}(\sH)
& \leq \Gamma_1\paren*{\sE_{\ell}(h)-\sE_{\ell}^*(\sH) +\sM_{\ell}(\sH)}\\
\sE_{\ell_{0-1}^{\rm{binary}}}(r) - \sE_{\ell_{0-1}^{\rm{binary}}}^*(\sR) + \sM_{\ell_{0-1}^{\rm{binary}}}(\sR)
& \leq \Gamma_2\paren*{\sE_{\Phi}(r)-\sE_{\Phi}^*(\sR) +\sM_{\Phi}(\sR)}.
\end{align*}
Then, the following $(\sH,\sR)$-consistency bound holds for all $ h\in \sH$, $r\in \sR$ and any distribution:
\begin{align*}
\sE_{\labs}(h, r) - \sE_{\labs}^*(\sH, \sR) + \sM_{\labs}(\sH, \sR)
& \leq \Gamma_1\paren*{\sE_{\ell}(h)-\sE_{\ell}^*(\sH) +\sM_{\ell}(\sH)}\\
& \quad + (1+c)\Gamma_2\paren*{\paren*{\sE_{\ell_{\Phi,h}}(r)-\sE_{\ell_{\Phi,h}}^*(\sR) +\sM_{\ell_{\Phi,h}}(\sR)}/c},
\end{align*}
where the
constant factors $(1 + c)$ and $\frac{1}{c}$ can be removed 
when $\Gamma_2$ is linear.
\end{restatable}
In the proof (Appendix~\ref{app:general-positive-two-stage}), we
express the pointwise estimation error term for the target abstention
loss as the sum of two terms. The first term represents the pointwise
estimation error of the abstention loss with a fixed $h$, while the
second term denotes that with a fixed $r^*$. This proof is entirely
novel and distinct from the approach used for a standard loss without
abstention in \citep{awasthi2022multi}. Discussions on
Theorem~\ref{Thm:bound-general-two-step} are given in
Remark~\ref{remark:bound-general-two-step} in
Appendix~\ref{app:remark}. As before, when $\sH$ and $\sR$ are the
family of all measurable functions, the following result on excess
error bounds holds.
\begin{corollary}
\label{cor:tshr}
Assume that $\ell$ admits an excess error bound with respect to $
\ell_{0-1}$ and that $\Phi$ admits an excess error bound with respect
to $\ell_{0-1}^{\rm{binary}}$.  Thus, there are non-decreasing concave
functions $\Gamma_1$ and $\Gamma_2$ such that, for all $h\in
\sH_{\rm{all}}$ and $r \in \sR_{\rm{all}}$,
\begin{align*}
\sE_{\ell_{0-1}}(h) - \sE_{\ell_{0-1}}^*(\sH_{\rm{all}})
& \leq \Gamma_1\paren*{\sE_{\ell}(h)-\sE_{\ell}^*(\sH_{\rm{all}})}\\
\sE_{\ell_{0-1}^{\rm{binary}}}(r) - \sE_{\ell_{0-1}^{\rm{binary}}}^*(\sR_{\rm{all}})
& \leq \Gamma_2\paren*{\sE_{\Phi}(r)-\sE_{\Phi}^*(\sR_{\rm{all}})}.
\end{align*}
Then, the following excess error bound holds for all $ h\in \sH_{\rm{all}}$ and $r\in \sR_{\rm{all}}$ and any distribution:
\begin{align*}
\sE_{\labs}(h, r) - \sE_{\labs}^*(\sH_{\rm{all}}, \sR_{\rm{all}})
& \leq \Gamma_1\paren*{\sE_{\ell}(h)-\sE_{\ell}^*(\sH_{\rm{all}})} + (1 + c)\Gamma_2\paren*{\paren*{\sE_{\ell_{\Phi,h}}(r)-\sE_{\ell_{\Phi,h}}^*(\sR_{\rm{all}})}/c},
\end{align*}
where the
constant factors $(1 + c)$ and $\frac{1}{c}$ can be removed 
when $\Gamma_2$ is linear.
\end{corollary}
See Remark~\ref{remark:tshr} (Appendix~\ref{app:remark}) for a
brief discussion of Corollary~\ref{cor:tshr}.

These results provide a strong guarantee for surrogate losses in the
two-stage setting. Additionally, while the choice of $\ell$ in the
single-stage setting was subject to certain conditions, here, the
multi-class surrogate loss $\ell$ can be chosen more
flexibly. Specifically, it can be selected as the logistic loss (or
cross-entropy with softmax), which is not only easier to optimize but
is also better tailored for complex neural networks. In the second
stage, the formulation is straightforward, and the choice of function
$\Phi$ is flexible, resulting in a simple smooth convex optimization
problem with respect to the rejector function $r$. Moreover, the
second stage simplifies the process as $h$ remains constant and only
the rejector is optimized. This approach can enhance optimization
efficiency. In Appendix~\ref{app:two-stage}, we further highlight the
significance of our findings regarding the two-stage formulation in
comparison with single-stage surrogate losses.

\subsection{Other advantages of the predictor-rejector formulation}
\label{sec:general-realizable}

In this section, we prove another advantage of our predictor-rejector
surrogate losses, that is \emph{realizable consistency}
\citep{long2013consistency, zhang2020bayes}. The property involves
\emph{$(\sH,\sR)$-realizable} distributions, \ignore{that are
  distributions}under which there exist $h^*\in \sH$ and $r^*\in \sR$
such that $\sE_{\labs}(h^*,r^*)=0$. The realizable distribution allows the optimal solution to abstain on points where the cost $c(x)$ is zero.
\begin{definition}
$\sfL$ is \emph{realizable $(\sH,
\sR)$-consistent} with respect to $\labs$ if, for any $(\sH,
\sR)$-realizable distribution, $\lim_{n \rightarrow +\infty}\sE_{\sfL}(h_n,r_n) =  \sE^*_{\sfL}(\sH,\sR) \implies \lim_{n \rightarrow +\infty}\sE_{\labs}(h_n,r_n) =  \sE^*_{\labs}(\sH,\sR)$.
\end{definition}  
In the following, we will establish that our predictor-rejector
surrogate losses, both the single-stage and two-stage variants, are
realizable $(\sH, \sR)$-consistent when $\sH$ and $\sR$ are
\emph{closed under scaling}. A hypothesis set $\sG$ is closed under
scaling if, $g \in \sG$ implies that $\nu g$ is also in $\sG$ for any
$\nu \in \Rset$.\ignore{ This property underscores the advantages of our
predictor-rejector formulation and is supported by the empirical
success of our proposed surrogate losses, as demonstrated in
Section~\ref{sec:experiments}.}
We will adopt the following mild assumption for the auxiliary function
$\Phi$ in the \ignore{both the single-stage and
  two-stage}predictor-rejector surrogate losses.
\begin{assumption}
\label{assumption:phi}
For any $t \in \Rset$, $\Phi(t) \geq \1_{t \leq 0}$ and $\lim_{t\to
  \plus \infty}\Phi(t) = 0$.
\end{assumption}

In other words, $\Phi$ upper bounds the indicator function and
approaches zero as $t$ goes to infinity.  \ignore{This is satisfied by
  common margin-based loss functions, e.g. the hinge function $t
  \mapsto \max(0, 1 - t)$ used in support vector machines, the
  exponential function $t \mapsto e^{-t}$ in Adaboost, and the
  logistic function $t \mapsto \log (1 + e^{-t})$ in logistic
  regression, etc.}We first prove a general result showing that 
single-stage predictor-rejector surrogate losses are realizable
$(\sH,\sR)$-consistent with respect to $\sfL_{\rm{abs}}$ if the
adopted $\ell$ satisfies Assumption~\ref{assumption:ell}.
\begin{assumption}
\label{assumption:ell}
When $h(x, y) - \max_{y' \neq y}h (x, y') > 0$, $\lim_{\nu \to \plus
  \infty}\ell(\nu h, x,y ) = 0$ and $\ell \geq \ell_{0-1}$.
\end{assumption}
The assumption implies that for a sample $(x, y)$ for which a predictor
$h$ achieves zero error, the infimum value of the loss function $\ell$
is zero for any hypothesis set $\sH$ that includes the
predictor $h$, provided that $\sH$ is closed under scaling.

\begin{restatable}{theorem}{SpecificLossBoundRealizable}
\label{Thm:spcific-loss-bound-realizable}
Assume that $\sH$ and $\sR$ are closed under scaling. Let $\Psi(0)=0$
and $\Phi$ satisfy Assumption~\ref{assumption:phi}. Then, for any $\ell$ that satisfies Assumption~\ref{assumption:ell},
the following $(\sH, \sR)$-consistency
bound holds for any $(\sH,\sR)$-realizable distribution, $h \in \sH$ and 
$r \in \sR$:
\begin{equation*}
\sE_{\labs}(h, r) - \sE_{\labs}^*(\sH, \sR)
\leq \sE_{\sfL}(h, r)-\sE_{\sfL}^*(\sH, \sR).
\end{equation*}
\end{restatable}
The proof is included in Appendix~\ref{app:general-positive-single-stage-realizable}. We first establish upper bounds for $\sE_{\sfL}^*(\sH, \sR)$ using the optimal predictor and rejector for the abstention loss, subsequently expressing the upper bound as the sum of two terms.  By applying the Lebesgue dominated convergence theorem, we show that both terms vanish, and thus $\sE_{\sfL}^*(\sH, \sR) = 0$. It is worth noting that $(\sH, \sR)$-consistency bounds in Theorem~\ref{Thm:spcific-loss-bound} imply the realizable-consistency for considered loss functions. This is because under the realizable assumption, the minimizability gaps vanish for these loss functions. Nevertheless, Theorem~\ref{Thm:spcific-loss-bound-realizable} proves that a more general family of loss functions can actually achieve realizable consistency.

According to
Theorem~\ref{Thm:spcific-loss-bound-realizable}, for any distribution
that is $(\sH, \sR)$-realizable, minimizing the single-stage surrogate
estimation loss $\sE_{\sfL}(h, r)-\sE_{\sfL}^*(\sH, \sR)$ results in
the minimization of the abstention estimation loss $\sE_{\labs}(h, r)
- \sE_{\labs}^*(\sH, \sR)$. This suggests that the single-stage
predictor-rejector surrogate loss functions are realizable
$(\sH, \sR)$-consistent. In particular, when $\ell$ is chosen as $2\ell_{\rm{mae}}$, $\ell = \ell_{\rho}$ and $\ell =
\ell_{\rho-\mathrm{hinge}}$ as suggested in Section~\ref{sec:single-stage}, Assumption~\ref{assumption:phi} is satisfied. Thus, we obtain the following corollary.
\begin{corollary}
Under the same assumption as in Theorem~\ref{Thm:spcific-loss-bound-realizable}, for
$\ell= 2\ell_{\rm{mae}}$, $\ell = \ell_{\rho}$ and $\ell =
\ell_{\rho-\mathrm{hinge}}$,  the single-stage
predictor-rejector surrogate $\sfL$ is realizable $(\sH,
\sR)$-consistent with respect to $\labs$.

\ignore{the following $(\sH, \sR)$-consistency
bound holds for any $(\sH,\sR)$-realizable distribution, $h \in \sH$ and 
$r \in \sR$:
\begin{equation*}
\sE_{\labs}(h, r) - \sE_{\labs}^*(\sH, \sR)
\leq \sE_{\sfL}(h, r)-\sE_{\sfL}^*(\sH, \sR).
\end{equation*}}
\end{corollary}

Next, we prove a similar result showing that the two-stage predictor-rejector surrogate losses are realizable $(\sH,\sR)$-consistent with respect to $\sfL_{\rm{abs}}$ if the multi-class surrogate loss $\ell$ is realizable $\sH$-consistent with respect to the multi-class zero-one loss $\ell_{0-1}$ when $\sH$ is closed under scaling.
\begin{definition}
We say that $\ell$ is realizable $\sH$-consistent with respect to $\ell_{0-1}$ if, for any distribution such that $\sE^*_{\ell_{0-1}}(\sH) = 0$, $\lim_{n \rightarrow +\infty}\sE_{\ell}(h_n) = \sE^*_{\ell}(\sH) \implies \lim_{n \rightarrow +\infty}\sE_{\ell_{0-1}}(h_n) =  \sE^*_{\ell_{0-1}}(\sH) = 0 $.
\end{definition}

\begin{restatable}{theorem}{BoundGenralTwoStepRealizable}
\label{Thm:bound-general-two-step-realizable}
Assume that $\sH$ and $\sR$ are closed under scaling.  Let $\ell$ be
any multi-class surrogate loss that is realizable $\sH$-consistent with respect to $\ell_{0-1}$ when $\sH$ is closed under scaling and $\Phi$ satisfies Assumption~\ref{assumption:phi}.  Let $\hat h$ be the minimizer of $\sE_{\ell}$ and $\hat r$ be the minimizer of $\sE_{\ell_{\Phi, \hat h}}$.
  Then, for any
$(\sH,\sR)$-realizable distribution, $\sE_{\labs}(\hat h, \hat r) = 0$.
\end{restatable}
The proof is included in
Appendix~\ref{app:general-positive-two-stage-realizable}.  We first
establish the upper bound $\sE_{\labs}(\hat h, \hat r) \leq
\sE_{\ell_{\Phi, \hat h}}(\hat r)$. Next, we analyze two cases:
whether abstention occurs or not. By applying the Lebesgue dominated
convergence theorem, we show that $\sE_{\ell_{\Phi, \hat h}}(\hat r) =
0$ in both cases, consequently leading to $\sE_{\labs}(\hat h, \hat r)
= 0$.  By Theorem~\ref{Thm:bound-general-two-step-realizable}, under
the realizability assumption, minimizing a two-stage
predictor-rejector surrogate loss leads to zero abstention loss. This
implies that the two-stage predictor-rejector surrogate loss functions
are also realizable $(\sH,
\sR)$-consistent. \citet{KuznetsovMohriSyed2014} prove the realizable
$\sH$-consistency of a broad family of multi-class surrogate losses
including the logistic loss commonly used in practice. Thus, we obtain
the following corollary.
\begin{corollary}
\label{cor:bound-general-two-step-realizable}
Under the same assumption as in
Theorem~\ref{Thm:bound-general-two-step-realizable}, for $\ell$ being
the logistic loss, the two-stage predictor-rejector surrogate loss is
realizable $(\sH, \sR)$-consistent with respect to $\labs$.
\end{corollary}

Note that existing score-based abstention surrogate losses were shown
to be not realizable consistent in \citep{pmlr-v206-mozannar23a}.
Recall that in Section~\ref{sec:single-stage} and
Section~\ref{sec:two-stage}, the $(\sH, \sR)$-consistency bounds
guarantees (applicable to all distributions without any assumptions)
indicate that both our single-stage and two-stage predictor-rejector
surrogate losses are also Bayes-consistent, while it is unknown if the
surrogate loss proposed by \citet{pmlr-v206-mozannar23a} is.
By combining the results from Section~\ref{sec:single-stage},
Section~\ref{sec:two-stage} and Section~\ref{sec:general-realizable},
we demonstrate the advantages of the predictor-rejector
formulation. As a by-product of our results, we address two open
questions in the literature \citep{NiCHS19} and
\citep{pmlr-v206-mozannar23a} indicated in Section~\ref{sec:intro}.

\section{Experiments}
\label{sec:experiments}

\begin{table}[t]
\caption{Abstention loss of our predictor-rejector surrogate losses against baselines:
the state-of-the-art score-based abstention surrogate losses in
\citep{mozannar2020consistent,caogeneralizing}.}
  \vskip -0.2in
    \label{tab:comparison}
\begin{center}
    \begin{tabular}{@{\hspace{0pt}}lll@{\hspace{0pt}}}
    \toprule
      Dataset & Method & Abstention loss \\
    \toprule
    \multirow{4}{*}{SVHN} & \citep{mozannar2020consistent} &  1.61\% $\pm$ 0.06\% \\
     & \citep{caogeneralizing} & 2.16\% $\pm$ 0.04\%\\
     & single-stage predictor-rejector ($\ell_{\rm{mae}}$) &  2.22\% $\pm$ 0.01\% \\
     & \textbf{two-stage predictor-rejector}  & \textbf{0.94\% \!$\pm$ 0.02\%} \\
    \midrule
    \multirow{4}{*}{CIFAR-10} & \citep{mozannar2020consistent} &  4.48\% $\pm$ 0.10\% \\
     & \citep{caogeneralizing}  & 3.62\% $\pm$ 0.07\%  \\
     & single-stage predictor-rejector ($\ell_{\rm{mae}}$) & 3.64\% $\pm$ 0.05\% \\
     & \textbf{two-stage predictor-rejector}  & \textbf{3.31\% \!$\pm$ 0.02\%}   \\
    \midrule
    \multirow{4}{*}{CIFAR-100} & \citep{mozannar2020consistent} &  10.40\% $\pm$ 0.10\% \\
     & \citep{caogeneralizing}  & 14.99\% $\pm$ 0.01\%\\
     & single-stage predictor-rejector ($\ell_{\rm{mae}}$) & 14.99\% $\pm$ 0.01\% \\
     & \textbf{two-stage predictor-rejector} & \textbf{\phantom{0}9.23\% \!$\pm$ 0.03\%}  \\
    \bottomrule
    \end{tabular}
\end{center}
    \vskip -0.25in
\end{table}
In this section, we present experimental results for the single-stage
and two-stage predictor-rejector surrogate losses, as well as for the
state-of-the-art score-based abstention surrogate losses
\citep{mozannar2020consistent,caogeneralizing} on three popular
datasets: SVHN \citep{Netzer2011}, CIFAR-10 and CIFAR-100
\citep{Krizhevsky09learningmultiple}. Note that the basic
confidence-based approach has already been shown in
\citep{caogeneralizing} to be empirically inferior to state-of-the-art
score-based abstention surrogate losses. More details on the
experiments are included in Appendix~\ref{app:setup}.

\textbf{Results.}  In Table~\ref{tab:comparison}, we report the mean
and standard deviation of the abstention loss over three runs for our
algorithms and the baselines.
Table~\ref{tab:comparison} shows that our two-stage predictor-rejector
surrogate loss consistently outperforms the state-of-the-art
score-based abstention surrogate losses in
\citep{mozannar2020consistent,caogeneralizing} across all cases. The
single-stage predictor-rejector surrogate loss with $\ell$ set as the
mean absolute error loss achieves comparable results. Our
predictor-rejector surrogate losses, both the single-stage and
two-stage variants, benefit from $(\sH, \sR)$-consistency bounds and
realizable $(\sH, \sR)$-consistency guarantees. 
While the optimization of mean absolute error loss is known to be
challenging, as highlighted in the study by Zhang et al. (2018), our
two-stage algorithm sidesteps this hurdle since it can use the
more tractable logistic loss.

\section{Conclusion}

We presented a series of theoretical, algorithmic, and empirical
results for multi-class classification with abstention. Our
theoretical analysis, including proofs of $(\sH, \sR)$-consistency
bounds and realizable $(\sH, \sR)$-consistency, covers both single-stage
and two-stage predictor-rejector surrogate losses.
These results further provide valuable tools applicable to the
analysis of other loss functions in learning with abstention.

Our two-stage algorithmic approach provides practical and efficient
solutions for multi-class abstention across various tasks. This
approach proves particularly advantageous in scenarios where a large
pre-trained prediction model is readily available, and the expense
associated with retraining is prohibitive. Our empirical findings
corroborate the efficacy of these algorithms, further reinforcing
their practical usefulness.
Additionally, our work reveals some limitations of the 
score-based abstention formulation, such as its inability to
consistently yield optimal solutions in certain cases. In contrast, we
present a collection of positive outcomes for various families of
predictor-rejector surrogate loss functions. Importantly, our findings
also provide resolutions to two open questions within the literature.

We believe that our analysis and the novel loss functions we
introduced can guide the design of algorithms across a
broad spectrum of scenarios beyond classification with abstention.


\bibliography{mabs}

\begin{thebibliography}{99}
\providecommand{\natexlab}[1]{#1}
\providecommand{\url}[1]{\texttt{#1}}
\expandafter\ifx\csname urlstyle\endcsname\relax
  \providecommand{\doi}[1]{doi: #1}\else
  \providecommand{\doi}{doi: \begingroup \urlstyle{rm}\Url}\fi

\bibitem[Awasthi et~al.(2021{\natexlab{a}})Awasthi, Frank, Mao, Mohri, and
  Zhong]{awasthi2021calibration}
Pranjal Awasthi, Natalie Frank, Anqi Mao, Mehryar Mohri, and Yutao Zhong.
\newblock Calibration and consistency of adversarial surrogate losses.
\newblock In \emph{Advances in Neural Information Processing Systems},
  2021{\natexlab{a}}.

\bibitem[Awasthi et~al.(2021{\natexlab{b}})Awasthi, Frank, and
  Mohri]{awasthi2021existence}
Pranjal Awasthi, Natalie Frank, and Mehryar Mohri.
\newblock On the existence of the adversarial bayes classifier.
\newblock In \emph{Advances in Neural Information Processing Systems}, pages
  2978--2990, 2021{\natexlab{b}}.

\bibitem[Awasthi et~al.(2021{\natexlab{c}})Awasthi, Mao, Mohri, and
  Zhong]{awasthi2021finer}
Pranjal Awasthi, Anqi Mao, Mehryar Mohri, and Yutao Zhong.
\newblock A finer calibration analysis for adversarial robustness.
\newblock \emph{arXiv preprint arXiv:2105.01550}, 2021{\natexlab{c}}.

\bibitem[Awasthi et~al.(2022{\natexlab{a}})Awasthi, Mao, Mohri, and
  Zhong]{awasthi2022Hconsistency}
Pranjal Awasthi, Anqi Mao, Mehryar Mohri, and Yutao Zhong.
\newblock {${\mathscr H}$}-consistency bounds for surrogate loss minimizers.
\newblock In \emph{International Conference on Machine Learning},
  2022{\natexlab{a}}.

\bibitem[Awasthi et~al.(2022{\natexlab{b}})Awasthi, Mao, Mohri, and
  Zhong]{awasthi2022multi}
Pranjal Awasthi, Anqi Mao, Mehryar Mohri, and Yutao Zhong.
\newblock Multi-class {${\mathscr H}$}-consistency bounds.
\newblock In \emph{Advances in neural information processing systems},
  2022{\natexlab{b}}.

\bibitem[Awasthi et~al.(2023)Awasthi, Mao, Mohri, and
  Zhong]{AwasthiMaoMohriZhong2023theoretically}
Pranjal Awasthi, Anqi Mao, Mehryar Mohri, and Yutao Zhong.
\newblock Theoretically grounded loss functions and algorithms for adversarial
  robustness.
\newblock In \emph{International Conference on Artificial Intelligence and
  Statistics}, pages 10077--10094, 2023.

\bibitem[Awasthi et~al.(2024)Awasthi, Mao, Mohri, and Zhong]{awasthi2024dc}
Pranjal Awasthi, Anqi Mao, Mehryar Mohri, and Yutao Zhong.
\newblock {DC}-programming for neural network optimizations.
\newblock \emph{Journal of Global Optimization}, pages 1--17, 2024.

\bibitem[Bansal et~al.(2021)Bansal, Nushi, Kamar, Horvitz, and
  Weld]{bansal2021most}
Gagan Bansal, Besmira Nushi, Ece Kamar, Eric Horvitz, and Daniel~S Weld.
\newblock Is the most accurate ai the best teammate? optimizing ai for
  teamwork.
\newblock In \emph{Proceedings of the AAAI Conference on Artificial
  Intelligence}, pages 11405--11414, 2021.

\bibitem[Bartlett and Wegkamp(2008)]{bartlett2008classification}
Peter~L Bartlett and Marten~H Wegkamp.
\newblock Classification with a reject option using a hinge loss.
\newblock \emph{Journal of Machine Learning Research}, 9\penalty0 (8), 2008.

\bibitem[Bounsiar et~al.(2007)Bounsiar, Grall, and
  Beauseroy]{BounsiarGrallBeauseroy2007}
Adbenour Bounsiar, Edith Grall, and Pierre Beauseroy.
\newblock Kernel based rejection method for supervised classification.
\newblock In \emph{WASET}, 2007.

\bibitem[Cao et~al.(2022)Cao, Cai, Feng, Gu, Gu, An, Niu, and
  Sugiyama]{caogeneralizing}
Yuzhou Cao, Tianchi Cai, Lei Feng, Lihong Gu, Jinjie Gu, Bo~An, Gang Niu, and
  Masashi Sugiyama.
\newblock Generalizing consistent multi-class classification with rejection to
  be compatible with arbitrary losses.
\newblock In \emph{Advances in neural information processing systems}, 2022.

\bibitem[Cao et~al.(2023)Cao, Mozannar, Feng, Wei, and An]{cao2023defense}
Yuzhou Cao, Hussein Mozannar, Lei Feng, Hongxin Wei, and Bo~An.
\newblock In defense of softmax parametrization for calibrated and consistent
  learning to defer.
\newblock In \emph{Advances in Neural Information Processing Systems}, 2023.

\bibitem[Carlini and Wagner(2017)]{carlini2017towards}
Nicholas Carlini and David Wagner.
\newblock Towards evaluating the robustness of neural networks.
\newblock In \emph{IEEE Symposium on Security and Privacy (SP)}, pages 39--57,
  2017.

\bibitem[Charoenphakdee et~al.(2021)Charoenphakdee, Cui, Zhang, and
  Sugiyama]{charoenphakdee2021classification}
Nontawat Charoenphakdee, Zhenghang Cui, Yivan Zhang, and Masashi Sugiyama.
\newblock Classification with rejection based on cost-sensitive classification.
\newblock In \emph{International Conference on Machine Learning}, pages
  1507--1517, 2021.

\bibitem[Chen et~al.(2024)Chen, Li, Sun, and Wang]{chen2024learning}
Guanting Chen, Xiaocheng Li, Chunlin Sun, and Hanzhao Wang.
\newblock Learning to make adherence-aware advice.
\newblock In \emph{International Conference on Learning Representations}, 2024.

\bibitem[Cheng et~al.(2023)Cheng, Cao, Wang, Wei, An, and
  Feng]{cheng2023regression}
Xin Cheng, Yuzhou Cao, Haobo Wang, Hongxin Wei, Bo~An, and Lei Feng.
\newblock Regression with cost-based rejection.
\newblock In \emph{Advances in Neural Information Processing Systems}, 2023.

\bibitem[Chow(1970)]{chow1970optimum}
C~Chow.
\newblock On optimum recognition error and reject tradeoff.
\newblock \emph{IEEE Transactions on information theory}, 16\penalty0
  (1):\penalty0 41--46, 1970.

\bibitem[Chow(1957)]{Chow1957}
C.K. Chow.
\newblock An optimum character recognition system using decision function.
\newblock \emph{IEEE T. C.}, 1957.

\bibitem[Chzhen et~al.(2021)Chzhen, Denis, Hebiri, and Lorieul]{chzhen2021set}
Evgenii Chzhen, Christophe Denis, Mohamed Hebiri, and Titouan Lorieul.
\newblock Set-valued classification--overview via a unified framework.
\newblock \emph{arXiv preprint arXiv:2102.12318}, 2021.

\bibitem[Cortes et~al.(2016{\natexlab{a}})Cortes, DeSalvo, and
  Mohri]{CortesDeSalvoMohri2016}
Corinna Cortes, Giulia DeSalvo, and Mehryar Mohri.
\newblock Learning with rejection.
\newblock In \emph{International Conference on Algorithmic Learning Theory},
  pages 67--82, 2016{\natexlab{a}}.

\bibitem[Cortes et~al.(2016{\natexlab{b}})Cortes, DeSalvo, and
  Mohri]{CortesDeSalvoMohri2016bis}
Corinna Cortes, Giulia DeSalvo, and Mehryar Mohri.
\newblock Boosting with abstention.
\newblock In \emph{Advances in Neural Information Processing Systems}, pages
  1660--1668, 2016{\natexlab{b}}.

\bibitem[Cortes et~al.(2023)Cortes, DeSalvo, and Mohri]{CortesDeSalvoMohri2023}
Corinna Cortes, Giulia DeSalvo, and Mehryar Mohri.
\newblock Theory and algorithms for learning with rejection in binary
  classification.
\newblock \emph{Annals of Mathematics and Artificial Intelligence}, pages
  1--39, 2023.

\bibitem[Denis et~al.(2022)Denis, Hebiri, Njike, and Siebert]{denis2022active}
Christophe Denis, Mohamed Hebiri, Boris~Ndjia Njike, and Xavier Siebert.
\newblock Active learning algorithm through the lens of rejection arguments.
\newblock \emph{arXiv preprint arXiv:2208.14682}, 2022.

\bibitem[Dubuisson and Masson(1993)]{dubuisson1993statistical}
Bernard Dubuisson and Mylene Masson.
\newblock A statistical decision rule with incomplete knowledge about classes.
\newblock \emph{Pattern recognition}, 26\penalty0 (1):\penalty0 155--165, 1993.

\bibitem[El-Yaniv and Wiener(2012)]{el2012active}
Ran El-Yaniv and Yair Wiener.
\newblock Active learning via perfect selective classification.
\newblock \emph{Journal of Machine Learning Research}, 13\penalty0 (2), 2012.

\bibitem[El-Yaniv et~al.(2010)]{el2010foundations}
Ran El-Yaniv et~al.
\newblock On the foundations of noise-free selective classification.
\newblock \emph{Journal of Machine Learning Research}, 11\penalty0 (5), 2010.

\bibitem[Elkan(2001)]{elkan2001foundations}
Charles Elkan.
\newblock The foundations of cost-sensitive learning.
\newblock In \emph{International joint conference on artificial intelligence},
  pages 973--978, 2001.

\bibitem[Filippova(2020)]{Filippova2020}
Katja Filippova.
\newblock Controlled hallucinations:learning to generate faithfully from noisy
  data.
\newblock In \emph{Findings of EMNLP 2020}, 2020.

\bibitem[Fumera and Roli(2002)]{FumeraRoli2002}
Giorgio Fumera and Fabio Roli.
\newblock Support vector machines with embedded reject option.
\newblock In \emph{ICPR}, 2002.

\bibitem[Fumera et~al.(2000)Fumera, Roli, and Giacinto]{FumeraRoliGiacinto2000}
Giorgio Fumera, Fabio Roli, and Giorgio Giacinto.
\newblock Multiple reject thresholds for improving classification reliability.
\newblock In \emph{ICAPR}, 2000.

\bibitem[Gangrade et~al.(2021)Gangrade, Kag, and
  Saligrama]{gangrade2021selective}
Aditya Gangrade, Anil Kag, and Venkatesh Saligrama.
\newblock Selective classification via one-sided prediction.
\newblock In \emph{International Conference on Artificial Intelligence and
  Statistics}, pages 2179--2187, 2021.

\bibitem[Geifman and El-Yaniv(2017)]{geifman2017selective}
Yonatan Geifman and Ran El-Yaniv.
\newblock Selective classification for deep neural networks.
\newblock In \emph{Advances in neural information processing systems}, 2017.

\bibitem[Geifman and El-Yaniv(2019)]{geifman2019selectivenet}
Yonatan Geifman and Ran El-Yaniv.
\newblock Selectivenet: A deep neural network with an integrated reject option.
\newblock In \emph{International conference on machine learning}, pages
  2151--2159, 2019.

\bibitem[Ghosh et~al.(2017)Ghosh, Kumar, and Sastry]{ghosh2017robust}
Aritra Ghosh, Himanshu Kumar, and P~Shanti Sastry.
\newblock Robust loss functions under label noise for deep neural networks.
\newblock In \emph{Proceedings of the AAAI conference on artificial
  intelligence}, 2017.

\bibitem[Goodfellow et~al.(2014)Goodfellow, Shlens, and
  Szegedy]{goodfellow2014explaining}
Ian~J Goodfellow, Jonathon Shlens, and Christian Szegedy.
\newblock Explaining and harnessing adversarial examples.
\newblock \emph{arXiv preprint arXiv:1412.6572}, 2014.

\bibitem[Grandvalet et~al.(2008)Grandvalet, Keshet, Rakotomamonjy, and
  Canu]{GrandvaletKeshetRakotomamonjyCanu2008}
Yves Grandvalet, Joseph Keshet, Alain Rakotomamonjy, and Stephane Canu.
\newblock Suppport vector machines with a reject option.
\newblock In \emph{NIPS}, 2008.

\bibitem[He et~al.(2016)He, Zhang, Ren, and Sun]{he2016deep}
Kaiming He, Xiangyu Zhang, Shaoqing Ren, and Jian Sun.
\newblock Deep residual learning for image recognition.
\newblock In \emph{Proceedings of the IEEE conference on computer vision and
  pattern recognition}, pages 770--778, 2016.

\bibitem[Herbei and Wegkamp(2005)]{HerbeiWegkamp2005}
Radu Herbei and Marten Wegkamp.
\newblock Classification with reject option.
\newblock \emph{Can. J. Stat.}, 2005.

\bibitem[Krizhevsky(2009)]{Krizhevsky09learningmultiple}
Alex Krizhevsky.
\newblock Learning multiple layers of features from tiny images.
\newblock Technical report, Toronto University, 2009.

\bibitem[Kuznetsov et~al.(2014)Kuznetsov, Mohri, and
  Syed]{KuznetsovMohriSyed2014}
Vitaly Kuznetsov, Mehryar Mohri, and Umar Syed.
\newblock Multi-class deep boosting.
\newblock In \emph{Advances in Neural Information Processing Systems}, pages
  2501--2509, 2014.

\bibitem[Landgrebe et~al.(2005)Landgrebe, Tax, Paclik, and
  Duin]{LandgrebeTaxPaclikDuin2005}
Thomas Landgrebe, David Tax, Pavel Paclik, and Robert Duin.
\newblock Interaction between classification and reject performance for
  distance-based reject-option classifiers.
\newblock \emph{PRL}, 2005.

\bibitem[Le~Capitaine and Frelicot(2010)]{le2010optimum}
Hoel Le~Capitaine and Carl Frelicot.
\newblock An optimum class-rejective decision rule and its evaluation.
\newblock In \emph{International Conference on Pattern Recognition}, pages
  3312--3315, 2010.

\bibitem[Lee et~al.(2004)Lee, Lin, and Wahba]{lee2004multicategory}
Yoonkyung Lee, Yi~Lin, and Grace Wahba.
\newblock Multicategory support vector machines: Theory and application to the
  classification of microarray data and satellite radiance data.
\newblock \emph{Journal of the American Statistical Association}, 99\penalty0
  (465):\penalty0 67--81, 2004.

\bibitem[Li et~al.(2008)Li, Littman, and Walsh]{li2008knows}
Lihong Li, Michael~L Littman, and Thomas~J Walsh.
\newblock Knows what it knows: a framework for self-aware learning.
\newblock In \emph{International conference on Machine learning}, pages
  568--575, 2008.

\bibitem[Li et~al.(2024)Li, Liu, Sun, and Wang]{li2024no}
Xiaocheng Li, Shang Liu, Chunlin Sun, and Hanzhao Wang.
\newblock When no-rejection learning is optimal for regression with rejection.
\newblock In \emph{International Conference on Artificial Intelligence and
  Statistics}, 2024.

\bibitem[Long and Servedio(2013)]{long2013consistency}
Phil Long and Rocco Servedio.
\newblock Consistency versus realizable {H}-consistency for multiclass
  classification.
\newblock In \emph{International Conference on Machine Learning}, pages
  801--809, 2013.

\bibitem[Loshchilov and Hutter(2016)]{loshchilov2016sgdr}
Ilya Loshchilov and Frank Hutter.
\newblock {SGDR}: Stochastic gradient descent with warm restarts.
\newblock \emph{arXiv preprint arXiv:1608.03983}, 2016.

\bibitem[Madras et~al.(2018)Madras, Pitassi, and Zemel]{madras2018predict}
David Madras, Toni Pitassi, and Richard Zemel.
\newblock Predict responsibly: improving fairness and accuracy by learning to
  defer.
\newblock In \emph{Advances in Neural Information Processing Systems}, 2018.

\bibitem[Madry et~al.(2017)Madry, Makelov, Schmidt, Tsipras, and
  Vladu]{madry2017towards}
Aleksander Madry, Aleksandar Makelov, Ludwig Schmidt, Dimitris Tsipras, and
  Adrian Vladu.
\newblock Towards deep learning models resistant to adversarial attacks.
\newblock \emph{arXiv preprint arXiv:1706.06083}, 2017.

\bibitem[Mao et~al.(2023{\natexlab{a}})Mao, Mohri, Mohri, and
  Zhong]{mao2023two}
Anqi Mao, Christopher Mohri, Mehryar Mohri, and Yutao Zhong.
\newblock Two-stage learning to defer with multiple experts.
\newblock In \emph{Advances in Neural Information Processing Systems},
  2023{\natexlab{a}}.

\bibitem[Mao et~al.(2023{\natexlab{b}})Mao, Mohri, and
  Zhong]{MaoMohriZhong2023characterization}
Anqi Mao, Mehryar Mohri, and Yutao Zhong.
\newblock {H}-consistency bounds: Characterization and extensions.
\newblock In \emph{Advances in Neural Information Processing Systems},
  2023{\natexlab{b}}.

\bibitem[Mao et~al.(2023{\natexlab{c}})Mao, Mohri, and
  Zhong]{MaoMohriZhong2023cross}
Anqi Mao, Mehryar Mohri, and Yutao Zhong.
\newblock Cross-entropy loss functions: Theoretical analysis and applications.
\newblock In \emph{International conference on Machine learning},
  2023{\natexlab{c}}.

\bibitem[Mao et~al.(2023{\natexlab{d}})Mao, Mohri, and
  Zhong]{MaoMohriZhong2023ranking}
Anqi Mao, Mehryar Mohri, and Yutao Zhong.
\newblock {H}-consistency bounds for pairwise misranking loss surrogates.
\newblock In \emph{International conference on Machine learning},
  2023{\natexlab{d}}.

\bibitem[Mao et~al.(2023{\natexlab{e}})Mao, Mohri, and
  Zhong]{MaoMohriZhong2023rankingabs}
Anqi Mao, Mehryar Mohri, and Yutao Zhong.
\newblock Ranking with abstention.
\newblock In \emph{ICML 2023 Workshop The Many Facets of Preference-Based
  Learning}, 2023{\natexlab{e}}.

\bibitem[Mao et~al.(2023{\natexlab{f}})Mao, Mohri, and
  Zhong]{MaoMohriZhong2023structured}
Anqi Mao, Mehryar Mohri, and Yutao Zhong.
\newblock Structured prediction with stronger consistency guarantees.
\newblock In \emph{Advances in Neural Information Processing Systems},
  2023{\natexlab{f}}.

\bibitem[Mao et~al.(2024{\natexlab{a}})Mao, Mohri, and
  Zhong]{MaoMohriZhong2024deferral}
Anqi Mao, Mehryar Mohri, and Yutao Zhong.
\newblock Principled approaches for learning to defer with multiple experts.
\newblock In \emph{International Symposium on Artificial Intelligence and
  Mathematics}, 2024{\natexlab{a}}.

\bibitem[Mao et~al.(2024{\natexlab{b}})Mao, Mohri, and
  Zhong]{MaoMohriZhong2024score}
Anqi Mao, Mehryar Mohri, and Yutao Zhong.
\newblock Theoretically grounded loss functions and algorithms for score-based
  multi-class abstention.
\newblock In \emph{International Conference on Artificial Intelligence and
  Statistics}, 2024{\natexlab{b}}.

\bibitem[Mao et~al.(2024{\natexlab{c}})Mao, Mohri, and Zhong]{mao2024h}
Anqi Mao, Mehryar Mohri, and Yutao Zhong.
\newblock {${\mathscr H}$}-consistency guarantees for regression.
\newblock \emph{arXiv preprint arXiv:2403.19480}, 2024{\natexlab{c}}.

\bibitem[Mao et~al.(2024{\natexlab{d}})Mao, Mohri, and
  Zhong]{mao2024regression}
Anqi Mao, Mehryar Mohri, and Yutao Zhong.
\newblock Regression with multi-expert deferral.
\newblock \emph{arXiv preprint arXiv:2403.19494}, 2024{\natexlab{d}}.

\bibitem[Mao et~al.(2024{\natexlab{e}})Mao, Mohri, and Zhong]{mao2024top}
Anqi Mao, Mehryar Mohri, and Yutao Zhong.
\newblock Top-$ k $ classification and cardinality-aware prediction.
\newblock \emph{arXiv preprint arXiv:2403.19625}, 2024{\natexlab{e}}.

\bibitem[Maynez et~al.(2020)Maynez, Narayan, Bohnet, and McDonald]{maynez2020}
Joshua Maynez, Shashi Narayan, Bernd Bohnet, and Ryan McDonald.
\newblock On faithfulness and factuality in abstractive summarization.
\newblock In \emph{Proceedings of the 58th Annual Meeting of the Association
  for Computational Linguistics}, pages 1906--1919, Online, July 2020.
  Association for Computational Linguistics.
\newblock \doi{10.18653/v1/2020.acl-main.173}.

\bibitem[Melvin et~al.(2008)Melvin, Weston, Leslie, and Noble]{Melvin2008}
Iain Melvin, Jason Weston, Christina~S. Leslie, and William~S. Noble.
\newblock Combining classifiers for improved classification of proteins from
  sequence or structure.
\newblock \emph{BMCB}, 2008.

\bibitem[Mohri et~al.(2024)Mohri, Andor, Choi, Collins, Mao, and
  Zhong]{MohriAndorChoiCollinsMaoZhong2024learning}
Christopher Mohri, Daniel Andor, Eunsol Choi, Michael Collins, Anqi Mao, and
  Yutao Zhong.
\newblock Learning to reject with a fixed predictor: Application to
  decontextualization.
\newblock In \emph{International Conference on Learning Representations}, 2024.

\bibitem[Mozannar and Sontag(2020)]{mozannar2020consistent}
Hussein Mozannar and David Sontag.
\newblock Consistent estimators for learning to defer to an expert.
\newblock In \emph{International Conference on Machine Learning}, pages
  7076--7087, 2020.

\bibitem[Mozannar et~al.(2023)Mozannar, Lang, Wei, Sattigeri, Das, and
  Sontag]{pmlr-v206-mozannar23a}
Hussein Mozannar, Hunter Lang, Dennis Wei, Prasanna Sattigeri, Subhro Das, and
  David Sontag.
\newblock Who should predict? exact algorithms for learning to defer to humans.
\newblock In \emph{International Conference on Artificial Intelligence and
  Statistics}, pages 10520--10545, 2023.

\bibitem[Narasimhan et~al.(2022)Narasimhan, Jitkrittum, Menon, Rawat, and
  Kumar]{narasimhanpost}
Harikrishna Narasimhan, Wittawat Jitkrittum, Aditya~Krishna Menon, Ankit~Singh
  Rawat, and Sanjiv Kumar.
\newblock Post-hoc estimators for learning to defer to an expert.
\newblock In \emph{Advances in Neural Information Processing Systems}, 2022.

\bibitem[Narasimhan et~al.(2023)Narasimhan, Menon, Jitkrittum, and
  Kumar]{narasimhan2023learning}
Harikrishna Narasimhan, Aditya~Krishna Menon, Wittawat Jitkrittum, and Sanjiv
  Kumar.
\newblock Learning to reject meets ood detection: Are all abstentions created
  equal?
\newblock \emph{arXiv preprint arXiv:2301.12386}, 2023.

\bibitem[Nesterov(1983)]{nesterov1983method}
Yurii~E Nesterov.
\newblock A method for solving the convex programming problem with convergence
  rate $o(1/k^2)$.
\newblock \emph{Dokl. akad. nauk Sssr}, 269:\penalty0 543--547, 1983.

\bibitem[Netzer et~al.(2011)Netzer, Wang, Coates, Bissacco, Wu, and
  Ng]{Netzer2011}
Yuval Netzer, Tao Wang, Adam Coates, Alessandro Bissacco, Bo~Wu, and Andrew~Y
  Ng.
\newblock Reading digits in natural images with unsupervised feature learning.
\newblock In \emph{Advances in Neural Information Processing Systems}, 2011.

\bibitem[Ni et~al.(2019)Ni, Charoenphakdee, Honda, and Sugiyama]{NiCHS19}
Chenri Ni, Nontawat Charoenphakdee, Junya Honda, and Masashi Sugiyama.
\newblock On the calibration of multiclass classification with rejection.
\newblock In \emph{Advances in Neural Information Processing Systems}, pages
  2582--2592, 2019.

\bibitem[Okati et~al.(2021)Okati, De, and Rodriguez]{okati2021differentiable}
Nastaran Okati, Abir De, and Manuel Rodriguez.
\newblock Differentiable learning under triage.
\newblock \emph{Advances in Neural Information Processing Systems},
  34:\penalty0 9140--9151, 2021.

\bibitem[Pereira and Pires(2005)]{SantosPires2005}
Carla~S Pereira and Ana Pires.
\newblock On optimal reject rules and {ROC} curves.
\newblock \emph{PRL}, 2005.

\bibitem[Pietraszek(2005)]{Pietraszek2005}
Tadeusz Pietraszek.
\newblock Optimizing abstaining classifiers using {ROC}.
\newblock In \emph{ICML}, 2005.

\bibitem[Puchkin and Zhivotovskiy(2021)]{puchkin2021exponential}
Nikita Puchkin and Nikita Zhivotovskiy.
\newblock Exponential savings in agnostic active learning through abstention.
\newblock In \emph{Conference on Learning Theory}, pages 3806--3832, 2021.

\bibitem[Raghu et~al.(2019{\natexlab{a}})Raghu, Blumer, Corrado, Kleinberg,
  Obermeyer, and Mullainathan]{raghu2019algorithmic}
Maithra Raghu, Katy Blumer, Greg Corrado, Jon Kleinberg, Ziad Obermeyer, and
  Sendhil Mullainathan.
\newblock The algorithmic automation problem: Prediction, triage, and human
  effort.
\newblock \emph{arXiv preprint arXiv:1903.12220}, 2019{\natexlab{a}}.

\bibitem[Raghu et~al.(2019{\natexlab{b}})Raghu, Blumer, Sayres, Obermeyer,
  Kleinberg, Mullainathan, and Kleinberg]{raghu2019direct}
Maithra Raghu, Katy Blumer, Rory Sayres, Ziad Obermeyer, Bobby Kleinberg,
  Sendhil Mullainathan, and Jon Kleinberg.
\newblock Direct uncertainty prediction for medical second opinions.
\newblock In \emph{International Conference on Machine Learning}, pages
  5281--5290, 2019{\natexlab{b}}.

\bibitem[Ramaswamy et~al.(2018)Ramaswamy, Tewari, and
  Agarwal]{ramaswamy2018consistent}
Harish~G Ramaswamy, Ambuj Tewari, and Shivani Agarwal.
\newblock Consistent algorithms for multiclass classification with an abstain
  option.
\newblock \emph{Electronic Journal of Statistics}, 12\penalty0 (1):\penalty0
  530--554, 2018.

\bibitem[Reid and Williamson(2010)]{reid2010composite}
Mark~D Reid and Robert~C Williamson.
\newblock Composite binary losses.
\newblock \emph{The Journal of Machine Learning Research}, 11:\penalty0
  2387--2422, 2010.

\bibitem[Schreuder and Chzhen(2021)]{schreuder2021classification}
Nicolas Schreuder and Evgenii Chzhen.
\newblock Classification with abstention but without disparities.
\newblock In \emph{Uncertainty in Artificial Intelligence}, pages 1227--1236.
  PMLR, 2021.

\bibitem[Steinwart(2007)]{steinwart2007compare}
Ingo Steinwart.
\newblock How to compare different loss functions and their risks.
\newblock \emph{Constructive Approximation}, 26\penalty0 (2):\penalty0
  225--287, 2007.

\bibitem[Tax and Duin(2008)]{tax2008growing}
David~MJ Tax and Robert~PW Duin.
\newblock Growing a multi-class classifier with a reject option.
\newblock \emph{Pattern Recognition Letters}, 29\penalty0 (10):\penalty0
  1565--1570, 2008.

\bibitem[Tortorella(2001)]{Tortorella2001}
Francesco Tortorella.
\newblock An optimal reject rule for binary classifiers.
\newblock In \emph{ICAPR}, 2001.

\bibitem[Tsipras et~al.(2018)Tsipras, Santurkar, Engstrom, Turner, and
  Madry]{tsipras2018robustness}
Dimitris Tsipras, Shibani Santurkar, Logan Engstrom, Alexander Turner, and
  Aleksander Madry.
\newblock Robustness may be at odds with accuracy.
\newblock \emph{arXiv preprint arXiv:1805.12152}, 2018.

\bibitem[Verma and Nalisnick(2022)]{verma2022calibrated}
Rajeev Verma and Eric Nalisnick.
\newblock Calibrated learning to defer with one-vs-all classifiers.
\newblock In \emph{International Conference on Machine Learning}, pages
  22184--22202, 2022.

\bibitem[Verma et~al.(2023)Verma, Barrej{\'o}n, and
  Nalisnick]{verma2023learning}
Rajeev Verma, Daniel Barrej{\'o}n, and Eric Nalisnick.
\newblock Learning to defer to multiple experts: Consistent surrogate losses,
  confidence calibration, and conformal ensembles.
\newblock In \emph{International Conference on Artificial Intelligence and
  Statistics}, pages 11415--11434, 2023.

\bibitem[Wei et~al.(2022)Wei, Tay, Bommasani, Raffel, Zoph, Borgeaud, Yogatama,
  Bosma, Zhou, Metzler, Chi, Hashimoto, Vinyals, Liang, Dean, and
  Fedus]{WeiEtAl2022}
Jason Wei, Yi~Tay, Rishi Bommasani, Colin Raffel, Barret Zoph, Sebastian
  Borgeaud, Dani Yogatama, Maarten Bosma, Denny Zhou, Donald Metzler, Ed~H.
  Chi, Tatsunori Hashimoto, Oriol Vinyals, Percy Liang, Jeff Dean, and William
  Fedus.
\newblock Emergent abilities of large language models.
\newblock \emph{CoRR}, abs/2206.07682, 2022.

\bibitem[Wiener and El-Yaniv(2011)]{wiener2011agnostic}
Yair Wiener and Ran El-Yaniv.
\newblock Agnostic selective classification.
\newblock In \emph{Advances in neural information processing systems}, 2011.

\bibitem[Wiener and El-Yaniv(2015)]{wiener2015agnostic}
Yair Wiener and Ran El-Yaniv.
\newblock Agnostic pointwise-competitive selective classification.
\newblock \emph{Journal of Artificial Intelligence Research}, 52:\penalty0
  171--201, 2015.

\bibitem[Wiener et~al.(2015)Wiener, Hanneke, and
  El-Yaniv]{wiener2015compression}
Yair Wiener, Steve Hanneke, and Ran El-Yaniv.
\newblock A compression technique for analyzing disagreement-based active
  learning.
\newblock \emph{J. Mach. Learn. Res.}, 16:\penalty0 713--745, 2015.

\bibitem[Wilder et~al.(2021)Wilder, Horvitz, and Kamar]{wilder2021learning}
Bryan Wilder, Eric Horvitz, and Ece Kamar.
\newblock Learning to complement humans.
\newblock In \emph{International Joint Conferences on Artificial Intelligence},
  pages 1526--1533, 2021.

\bibitem[Yuan and Wegkamp(2010)]{yuan2010classification}
Ming Yuan and Marten Wegkamp.
\newblock Classification methods with reject option based on convex risk
  minimization.
\newblock \emph{Journal of Machine Learning Research}, 11\penalty0 (1), 2010.

\bibitem[Yuan and Wegkamp(2011)]{WegkampYuan2011}
Ming Yuan and Marten Wegkamp.
\newblock {SVM}s with a reject option.
\newblock In \emph{Bernoulli}, 2011.

\bibitem[Zagoruyko and Komodakis(2016)]{zagoruyko2016wide}
Sergey Zagoruyko and Nikos Komodakis.
\newblock Wide residual networks.
\newblock \emph{arXiv preprint arXiv:1605.07146}, 2016.

\bibitem[Zhang and Chaudhuri(2016)]{zhang2016extended}
Chicheng Zhang and Kamalika Chaudhuri.
\newblock The extended littlestone’s dimension for learning with mistakes and
  abstentions.
\newblock In \emph{Conference on Learning Theory}, pages 1584--1616, 2016.

\bibitem[Zhang and Agarwal(2020)]{zhang2020bayes}
Mingyuan Zhang and Shivani Agarwal.
\newblock Bayes consistency vs. {H}-consistency: The interplay between
  surrogate loss functions and the scoring function class.
\newblock In \emph{Advances in Neural Information Processing Systems}, 2020.

\bibitem[Zhang and Sabuncu(2018)]{zhang2018generalized}
Zhilu Zhang and Mert Sabuncu.
\newblock Generalized cross entropy loss for training deep neural networks with
  noisy labels.
\newblock In \emph{Advances in neural information processing systems}, 2018.

\bibitem[Zheng et~al.(2023)Zheng, Wu, Bao, Cao, Li, and
  Zhu]{zheng2023revisiting}
Chenyu Zheng, Guoqiang Wu, Fan Bao, Yue Cao, Chongxuan Li, and Jun Zhu.
\newblock Revisiting discriminative vs. generative classifiers: Theory and
  implications.
\newblock In \emph{International Conference on Machine Learning}, 2023.

\bibitem[Zhu and Nowak(2022)]{zhu2022efficient}
Yinglun Zhu and Robert Nowak.
\newblock Efficient active learning with abstention.
\newblock \emph{arXiv preprint arXiv:2204.00043}, 2022.

\bibitem[Ziyin et~al.(2019)Ziyin, Wang, Liang, Salakhutdinov, Morency, and
  Ueda]{ziyin2019deep}
Liu Ziyin, Zhikang Wang, Paul~Pu Liang, Ruslan Salakhutdinov, Louis-Philippe
  Morency, and Masahito Ueda.
\newblock Deep gamblers: Learning to abstain with portfolio theory.
\newblock \emph{arXiv preprint arXiv:1907.00208}, 2019.

\end{thebibliography}

\newpage
\appendix

\renewcommand{\contentsname}{Contents of Appendix}
\tableofcontents
\addtocontents{toc}{\protect\setcounter{tocdepth}{3}} 
\clearpage


\section{Related work}
\label{app:related-work}

Several broad methods for learning with abstention can be distinguished
in the literature: \emph{confidence-based methods}, which consist of
abstaining when the score returned by a pre-trained model falls below
some threshold
\citep{Chow1957,chow1970optimum,bartlett2008classification,
  yuan2010classification,WegkampYuan2011,ramaswamy2018consistent,NiCHS19}; \emph{selective
classification}, which analyzes a set-up with a \emph{predictor} and a
\emph{selector} and defines a selection risk or loss normalized by the expected selection or coverage
\citep{el2010foundations,wiener2011agnostic,el2012active,
  wiener2015agnostic,geifman2017selective,geifman2019selectivenet}; a
\emph{predictor-rejector formulation}, which is based on learning both
a \emph{predictor} and a \emph{rejector}, each from a different family
of functions, and that takes into account explicitly the cost of
abstention
\citep{CortesDeSalvoMohri2016,CortesDeSalvoMohri2016bis,CortesDeSalvoMohri2023,cheng2023regression,MohriAndorChoiCollinsMaoZhong2024learning,li2024no}; and a more
recent \emph{score-based formulation} that consists of augmenting the
multi-class categories with a rejection label and of abstaining when
the score assigned to the rejection label is the highest
\citep{mozannar2020consistent,caogeneralizing,MaoMohriZhong2024score}.  Another problem closely related to
abstention is that of \emph{deferring} to an alternative model, or
even to a human in some instances. This can also be considered as a
special case of the general abstention scenario and tackled
in a similar way
\citep{madras2018predict,raghu2019algorithmic,mozannar2020consistent,
  okati2021differentiable,wilder2021learning,verma2022calibrated,
  narasimhanpost,verma2023learning,mao2023two,cao2023defense,MaoMohriZhong2024deferral,chen2024learning,mao2024regression}.

The study of confidence-based methods was initiated by \citet{Chow1957,chow1970optimum} who explored the
trade-off between error rate and rejection rate, and also presented an
analysis of the Bayes optimal decision in this context. Later,
\citet{FumeraRoliGiacinto2000} proposed a multiple thresholds rule for
situations where a posteriori probabilities were impacted by
errors. \citet{Tortorella2001} introduced an optimal rejection rule
for binary classifiers, relying on the Receiver Operating
Characteristic (ROC) curve. Additionally, \citet{SantosPires2005}
compared their methodology with that of \citet{chow1970optimum}. Numerous publications have proposed various rejection techniques to
reduce the misclassification rate, though without theoretical analysis
\citep{FumeraRoli2002, Pietraszek2005, BounsiarGrallBeauseroy2007,
  LandgrebeTaxPaclikDuin2005, Melvin2008}. \citet{HerbeiWegkamp2005}
examined classification with a rejection option involving a cost and
provided excess error bounds for these ternary functions.  \citet{bartlett2008classification}
developed a loss function for this scenario that takes into account
the abstention cost $c$. They proposed learning a predictor using a
\emph{double hinge loss} and demonstrated its consistency
benefits. This approach has been further explored in several
subsequent publications \citep{GrandvaletKeshetRakotomamonjyCanu2008,
  yuan2010classification,WegkampYuan2011}. \citet{ramaswamy2018consistent} further examined
confidence-based abstention in multi-class classification, showing
that certain multi-class hinge loss formulations and a newly
constructed polyhedral binary-encoded predictions (BEP) surrogate loss
are Bayes-consistent. \citet{charoenphakdee2021classification}
suggested a cost-sensitive approach for multi-class abstention by
breaking down the multi-class problem into multiple binary
cost-sensitive classification problems
\citep{elkan2001foundations}. They introduced a family of
cost-sensitive one-versus-all surrogate losses, which are
Bayes-consistent in that context. \citet{narasimhan2023learning} investigated the connection between learning with abstention and out-of-distribution detection. They developed a plug-in method aimed at approximating the Bayes-optimal classifier, and demonstrated its application in the context of learning an out-of-distribution (OOD) aware classifier.

Selective classification methods were introduced by \citet{el2010foundations} who
investigated the trade-off between classifier coverage and
accuracy. In a follow-up study, \citet{wiener2011agnostic} developed a
strategy for learning a specific kind of selective classification
called weakly optimal, which has a diminishing rejection rate under
certain Bernstein-type conditions. Many successful connections to
selective classification have been established, including active learning
\citep{el2012active,wiener2015compression,wiener2015agnostic,puchkin2021exponential,denis2022active,zhu2022efficient},
multi-class rejection
\citep{tax2008growing,dubuisson1993statistical,le2010optimum},
reinforcement learning \citep{li2008knows}, online learning
\citep{zhang2016extended}, modern confidence-based rejection methods
\citep{geifman2017selective}, neural network architectures  \citep{geifman2019selectivenet}, loss functions based on
gambling's doubling rate \citep{ziyin2019deep}, disparity-free
approaches \citep{schreuder2021classification}, and the abstention
problem in a "confidence set" framework
\citep{gangrade2021selective,chzhen2021set}.

The predictor-rejector formulation was advocated by \citet*{CortesDeSalvoMohri2016} who contended that
confidence-based abstention is generally suboptimal, unless the
learned predictor is the Bayes classifier. They demonstrated that, in
most cases, no threshold-based abstention can achieve the desired
outcome. They proposed a new abstention framework that involves
learning both a predictor $h$ and a rejector $r$
\emph{simultaneously}, which can generally differ from a
threshold-based function. They defined a predictor-rejector
formulation loss function for the pair $(h, r)$, considering the
abstention cost $c$. The authors provided Rademacher complexity-based
generalization bounds for this learning problem and proposed various
surrogate loss functions for the binary classification abstention
loss. They demonstrated that these surrogate losses offered
consistency guarantees and developed algorithms based on these
surrogate losses, which empirically outperformed confidence-based
abstention benchmarks. This work led to several follow-up studies,
including a theoretical and algorithmic investigation of boosting with
abstention \citep{CortesDeSalvoMohri2016bis} and an analysis of
extending the results to a multi-class setting \citep{NiCHS19}. These
authors acknowledged the difficulty in designing calibrated or
Bayes-consistent surrogate losses based on the predictor-rejector
abstention by \citet{CortesDeSalvoMohri2016} and left it as an open
question. Furthermore, \cite{cheng2023regression} applied this framework in the context of regression with abstention, introducing Bayes-consistent surrogate losses.
\citet{MohriAndorChoiCollinsMaoZhong2024learning} examined the framework in the scenario of learning with a fixed predictor, where they proposed novel algorithms for decontextualization tasks. Additionally, \citet{li2024no} studied the Bayes-consistency of
no-rejection learning for regression with
abstention. 

\citet{mozannar2020consistent} introduced an alternative
\emph{score-based formulation} for multi-class abstention. In this
approach, besides the standard scoring functions associated with each
label, a new scoring function is linked to a new rejection
label. Rejection occurs when the score assigned to the rejection label
exceeds other scores, implicitly defining the rejector through this
specific rule. The authors proposed a surrogate loss for their method
based on cross-entropy (logistic loss with softmax applied to neural
network outputs), which they demonstrated to be Bayes-consistent. Building upon this work,
\citet{caogeneralizing} presented a more comprehensive collection of
Bayes-consistent surrogate losses for the score-based
formulation. These surrogate losses can be constructed using any
consistent loss function for the standard multi-class classification
problem. More recently, \citet{MaoMohriZhong2024score} gave an extensive analysis of surrogate losses for the score-based formulation supported by \emph{$\sH$-consistency bounds}. These are strong non-asymptotic and hypothesis set-specific consistency guarantees introduced by \citep{awasthi2022Hconsistency,awasthi2022multi}. They were later studied for cross-entropy-like loss functions \citep{MaoMohriZhong2023cross,zheng2023revisiting,MaoMohriZhong2023characterization}, ranking loss functions \citep{MaoMohriZhong2023ranking,MaoMohriZhong2023rankingabs}, structured prediction losses \citep{MaoMohriZhong2023structured}, regression losses \citep{mao2024regression}, and top-$k$ classification loss functions \citep{mao2024top}.
Additionally, these guarantees have also been used by \citep{AwasthiMaoMohriZhong2023theoretically,MaoMohriZhong2023cross} in the study of adversarial robustness \citep{goodfellow2014explaining,madry2017towards,tsipras2018robustness,carlini2017towards,awasthi2021calibration,awasthi2021finer,awasthi2021existence,awasthi2024dc}.

The problem of learning to defer is closely related to our study and
can be seen as a specific instance of learning with
abstention. Several recent publications have investigated this
problem, including \citep{madras2018predict, raghu2019algorithmic,
  raghu2019direct, mozannar2020consistent, okati2021differentiable,
  wilder2021learning, bansal2021most, verma2022calibrated,
  narasimhanpost, verma2023learning,mao2023two,cao2023defense,MaoMohriZhong2024deferral,chen2024learning,mao2024regression}. Confidence-based methods for
deferral decisions were examined by \citet{raghu2019direct,
  wilder2021learning, bansal2021most}, but these methods may not be
optimal for low capital models \citep{CortesDeSalvoMohri2016}. To
address this issue, \citet{mozannar2020consistent} proposed a
cost-sensitive logistic loss and \citet{verma2022calibrated} proposed
a cost-sensitive one-versus-all proper composite loss
\citep{reid2010composite}, both in the score-based
formulation. \citet{verma2023learning} generalized the surrogate loss
in \citep{verma2022calibrated} to the deferral setting with multiple
experts. Additionally,  \citet{MaoMohriZhong2024deferral} proposed a novel and broader family of surrogate losses derived from first principles for this context and proved $\sH$-consistency bounds for these loss functions.
Recently, \citet{narasimhanpost} pointed out that the
existing surrogate losses for learning to defer
\citep{mozannar2020consistent, verma2022calibrated} may underfit in
some practical settings and proposed a post-hoc correction for these
loss functions. Furthermore, \citet{mao2023two} explored a two-stage approach for learning to defer with multiple experts. In this approach, a predictor is initially learned using a standard loss function like cross-entropy, followed by the learning of a deferral function. They proposed a new family of surrogate losses and algorithms tailored for this important scenario, and provided $\sH$-consistency bounds guarantees for them. \citet{cao2023defense} proposed a new Bayes-consistent and asymmetric softmax-based surrogate for yielding valid estimates while avoiding the issue of unboundedness. \citet{chen2024learning} integrated deferral into a sequential decision-making model, resulting in enhanced theoretical convergence and better empirical performance. \cite{mao2024regression} formulated the problem of regression with multiple experts and proposed novel $\sH$-consistent surrogate losses for that problem.

We will be particularly interested in the predictor-rejector
formulation, which explicitly models the cost of abstention. The
selective classification of \citet{el2010foundations} is also
interesting, but it does not explicitly factor in the cost $c$ and is
based on a distinct objective. Confidence-based methods are also very
natural and straightforward, but they may fail when the predictor is
not calibrated, a property that often does not hold. Additionally,
they have been shown to be suboptimal when the predictor differs from
the Bayes classifier \citep{CortesDeSalvoMohri2016}. The score-based
formulation \citep{mozannar2020consistent} admits very common
properties and also explicitly takes into account the rejection cost
$c$.  We will compare the predictor-rejector formulation with the
score-based one. We will show via an example that the
predictor-rejector is more natural in some instances and will also
compare the two formulations in our experiments.
We further elaborate on the difference
between the two formulations in Appendix~\ref{app:difference}.

\section{Remarks on some key results}
\label{app:remark}

\begin{remark}
\label{remark:spcific-loss-bound}
When the best-in-class error coincides with the Bayes error $\sE^*_{\sfL}(\sH, \sR) = \sE^*_{\sfL}\paren*{\sH_{\rm{all}}, \sR_{\rm{all}}}$, the minimizability gaps $\sM_{\sfL}(\sH,\sR)$ vanish. In those cases, the  $(\sH,\sR)$-consistency bound in Theorem~\ref{Thm:spcific-loss-bound} guarantees that when the surrogate estimation error $\sE_{\sfL}(h, r) -  \sE_{\sfL}^*(\sH,\sR)$ is optimized up to $\e$, the estimation error of the abstention loss $ \sE_{\labs}(h, r) - \sE_{\labs}^*(\sH,\sR)$ is upper bounded by $\Gamma(\e)$. For all the three loss functions, when $\e$ is sufficiently small, the
dependence of $\Gamma$ on $\e$ exhibits a square root
relationship. However, if this is not the case, the dependence becomes
linear. Note that the dependence is subject to the number of classes $n$ for $\ell = \ell_{\rm{mae}}$ and
$\ell =
\ell_{\rho-\mathrm{hinge}}$.
\end{remark}

\begin{remark}
\label{remark:bound-general-second-step}
When the best-in-class error coincides with the Bayes error $\sE^*_{\ell}(\sR) = \sE^*_{\ell}\paren*{\sR_{\rm{all}}}$ for $\ell = \ell_{\Phi, h}$ and $\ell = \ell_{\mathrm{abs}, h}$, the minimizability gaps $\sM_{\ell_{\Phi, h}}(\sR)$ and $\sM_{\ell_{\mathrm{abs}, h}}(\sR)$ vanish. In those cases, the $\sR$-consistency bound in Theorem~\ref{Thm:bound-general-second-step} guarantees that when the surrogate estimation error $\sE_{\ell_{\Phi, h}}(r) -  \sE_{\ell_{\Phi, h}}^*(\sR)$ is optimized up to $\e$, the target estimation error $ \sE_{\ell_{\mathrm{abs}, h}}(r) - \sE_{\ell_{\mathrm{abs}, h}}^*(\sR)$ is upper bounded by $\Gamma(\frac{\e}{c})$.
\end{remark}

\begin{remark}
\label{remark:tsr}
Corollary~\ref{cor:tsr} shows that $\ell_{\Phi, h}$ admits an excess error bound with respect to $\ell_{\mathrm{abs}, h}$ with functional form $\Gamma(\frac{\cdot}{c})$ when $\Phi$ admits an excess error bound with respect to $\ell_{0-1}^{\rm{binary}}$ with functional form $\Gamma(\cdot)$.
\end{remark}


\begin{remark}
\label{remark:bound-general-two-step}
Note that the minimizability gaps vanish when $\sH$ and $\sR$ are
families of all measurable functions or when they include the Bayes
predictor and rejector. In their absence,
Theorem~\ref{Thm:bound-general-two-step} shows that if the estimation
prediction loss $(\sE_{\ell}(h)-\sE_{\ell}^*(\sH))$ is reduced to
$\e_1$ and the estimation rejection loss
$(\sE_{\ell_{\Phi,h}}(r)-\sE_{\ell_{\Phi,h}}^*(\sR))$ to $\e_2$, then
the abstention estimation loss $(\sE_{\labs}(h, r) -
\sE_{\labs}^*(\sH, \sR))$ is, up to constant factors, bounded by
$\Gamma_1(\e_1) + \Gamma_2(\e_2)$.
\end{remark}

\begin{remark}
\label{remark:tshr}
Corollary~\ref{cor:tshr} shows that the proposed two-stage approach
admits an excess error bound with respect to $\labs$ with functional
form $\Gamma_1(\cdot) + (1 + c)\Gamma_2(\frac{\cdot}{c})$ when $\ell$
admits an excess error bound with respect to $\ell_{0-1}$ with
functional form $\Gamma_1(\cdot)$ and $\Phi$ admits an excess error
bound with respect to $\ell_{0-1}^{\rm{binary}}$ with functional form
$\Gamma_2(\cdot)$.
\end{remark}

\section{Significance of two-stage formulation compared
  with single-stage losses}
\label{app:two-stage}

Here, we wish to further highlight the significance of our findings
regarding the two-stage formulation. The $(\sH, \sR)$-consistency
bounds we established for this scenario directly motivate an algorithm
for a crucial scenario. As already indicated, in applications, often a
prediction function is already available and has been trained using a
standard loss function such as cross-entropy. Training may take days
or months for some large models. The cost of a one-stage approach,
which involves "retraining" to find a pair $(h, r)$ with a new $h$,
can thus be prohibitive. Instead, we demonstrate that a rejector $r$
can be learned using a suitable surrogate loss function based on the
existing predictor $h$ and that the solution formed by the existing
$h$ and this rejector $r$ benefits from $(\sH, \sR)$-consistency
bounds.

Both our one-stage and two-stage solutions using our surrogate losses
benefit from strong $(\sH, \sR)$-consistency bounds: in the limit of
large samples, both methods, one-stage and two-stage converge to the
same joint minimizer of the target abstention loss. However, as
already emphasized, the two-stage approach is advantageous in some
scenarios where a predictor $h$ is already available.  Moreover, the
two-stage solution is more beneficial from the optimization point of
view: the first-stage optimization can be standard and based on say
cross-entropy, and the second stage is based on a loss function
\eqref{eq:ell-Phi-h} that is straightforward to minimize. In contrast,
the one-stage minimization with the MAE loss is known to be more
difficult, see \citep{zhang2018generalized}. In
Section~\ref{sec:experiments}, our empirical results show a more
favorable performance for the two-stage solution, which we believe
reflects this difference in optimization.

\section{Difference between predictor-rejector and score-based formulations}
\label{app:difference}
Here, we hope to further emphasize our contributions by pointing out that the score-based formulation does not offer a direct loss function applicable to the predictor-rejector formulation. It is important to emphasize that the hypothesis set used in the score-based setting constitutes a subset of real-valued functions defined over $ \sX \times \tilde \sY$, where $\tilde \sY$ is the original label set $\sY$ augmented with an additional label corresponding to rejection. In contrast, the predictor-rejector function involves selecting a predictor $h$ out of a collection of real-valued functions defined over $\sX \times \sY$ and a rejector $r$ from a sub-family of real-valued functions defined over $\sX$. Thus, the hypothesis sets in these two frameworks differ entirely. Consequently, a score-based loss function cannot be directly applied to the hypothesis set of the predictor-rejector formulation.

One can instead, given the hypothesis sets $\sH$ and $\sR$ for the predictor and rejector functions in the predictor-rejector formulation, define a distinct hypothesis set $\tilde \sH$ of real-valued functions defined over $\sX \times \tilde \sY$. Functions $\tilde h\in\tilde \sH$ are defined from a pair $(h, r)\in \sH \times \sR$. The score-based loss function for $\tilde h\in \tilde \sH$ then coincides with the predictor-rejector loss of $(h, r)$. However, the family $\tilde \sH$ is complex. As pointed out in Section~\ref{sec:score-example}, for instance, when $\sH$ and $\sR$ are families of linear functions, $\tilde \sH$ is not linear and is more complex.  

Moreover, there is a non-trivial coupling relating the scoring functions defined for the rejection label $(n+1)$ and other scoring functions, while such a coupling is not present in the standard score-based formulation. This makes it more difficult to minimize
the loss function $\wt \ell(\wt h, x, y)$.
Indeed, the minimization problem requires that the constraint $\wt
h(x, n + 1) = \max_{y \in \sY} \wt h(x, y) - r(x)$ be satisfied. This
constraint relates the first $n$ scoring functions $\wt h(\cdot, y)$, $y
\in \sY$, to the last scoring function $\wt h(\cdot, n + 1)$, via a
maximum operator (and the function $r$). The constraint is
non-differentiable and non-convex, which makes the minimization
problem more challenging. 

This augmented complexity and coupling fundamentally differentiate the two formulations. Our counterexample in Section~\ref{sec:score-example} underscores this intrinsic distinction between these two formulations: While the predictor-rejector formulation can easily handle certain instances, the score-based framework falls short to tackle them unless a more complex hypothesis set is adopted. This key difference between the two formulations is also the underlying reason for the historical difficulty in devising a consistent surrogate loss function for the predictor-rejector formulation within the standard multi-class setting, while the task has been comparatively more straightforward within the score-based formulation.

It is important to highlight that our novel families of predictor-rejector surrogate losses, alongside similar variants, establish the first Bayes-consistent and realizable consistent surrogate losses within the predictor-rejector formulation and they address two previously open questions in the literature \citep{NiCHS19} and \citep{pmlr-v206-mozannar23a} (see Section~\ref{sec:general}). Moreover, they outperform the state-of-the-art surrogate losses found in the score-based formulation (see Section~\ref{sec:experiments}). This underscores both the innovative nature and the significant contribution of our work.

In the following section,
we will further showcase the advantages of the predictor-rejector
formulation through empirical evidence.

\section{Experimental details}
\label{app:setup}
\paragraph{Setup.}
We adopt ResNet-$34$ \citep{he2016deep}, a residual network with $34$ convolutional layers, for SVHN and CIFAR-10,
and WRN-$28$-$10$ \citep{zagoruyko2016wide}, a residual network with $28$ convolutional layers and a widening factor of $10$, for CIFAR-100 both with ReLU activations
\citep{zagoruyko2016wide}. We train for 200 epochs using Stochastic Gradient Descent (SGD) with Nesterov momentum
\citep{nesterov1983method} following the cosine decay learning rate schedule
\citep{loshchilov2016sgdr} of an initial learning rate $0.1$.
During the training, the batch size is set to $1\mathord,024$ and the weight decay is $1\times 10^{-4}$. Except for SVHN, we adopt the standard data augmentation: a four pixel padding with $32 \times 32$ random crops and random horizontal flips.

We compare with a score-based surrogate
loss proposed in \citep{mozannar2020consistent} based on cross-entropy and a score-based surrogate loss used in \citep{caogeneralizing} based on generalized cross-entropy \citep{zhang2018generalized}. For our single-stage predictor-rejector surrogate
loss, we set $\ell$ to be the mean absolute error loss
$\ell_{\rm{mae}}$ since the constrained hinge loss imposes a
restriction incompatible with the standard use of the softmax function
with neural network hypotheses, and the $\rho$-margin loss is
non-convex. For our two-stage predictor-rejector
surrogate loss, we first use standard training with the logistic loss
to learn a predictor $h^*$, and then in the second stage, optimize the
loss function $\ell_{\Phi, h^*}$ with $\Phi(t) = \exp(-t)$ to learn a
rejector. 
We set the cost $c$ to $0.03$ for SVHN, $0.05$ for CIFAR-10 and $0.15$ for CIFAR-100. We observe that the performance remains close for other neighboring values of $c$. We highlight this particular choice of cost because a cost value that is not too far from the best-in-class zero-one classification loss encourages in practice a reasonable amount of input instances to be abstained.

\paragraph{Metrics.}  We use as evaluation metrics the average
abstention loss, $\labs$ for predictor-rejector surrogate losses and
$\labsc$ for score-based abstention surrogate losses, which share the
same semantic meaning. It's important to emphasize that the two abstention losses, $\labs$ and $\labsc$ are indeed the same metric, albeit tailored for two distinct formulations. Consequently, their average numerical values can be directly compared. Note that both $\labs$ and $\labsc$  account for the zero-one misclassification error when the sample is accepted, and the cost when the sample is rejected. The reason they are adapted to the two formulations is due to the difference in the rejection method: $r(x) \leq 0$ in the predictor-rejector formulation and $\tilde h(x) = n + 1$ in the score-based abstention formulation. It should also be noted that the abstention loss serves as a comprehensive metric that integrates the rejection ratio and zero-one misclassification error on the accepted data, thereby providing a singular, fair ground for comparison in Table~\ref{tab:comparison}. Nevertheless, we include all three metrics in Table~\ref{tab:comparison-cifar10} as a detailed comparison.

\begin{table}[t]
\caption{Abstention loss, zero-one misclassification error on the accepted data and rejection ratio of our predictor-rejector surrogate losses against baselines: 
the state-of-the-art score-based abstention surrogate losses in
\citep{mozannar2020consistent,caogeneralizing} on CIFAR-10.}
  \vskip -0.2in
    \label{tab:comparison-cifar10}
\begin{center}
\resizebox{\columnwidth}{!}{
    \begin{tabular}{@{\hspace{0pt}}llll@{\hspace{0pt}}}
    \toprule
      Method & Abstention loss & Misclassification error & Rejection ratio  \\
    \toprule
     \citep{mozannar2020consistent} &  4.48\% $\pm$ 0.10\% & 4.30\% $\pm$ 0.14\% & 25.99\% ± 0.41\%       \\
     \citep{caogeneralizing}  & 3.62\% $\pm$ 0.07\% & 3.08\% $\pm$ 0.10\%  & 28.27\% $\pm$ 0.18\%  \\
     single-stage predictor-rejector ($\ell_{\rm{mae}}$) & 3.64\% $\pm$ 0.05\% & 3.54\% $\pm$ 0.05\% &                  \textbf{17.21\% $\pm$ 0.22\%} \\
    two-stage predictor-rejector  & \textbf{3.31\% \!$\pm$ 0.02\%} &  \textbf{2.69\% $\pm$ 0.05\%}  &                 22.83\% $\pm$ 0.21\% \\
    \bottomrule
    \end{tabular}
    }
\end{center}
    \vskip -0.25in
\end{table}

\section{Useful lemmas}
\label{app:general}

We first
introduce some notation before presenting a lemma that will be used in our proofs. Recall that we denote by
$p(x, y) = \sD(Y = y \!\mid\! X = x)$ the conditional probability of
$Y=y$ given $X = x$. Thus, the generalization error for a general
abstention surrogate loss can be rewritten as $ \sE_{\sfL}(h, r) =
\mathbb{E}_{X} \bracket*{\sC_{\sfL}(h, r, x)} $, where $\sC_{\sfL}(h,
r,x)$ is the conditional risk of $\sfL$, defined by
\begin{align*}
\sC_{\sfL}(h, r, x) = \sum_{y\in \sY} p(x, y) \sfL(h, r, x, y).
\end{align*}
We denote by $\sC_{\sfL}^*(\sH, \sR, x) = \inf_{h\in \sH, r\in
  \sR}\sC_{\sfL}(h, r, x)$ the best-in-class conditional risk of $\sfL$. Then,
the minimizability gap can be rewritten as follows:
\begin{align*}
\sM_{\sfL}(\sH, \sR)
= \sE^*_{\sfL}(\sH, \sR)
- \mathbb{E}_{X} \bracket* {\sC_{\sfL}^*(\sH, \sR, x)}.
\end{align*}
We further refer to $\sC_{\sfL}(h, r, x)-\sC_{\sfL}^*(\sH, \sR, x)$ as
the calibration gap and denote it by $\Delta\sC_{\sfL,\sH, \sR}(h,
r,x)$.  We first prove a lemma on the calibration gap of the general
abstention loss. For any $x \in \sX$, we define the set of labels
generated by hypotheses in $\sH$ as $\mathsf H(x) := \curl*{\hh(x)
  \colon h \in \sH}$.  We will consider hypothesis sets $\sR$ which
are \emph{regular for abstention}.
\begin{definition}[Regularity for Abstention]
We say that a hypothesis set $\sR$ is \emph{regular for abstention} 
if for any $x\in \sX$, there exist $f, g \in \sR$ 
such that $f(x)>0$ and 
$g(x)\leq 0$.
\end{definition}
In other words, if $\sR$ is regular for abstention, then, for any
instance $x$, there is an option to accept and an option to
reject. 

\subsection{Lemma~\ref{lemma:calibration_gap_general} and proof}
The following lemma characterizes the calibration gap of the
predictor-rejector abstention.
\begin{restatable}{lemma}{ConditionalRegret}
\label{lemma:calibration_gap_general}
Assume that $\sR$ is regular for abstention. For any $x \in \sX$,
the minimal conditional $\labs$-risk and
the calibration gap for $\labs$ can be expressed as follows:
\begin{align*}
\sC^*_{\labs}(\sH,\sR,x)  & = 1 - \max\curl*{\max_{y\in
    \mathsf H(x)}p(x, y),1 - c},\\
 \Delta\sC_{\labs,\sH, \sR}(h, r, x) & =
\begin{cases}
\max\curl*{\max_{y\in \mathsf H(x)} p(x, y),1 - c} - p(x, \hh(x)) &  r(x)>0\\
\max\curl*{\max_{y\in \mathsf H(x)} p(x, y)-1+c,0} & r(x)\leq 0.
\end{cases}
\end{align*}
\end{restatable}
\begin{proof}
By the definition, the conditional $\labs$-risk can be expressed as
follows:
\begin{align}
\label{eq:cond}
\sC_{\labs}(h, r, x)
=  \sum_{y\in \sY} p(x, y) \1_{\hh(x)\neq y}\1_{r(x)> 0} + c \1_{r(x)\leq 0}
=
\begin{cases}
1-p(x, \hh(x)) & \text{if } r(x)>0\\
c & \text{if } r(x)\leq 0.
\end{cases}
\end{align}
Since $\sR$ is regular for
abstention, the minimal conditional $\labs$-risk can be expressed as
follows:
\begin{align*}
\sC^*_{\labs}(\sH,\sR,x) = 1 - \max\curl*{\max_{y\in
    \mathsf H(x)}p(x, y),1 - c},
\end{align*}
which proves the first part of the lemma. By the definition of the calibration gap, we have
\begin{align*}
\Delta\sC_{\labs,\sH,\sR}(h, r, x)
& = \sC_{\labs}(h, r, x)- \sC^*_{\labs}(\sH,\sR,x)\\
& =
\begin{cases}
\max\curl*{\max_{y\in \mathsf H(x)} p(x, y),1 - c} - p(x, \hh(x)) & \text{if } r(x)>0\\
\max\curl*{\max_{y\in \mathsf H(x)} p(x, y)-1+c,0} & \text{if } r(x)\leq 0,
\end{cases}
\end{align*}
which completes the proof.
\end{proof}

\subsection{Lemma~\ref{lemma:aux} and proof}
The following lemma would be useful in the proofs for two-stage surrogate losses.
\begin{lemma}
\label{lemma:aux}
Assume that the following $\sR$-consistency bound holds for all $r \in \sR$ and any distribution,
\begin{equation*}
\sE_{\ell_{0-1}}(r) - \sE^*_{\ell_{0-1}}(\sR) + \sM_{\ell_{0-1}}(\sR) \leq \Gamma\paren*{\sE_{\Phi}(r) - \sE^*_{\Phi}(\sR) + \sM_{\Phi}(\sR)}.
\end{equation*}
Then, for any $p_1, p_2\in [0,1]$ such that $p_1 + p_2 =1$ and $x \in \sX$, we have
\begin{align*}
& p_1 1_{r(x) > 0} + p_2 1_{r(x) \leq 0} - \inf_{r \in \sR} \paren*{p_1 1_{r(x) > 0} + p_2 1_{r(x) \leq 0}}\\
&\quad \leq \Gamma \paren*{p_1 \Phi(-r(x)) + p_2 \Phi(r(x)) - \inf_{r \in \sR} \paren*{p_1 \Phi(-r(x)) + p_2 \Phi(r(x))}}
\end{align*}
\end{lemma}
\begin{proof}
For any $x \in \sX$, consider a distribution $\delta_{x}$ that concentrates on that point. Let $p_1 = \mathbb{P}(y = \minus 1 \mid x)$ and $p_2 = \mathbb{P}(y = \plus 1 \mid x)$. Then, by definition, $\sE_{\ell_{0-1}}(r) - \sE^*_{\ell_{0-1}}(\sR) + \sM_{\ell_{0-1}}(\sR)$ can be expressed as 
\begin{equation*}
\sE_{\ell_{0-1}}(r) - \sE^*_{\ell_{0-1}}(\sR) + \sM_{\ell_{0-1}}(\sR) = p_1 1_{r(x) > 0} + p_2 1_{r(x) \leq 0} - \inf_{r \in \sR} \paren*{p_1 1_{r(x) > 0} + 1_{r(x) \leq 0}}.
\end{equation*}
Similarly, $\sE_{\Phi}(r) - \sE^*_{\Phi}(\sR) + \sM_{\Phi}(\sR)$ can be expressed as
\begin{equation*}
\sE_{\Phi}(r) - \sE^*_{\Phi}(\sR) + \sM_{\Phi}(\sR) = p_1 \Phi(-r(x)) + p_2 \Phi(r(x)) - \inf_{r \in \sR} \paren*{p_1 \Phi(-r(x)) + p_2 \Phi(r(x))}.
\end{equation*}
Since the $\sR$-consistency bound holds by the assumption, we complete the proof.
\end{proof}

\section{Proofs of main theorems}

\subsection{Proof of negative result for single-stage surrogates (Theorem~\ref{Thm:negative-bound})}
\label{app:general-negativ}
\NegativeBound*
\begin{proof}
We prove by contradiction. Assume that the bound holds with some non-decreasing function $\Gamma$ with
$\lim_{t\to 0^{+}}\Gamma(t) = 0$, then, for all
$h\in \sH$, $r\in \sR$, and any distribution, $\sE_{\sfL}(h,
r)-\sE_{\sfL}^*(\sH, \sR) +\sM_{\sfL}(\sH, \sR) \to 0 \implies
\sE_{\labs}(h, r) - \sE_{\labs}^*(\sH, \sR) + \sM_{\labs}(\sH, \sR)
\to 0$. This further implies that for any $x\in \sX$, the minimizer
$h^*$ and $r^*$ of $\sC_{\sfL}(h, r, x)$ within $\sH$ and $\sR$ also
achieves the minimum of $\sC_{\labs}(h, r, x)$ within $\sH$ and
$\sR$. When $\sH$ is symmetric and complete, we have
$\mathsf{H}(x)=\sY$. Since $\sR$ is complete, $\sR$ is regular for
abstention. By Lemma~\ref{lemma:calibration_gap_general}, $h^*$ and
$r^*$ need to satisfy the following conditions:
\begin{equation}
\label{eq:star-general-cond}
p(x, \hh^*(x))
= \max_{y\in \sY} p(x, y),\quad \sign(r^*(x))
= \sign\paren*{\max_{y\in \sY}p(x, y) - (1 - c)}.
\end{equation}
Next, we will show that \eqref{eq:star-general-cond} contradicts the
assumption that there exists $x\in \sX$ such that \[\inf_{h \in \sH}
\E_y\bracket*{\ell(h,X, y) \mid X =x}\neq \frac{\beta\Psi\paren*{1 -
    \max_{y\in \sY}p(x, y)}}{\alpha}.\]  By definition, the conditional
$\sfL$-risk can be expressed as follows:
\begin{equation*}
\sC_{\sfL}(h, r, x)  =  \exp(\alpha r(x))\mathbb{E}_y
  \bracket*{\ell(h,X, y) \mid X = x} + \Psi(c)\exp(-\beta r(x) ).
\end{equation*}
Then, for any fixed $r\in \sR$, $\alpha>0$ and $\beta>0$, we have
\begin{equation*}
  \inf_{h\in \sH}\sC_{\sfL}(h, r, x)
  = \exp(\alpha r(x))\inf_{h \in \sH} \mathbb{E}_y
\bracket*{\ell(h,X, y) \mid X = x}
+ \Psi(c)\exp(-\beta r(x)):=\sF(r(x)).
\end{equation*}
By taking the derivative, we obtain
\begin{equation*}
\sF'(r(x)) = \alpha \exp(\alpha r(x))\inf_{h \in \sH} \mathbb{E}_y
  \bracket*{\ell(h,X, y) \mid X = x}  - \beta \Psi(c) \exp(-\beta r(x) ).
\end{equation*}
Since $\sF(r(x))$ is convex with respect to $r(x)$, we know that
$\sF'(r(x))$ is non-decreasing with respect to $r(x)$. The equation \eqref{eq:star-general-cond} implies that
 $\sF(r(x))$ is attained at $r^*(x)$ such that
$\sign(r^*(x)) = \sign\paren*{\max_{y\in \sY}p(x, y) - (1 - c)}$. Thus, we have
\begin{align*}
& \max_{y\in \sY}p(x, y)\geq (1 - c) \implies r^*(x)\geq 0 \implies \sF'(0)\leq \sF'(r^*(x)) =0\\
& \max_{y\in \sY}p(x, y)< (1 - c) \implies r^*(x)< 0 \implies \sF'(0)\geq \sF'(r^*(x))=0.
\end{align*}
This implies that
\begin{align*}
\alpha\inf_{h \in \sH} \mathbb{E}_y
  \bracket*{\ell(h,X, y) \mid X = x}  - \beta\Psi(c)\leq 0 \text{ whenever } \max_{y\in \sY}p(x, y)\geq (1 - c)\\
\alpha\inf_{h \in \sH} \mathbb{E}_y
  \bracket*{\ell(h,X, y) \mid X = x}  - \beta\Psi(c) \geq 0 \text{ whenever } \max_{y\in \sY}p(x, y)\leq (1 - c),
\end{align*} 
which leads to
\begin{align}
\label{eq:general-con}
\alpha\inf_{h \in \sH} \mathbb{E}_y
  \bracket*{\ell(h,X, y) \mid X = x}  - \beta\Psi(c) = 0 \text{ whenever } \max_{y\in \sY}p(x, y) = (1 - c).
\end{align}
It is clear that \eqref{eq:general-con} contradicts the assumption that $\exists \,x\in \sX$ such that $\inf_{h \in \sH} \E_y\bracket*{\ell(h,X, y) \mid X =x}\neq \frac{\beta\Psi\paren*{1 - \max_{y\in \sY}p(x, y)}}{\alpha}$. 
\end{proof}

\subsection{Proof of \texorpdfstring{$(\sH, \sR)$}{HR}-consistency
  bounds for single-stage surrogates
  (Theorem~\ref{Thm:spcific-loss-bound})}
\label{app:general-positive-single-stage}

\SpecificLossBound*
\begin{proof}
When $\sH$ is symmetric and complete, $\mathsf H(x)=\sY$. Since $\sR$
is complete, $\sR$ is regular for abstention. By
Lemma~\ref{lemma:calibration_gap_general},
\begin{equation}
\label{eq:calibration_gap_general}
\begin{aligned}
\sC^*_{\labs}(\sH,\sR,x)  & = 1 - \max\curl*{\max_{y\in
    \sY}p(x, y),1 - c},\\
 \Delta\sC_{\labs,\sH, \sR}(h, r, x) & =
\begin{cases}
\max\curl*{\max_{y\in \sY} p(x, y),1 - c} - p(x, \hh(x)) &  r(x)>0\\
\max\curl*{\max_{y\in \sY} p(x, y)-1+c,0} & r(x)\leq 0.
\end{cases}
\end{aligned}
\end{equation}
Then, the idea of proof for each loss $\ell$ is similar, which
consists of the analysis in four cases depending on the sign of
$\max_{y\in \sY}p(x,y)-(1-c)$ and the sign of $r(x)$, as shown below.

\textbf{Mean absolute error loss: $\ell= \ell_{\rm{mae}}$.}
\ignore{By Lemma~\ref{lemma:calibration_gap_general}, $\sC_{\labs}(h, r, x) =
\begin{cases}
1-p(x, \hh(x)) & r(x)\leq 0\\
c & r(x)> 0
\end{cases}
$ and $\sC^*_{\labs}(\sH, \sR, x) =1 - \max\curl*{\max_{y\in
    \sY}p(x, y),1 - c}$.} When $\ell= \ell_{\rm{mae}}$ with $\Psi(t)=t$, let $s_h(x, y) = \frac{e^{h(x, y)}}{\sum_{y'\in
    \sY}e^{h(x, y')}}$, by
the assumption that $\sH$ is symmetric and complete and $\sR$ is complete, we obtain $\inf_{h \in \sH} \E_y\bracket*{\ell_{\rm{mae}}(h,X, y) \mid X =x} =1 - \max_{y\in \sY}p(x, y)$ and
\begin{align*}
\sC_{\sfL}(h, r, x) &= \sum_{y\in \sY}p(x, y)(1-s_h(x, y))e^{\alpha r(x)} + c e^{-\alpha r(x)}\\
\sC^*_{\sfL}(\sH, \sR, x) & =2\sqrt{c\paren*{1 - \max_{y\in \sY}p(x, y)}}.
\end{align*}
Note that for any $h\in \sH$,
\begin{align*}
& \sum_{y\in \sY}p(x, y)(1-s_h(x, y)) - \paren*{1 - \max_{y\in \sY}p(x, y)}\\
& = \max_{y\in \sY}p(x, y) - \sum_{y\in \sY}p(x, y)s_h(x, y)\\
& \geq \max_{y\in \sY}p(x, y) - p(x, \hh(x)s_h(x, \hh(x))- \max_{y\in \sY}p(x, y)\paren*{1-s_h(x, \hh(x))}
\tag{$p(x, y)\leq \max_{y\in \sY}p(x, y)$, $\forall y\neq \hh(x)$}\\
& = s_h(x, \hh(x))\paren*{\max_{y\in \sY}p(x, y)-p(x, \hh(x)}\\
& \geq \frac{1}{n}\paren*{\max_{y\in \sY}p(x, y)-p(x, \hh(x)}
\tag{$s_h(x, \hh(x)) = \max_{y\in \sY}s_h(x, y)\geq \frac{1}{n}$}.
\end{align*}
We will then analyze the following four cases.
     \paragraph{(i)} $\max_{y\in \sY}p(x, y) > (1 - c)$ and $r(x)>0$. In this case, by \eqref{eq:calibration_gap_general},
      we have $\sC^*_{\labs}(\sH, \sR, x) =1 - \max_{y\in
        \sY}p(x, y)$ and $\Delta\sC_{\labs,\sH,
        \sR}(h, r, x) = \max_{y\in \sY}p(x, y)-p\paren*{x, \hh(x)}$. For
      the surrogate loss, we have
    \begin{align*}
    & \Delta\sC_{\sfL,\sH, \sR}(h, r, x)\\
    & = \sum_{y\in \sY}p(x, y)(1-s_h(x, y))e^{\alpha r(x)} + c e^{-\alpha r(x)} - 2\sqrt{c\paren*{1 - \max_{y\in \sY}p(x, y)}}\\
    & \geq \sum_{y\in \sY}p(x, y)(1-s_h(x, y))e^{\alpha r(x)} + c e^{-\alpha r(x)} - \paren*{1 - \max_{y\in \sY}p(x, y)}e^{\alpha r(x)} - c e^{-\alpha r(x)}
    \tag{AM–GM inequality}\\
    & \geq \sum_{y\in \sY}p(x, y)(1-s_h(x, y)) - \paren*{1 - \max_{y\in \sY}p(x, y)}
    \tag{$r(x)>0$}\\
    & \geq \frac{1}{n}\paren*{\max_{y\in \sY}p(x, y)-p(x, \hh(x)}\\
    \tag{$\sum_{y\in \sY}p(x, y)(1-s_h(x, y)) - \paren*{1 - \max_{y\in \sY}p(x, y)}\geq \frac{1}{n}\paren*{\max_{y\in \sY}p(x, y)-p(x, \hh(x)}$}\\
    & = \frac{1}{n}\Delta\sC_{\labs,\sH, \sR}(h, r, x).
    \end{align*}
     Therefore, 
    \begin{align*}
\sE_{\labs}(h, r) - \sE_{\labs}^*(\sH, \sR) + \sM_{\labs}(\sH, \sR) & = \mathbb{E}_{X}\bracket*{\Delta\sC_{\labs,\sH, \sR}(h, r, x)}\\
& \leq \mathbb{E}_{X}\bracket*{\Gamma_1\paren*{\Delta\sC_{\sfL,\sH, \sR}(h, r, x)}}\\
& \leq \Gamma_1\paren*{\mathbb{E}_{X}\bracket*{\Delta\sC_{\sfL,\sH, \sR}(h, r, x)}}
\tag{$\Gamma_1$ is concave}\\
& = \Gamma_1\paren*{\sE_{\sfL}(h, r)-\sE_{\sfL}^*(\sH, \sR) +\sM_{\sfL}(\sH, \sR)}
\end{align*}
where $\Gamma_1(t)= n\,t$.
  \paragraph{(ii)} $\max_{y\in \sY}p(x, y) \leq (1 - c)$ and $r(x)>0$. In this
      case, by \eqref{eq:calibration_gap_general}, we have $\sC^*_{\labs}(\sH, \sR, x) =c$ and
      $\Delta\sC_{\labs,\sH, \sR}(h, r, x) =1-c-p\paren*{x,
        \hh(x)}$. For the surrogate loss, we have
    \begin{align*}
    & \Delta\sC_{\sfL,\sH, \sR}(h, r, x) \\
    & = \sum_{y\in \sY}p(x, y)(1-s_h(x, y))e^{\alpha r(x)} + c e^{-\alpha r(x)} - 2\sqrt{c\paren*{1 - \max_{y\in \sY}p(x, y)}}\\
    & \geq \sum_{y\in \sY}p(x, y)(1-s_h(x, y))e^{\alpha r(x)} + c e^{-\alpha r(x)} - 2\sqrt{c\paren*{\sum_{y\in \sY}p(x, y)(1-s_h(x, y))}}
   \tag{$\sum_{y\in \sY}p(x, y)(1-s_h(x, y))\geq 1 - \max_{y\in \sY}p(x, y)$}\\
    & \geq \sum_{y\in \sY}p(x, y)(1-s_h(x, y)) + c - 2\sqrt{c\paren*{\sum_{y\in \sY}p(x, y)(1-s_h(x, y))}}
    \tag{increasing for $r(x)\geq 0$}\\
    & = \paren*{\sqrt{\sum_{y\in \sY}p(x, y)(1-s_h(x, y))}-\sqrt{c}}^2\\
    & = \paren*{\frac{\sum_{y\in \sY}p(x, y)(1-s_h(x, y))-c}{\sqrt{\sum_{y\in \sY}p(x, y)(1-s_h(x, y))}+\sqrt{c}}}^2\\
    & \geq  \paren*{\frac{\sum_{y\in \sY}p(x, y)(1-s_h(x, y))-\paren*{1 - \max_{y\in \sY}p(x, y)}+\paren*{1 - \max_{y\in \sY}p(x, y)-c}}{2}}^2
    \tag{$\sqrt{\sum_{y\in \sY}p(x, y)(1-s_h(x, y))}+\sqrt{c}\leq 2$}\\
    & \geq  \paren*{\frac{\frac{1}{n}\paren*{\max_{y\in \sY}p(x, y)-p(x, \hh(x)}+\frac{1}{n}\paren*{1 - \max_{y\in \sY}p(x, y)-c}}{2}}^2
    \tag{$\sum_{y\in \sY}p(x, y)(1-s_h(x, y)) - \paren*{1 - \max_{y\in \sY}p(x, y)}\geq \frac{1}{n}\paren*{\max_{y\in \sY}p(x, y)-p(x, \hh(x)}$ and $1 - \max_{y\in \sY}p(x, y)-c\geq 0$}\\
    & = \frac{1}{4n^2}\paren*{1-c-p\paren*{x,\hh(x)}}^2\\
    & = \frac{\Delta\sC_{\labs,\sH, \sR}(h, r, x)^2}{4n^2}
    \end{align*}
    Therefore, 
    \begin{align*}
\sE_{\labs}(h, r) - \sE_{\labs}^*(\sH, \sR) + \sM_{\labs}(\sH, \sR) & = \mathbb{E}_{X}\bracket*{\Delta\sC_{\labs,\sH, \sR}(h, r, x)}\\
& \leq \mathbb{E}_{X}\bracket*{\Gamma_2\paren*{\Delta\sC_{\sfL,\sH, \sR}(h, r, x)}}\\
& \leq \Gamma_2\paren*{\mathbb{E}_{X}\bracket*{\Delta\sC_{\sfL,\sH, \sR}(h, r, x)}}
\tag{$\Gamma_2$ is concave}\\
& = \Gamma_2\paren*{\sE_{\sfL}(h, r)-\sE_{\sfL}^*(\sH, \sR) +\sM_{\sfL}(\sH, \sR)}
\end{align*}
where $\Gamma_2(t)=2n\sqrt{t}$.
  \paragraph{(iii)} $\max_{y\in \sY}p(x, y) \leq  (1 - c)$ and $r(x)\leq 0$. In this case, by \eqref{eq:calibration_gap_general}, we have $\sC^*_{\labs}(\sH, \sR, x) =c$ and $\Delta\sC_{\labs,\sH, \sR}(h, r, x)=0$, which implies that $\sE_{\labs}(h, r) - \sE_{\labs}^*(\sH, \sR) + \sM_{\labs}(\sH, \sR) = \mathbb{E}_{X}\bracket*{\Delta\sC_{\labs,\sH, \sR}(h, r, x)}=0\leq \Gamma\paren*{\sE_{\sfL}(h, r)-\sE_{\sfL}^*(\sH, \sR) +\sM_{\sfL}(\sH, \sR)}$ for any $\Gamma\geq 0$.
  \paragraph{(iv)} $\max_{y\in \sY}p(x, y) > (1 - c)$ and $r(x)\leq 0$. In this
      case, by \eqref{eq:calibration_gap_general}, we have $\sC^*_{\labs}(\sH, \sR, x) =1 - \max_{y\in
        \sY}p(x, y)$ and $\Delta\sC_{\labs,\sH,
        \sR}(h, r, x) =\max_{y\in \sY} p(x, y)-1+c$. For the
      surrogate loss, we have
    \begin{align*}
    &\Delta\sC_{\sfL,\sH, \sR}(h, r, x)\\
    & = \sum_{y\in \sY}p(x, y)(1-s_h(x, y))e^{\alpha r(x)} + c e^{-\alpha r(x)} - 2\sqrt{c\paren*{1 - \max_{y\in \sY}p(x, y)}}\\
    & \geq \paren*{1 - \max_{y\in \sY}p(x, y)}e^{\alpha r(x)} + c e^{-\alpha r(x)} - 2\sqrt{c\paren*{1 - \max_{y\in \sY}p(x, y)}}
    \tag{$\sum_{y\in \sY}p(x, y)(1-s_h(x, y))\geq 1 - \max_{y\in \sY}p(x, y)$}\\
    & \geq 1 - \max_{y\in \sY}p(x, y)+c - 2\sqrt{c\paren*{1 - \max_{y\in \sY}p(x, y)}}
    \tag{decreasing for $r(x)\leq 0$}\\
    & = \paren*{\sqrt{1 - \max_{y\in \sY}p(x, y)}-\sqrt{c}}^2\\
    & = \paren*{\frac{1 - \max_{y\in \sY}p(x, y)-c}{\sqrt{1 - \max_{y\in \sY}p(x, y)}+\sqrt{c}}}^2\\
    & \geq \paren*{\frac{\max_{y\in \sY} p(x, y)-1+c}{2}}^2
    \tag{$\sqrt{1 - \max_{y\in \sY}p(x, y)}+\sqrt{c}\leq 2$}\\
    & = \frac{\Delta\sC_{\labs,\sH, \sR}(h, r, x)^2}{4}.
    \end{align*}
     Therefore, 
    \begin{align*}
      \sE_{\labs}(h, r) - \sE_{\labs}^*(\sH, \sR) + \sM_{\labs}(\sH, \sR)
      & = \mathbb{E}_{X}\bracket*{\Delta\sC_{\labs,\sH, \sR}(h, r, x)}\\
      & \leq \mathbb{E}_{X}\bracket*{\Gamma_3\paren*{\Delta\sC_{\sfL,\sH, \sR}(h, r, x)}}\\
      & \leq \Gamma_3\paren*{\mathbb{E}_{X}\bracket*{\Delta\sC_{\sfL,\sH, \sR}(h, r, x)}}
      \tag{$\Gamma_3$ is concave}\\
& = \Gamma_3\paren*{\sE_{\sfL}(h, r)-\sE_{\sfL}^*(\sH, \sR) +\sM_{\sfL}(\sH, \sR)}
\end{align*}
where $\Gamma_3(t)=2\sqrt{t}$.

Overall, we obtain
\begin{align*}
\sE_{\labs}(h, r) - \sE_{\labs}^*(\sH, \sR) + \sM_{\labs}(\sH, \sR) \leq \Gamma\paren*{\sE_{\sfL}(h, r)-\sE_{\sfL}^*(\sH, \sR) +\sM_{\sfL}(\sH, \sR)}
\end{align*}
where $\Gamma(t)=\max\curl*{\Gamma_1(t),\Gamma_2(t),\Gamma_3(t)}= \max\curl*{2n\sqrt{t},n\,t}$, which completes the proof.

\textbf{$\rho$-Margin loss: $\ell= \ell_{\rho}$.}
When $\ell= \ell_{\rho}$ with $\Psi(t)=t$, by
the assumption that $\sH$ is symmetric and complete and $\sR$ is complete, we obtain\ignore{$\sC_{\ell_{\rho}}(h, x) = \sum_{y\in \sY}p(x, y)\min\curl*{\max\curl*{0,1 - \frac{\rho_h(x,y)}{\rho}},1}
=1- \min\curl*{1,\frac{\rho_h(x, \hh(x))}{\rho}}\,p(x,\hh(x))$}
$
\inf_{h \in \sH} \E_y\bracket*{\ell_{\rho}(h,X, y) \mid X =x}  =1 - \max_{y\in \sY}p(x, y)$ and
\begin{align*}
\sC_{\sfL}(h, r, x) &= \sum_{y\in \sY}p(x, y)\min\curl*{\max\curl*{0,1 - \frac{\rho_h(x,y)}{\rho}},1}e^{\alpha r(x)} + c e^{-\alpha r(x)}\\
&=\paren*{1- \min\curl*{1,\frac{\rho_h(x, \hh(x))}{\rho}}\,p(x,\hh(x))}e^{\alpha r(x)} + c e^{-\alpha r(x)}\\
\sC^*_{\sfL}(\sH, \sR, x) & =2\sqrt{c\paren*{1 - \max_{y\in \sY}p(x, y)}}.
\end{align*}
where $\rho_h(x,y) = h(x, y) - \max_{y' \neq y} h(x, y')$ is the margin. Note that for any $h\in \sH$,
\begin{align*}
& 1- \min\curl*{1,\frac{\rho_h(x, \hh(x))}{\rho}}\,p(x,\hh(x)) - \paren*{1 - \max_{y\in \sY}p(x, y)}\\
& = \max_{y\in \sY}p(x, y) - \min\curl*{1,\frac{\rho_h(x, \hh(x))}{\rho}}\,p(x,\hh(x))\\
& \geq \max_{y\in \sY}p(x, y)-p(x, \hh(x)
\tag{$\min\curl*{1,\frac{\rho_h(x, \hh(x))}{\rho}}\leq 1$}.
\end{align*}
We will then analyze the following four cases.
\paragraph{(i)} $\max_{y\in \sY}p(x, y) \leq (1 - c)$ and $r(x)>0$. In this
      case, by \eqref{eq:calibration_gap_general}, we have $\sC^*_{\labs}(\sH, \sR, x) =c$ and
      $\Delta\sC_{\labs,\sH, \sR}(h, r, x) =1-c-p\paren*{x,
        \hh(x)}$. For the surrogate loss, we have
    \begin{align*}
    &\Delta\sC_{\sfL,\sH, \sR}(h, r, x)\\ 
    & = \paren*{1- \min\curl*{1,\frac{\rho_h(x, \hh(x))}{\rho}}\,p(x,\hh(x))}e^{\alpha r(x)} + c e^{-\alpha r(x)} - 2\sqrt{c\paren*{1 - \max_{y\in \sY}p(x, y)}}\\
    & \geq \paren*{1- \min\curl*{1,\frac{\rho_h(x, \hh(x))}{\rho}}\,p(x,\hh(x))}e^{\alpha r(x)} + c e^{-\alpha r(x)} - 2\sqrt{c\paren*{1- \min\curl*{1,\frac{\rho_h(x, \hh(x))}{\rho}}\,p(x,\hh(x))}}
   \tag{$1- \min\curl*{1,\frac{\rho_h(x, \hh(x))}{\rho}}\,p(x,\hh(x))\geq 1 - \max_{y\in \sY}p(x, y)$}\\
    & \geq 1- \min\curl*{1,\frac{\rho_h(x, \hh(x))}{\rho}}\,p(x,\hh(x)) + c -  2\sqrt{c\paren*{1- \min\curl*{1,\frac{\rho_h(x, \hh(x))}{\rho}}\,p(x,\hh(x))}}
    \tag{increasing for $r(x)\geq 0$}\\
    & = \paren*{\sqrt{1- \min\curl*{1,\frac{\rho_h(x, \hh(x))}{\rho}}\,p(x,\hh(x))}-\sqrt{c}}^2\\
    & = \paren*{\frac{1- \min\curl*{1,\frac{\rho_h(x, \hh(x))}{\rho}}\,p(x,\hh(x))-c}{\sqrt{1- \min\curl*{1,\frac{\rho_h(x, \hh(x))}{\rho}}\,p(x,\hh(x))}+\sqrt{c}}}^2\\
    & \geq  \paren*{\frac{\1- \min\curl*{1,\frac{\rho_h(x, \hh(x))}{\rho}}\,p(x,\hh(x))-\paren*{1 - \max_{y\in \sY}p(x, y)}+\paren*{1 - \max_{y\in \sY}p(x, y)-c}}{2}}^2
    \tag{$\sqrt{1- \min\curl*{1,\frac{\rho_h(x, \hh(x))}{\rho}}\,p(x,\hh(x))}+\sqrt{c}\leq 2$}\\
    & \geq  \paren*{\frac{\max_{y\in \sY}p(x, y)-p(x, \hh(x)+\paren*{1 - \max_{y\in \sY}p(x, y)-c}}{2}}^2
    \tag{$1- \min\curl*{1,\frac{\rho_h(x, \hh(x))}{\rho}}\,p(x,\hh(x)) - \paren*{1 - \max_{y\in \sY}p(x, y)}\geq \max_{y\in \sY}p(x, y)-p(x, \hh(x)$}\\
    & = \frac{1}{4}\paren*{1-c-p\paren*{x,\hh(x)}}^2\\
    & = \frac{\Delta\sC_{\labs,\sH, \sR}(h, r, x)^2}{4}
    \end{align*}
    Therefore, 
    \begin{align*}
\sE_{\labs}(h, r) - \sE_{\labs}^*(\sH, \sR) + \sM_{\labs}(\sH, \sR) & = \mathbb{E}_{X}\bracket*{\Delta\sC_{\labs,\sH, \sR}(h, r, x)}\\
& \leq \mathbb{E}_{X}\bracket*{\Gamma_2\paren*{\Delta\sC_{\sfL,\sH, \sR}(h, r, x)}}\\
& \leq \Gamma_2\paren*{\mathbb{E}_{X}\bracket*{\Delta\sC_{\sfL,\sH, \sR}(h, r, x)}}
\tag{$\Gamma_2$ is concave}\\
& = \Gamma_2\paren*{\sE_{\sfL}(h, r)-\sE_{\sfL}^*(\sH, \sR) +\sM_{\sfL}(\sH, \sR)}
\end{align*}
where $\Gamma_2(t)=2\sqrt{t}$.

\paragraph{(ii)} $\max_{y\in \sY}p(x, y) > (1 - c)$ and $r(x)>0$. In this case, by \eqref{eq:calibration_gap_general},
      we have $\sC^*_{\labs}(\sH, \sR, x) =1 - \max_{y\in
        \sY}p(x, y)$ and $\Delta\sC_{\labs,\sH,
        \sR}(h, r, x) = \max_{y\in \sY}p(x, y)-p\paren*{x, \hh(x)}$. For
      the surrogate loss, we have
    \begin{align*}
    &\Delta\sC_{\sfL,\sH, \sR}(h, r, x)\\ 
    & = \paren*{1- \min\curl*{1,\frac{\rho_h(x, \hh(x))}{\rho}}\,p(x,\hh(x))}e^{\alpha r(x)} + c e^{-\alpha r(x)} - 2\sqrt{c\paren*{1 - \max_{y\in \sY}p(x, y)}}\\
    & \geq \paren*{1- \min\curl*{1,\frac{\rho_h(x, \hh(x))}{\rho}}\,p(x,\hh(x))}e^{\alpha r(x)} + c e^{-\alpha r(x)} - \paren*{1 - \max_{y\in \sY}p(x, y)}e^{\alpha r(x)} - c e^{-\alpha r(x)}
    \tag{AM–GM inequality}\\
    & \geq 1- \min\curl*{1,\frac{\rho_h(x, \hh(x))}{\rho}}\,p(x,\hh(x)) - \paren*{1 - \max_{y\in \sY}p(x, y)}
    \tag{$r(x)>0$}\\
    & \geq \max_{y\in \sY}p(x, y)-p(x, \hh(x)\\
    \tag{$1- \min\curl*{1,\frac{\rho_h(x, \hh(x))}{\rho}}\,p(x,\hh(x))- \paren*{1 - \max_{y\in \sY}p(x, y)}\geq \max_{y\in \sY}p(x, y)-p(x, \hh(x)$}\\
    & = \Delta\sC_{\labs,\sH, \sR}(h, r, x).
    \end{align*}
     Therefore, $\sE_{\labs}(h, r) - \sE_{\labs}^*(\sH, \sR) + \sM_{\labs}(\sH, \sR)  = \mathbb{E}_{X}\bracket*{\Delta\sC_{\labs,\sH, \sR}(h, r, x)}
\leq \mathbb{E}_{X}\bracket*{\Gamma_1\paren*{\Delta\sC_{\sfL,\sH, \sR}(h, r, x)}}
 \leq \Gamma_1\paren*{\mathbb{E}_{X}\bracket*{\Delta\sC_{\sfL,\sH, \sR}(h, r, x)}}
= \Gamma_1\paren[big]{\sE_{\sfL}(h, r)-\sE_{\sfL}^*(\sH, \sR) +\sM_{\sfL}(\sH, \sR)}$, where $\Gamma_1(t)= t$ is concave.
\ignore{\begin{align*}
\sE_{\labs}(h, r) - \sE_{\labs}^*(\sH, \sR) + \sM_{\labs}(\sH, \sR) & = \mathbb{E}_{X}\bracket*{\Delta\sC_{\labs,\sH, \sR}(h, r, x)}\\
& \leq \mathbb{E}_{X}\bracket*{\Gamma_1\paren*{\Delta\sC_{\sfL,\sH, \sR}(h, r, x)}}\\
& \leq \Gamma_1\paren*{\mathbb{E}_{X}\bracket*{\Delta\sC_{\sfL,\sH, \sR}(h, r, x)}}
\tag{$\Gamma_1$ is concave}\\
& = \Gamma_1\paren*{\sE_{\sfL}(h, r)-\sE_{\sfL}^*(\sH, \sR) +\sM_{\sfL}(\sH, \sR)}
\end{align*}}

\paragraph{(iii)} $\max_{y\in \sY}p(x, y) \leq  (1 - c)$ and $r(x)\leq 0$. In this case, by \eqref{eq:calibration_gap_general}, we have $\sC^*_{\labs}(\sH, \sR, x) =c$ and $\Delta\sC_{\labs,\sH, \sR}(h, r, x)=0$, which implies that $\sE_{\labs}(h, r) - \sE_{\labs}^*(\sH, \sR) + \sM_{\labs}(\sH, \sR) = \mathbb{E}_{X}\bracket*{\Delta\sC_{\labs,\sH, \sR}(h, r, x)}=0\leq \Gamma\paren*{\sE_{\sfL}(h, r)-\sE_{\sfL}^*(\sH, \sR) +\sM_{\sfL}(\sH, \sR)}$ for any $\Gamma\geq 0$.
\paragraph{(iv)} $\max_{y\in \sY}p(x, y) > (1 - c)$ and $r(x)\leq 0$. In this
      case, by \eqref{eq:calibration_gap_general}, we have $\sC^*_{\labs}(\sH, \sR, x) =1 - \max_{y\in
        \sY}p(x, y)$ and $\Delta\sC_{\labs,\sH,
        \sR}(h, r, x) =\max_{y\in \sY} p(x, y)-1+c$. For the
      surrogate loss, we have
    \begin{align*}
    &\Delta\sC_{\sfL,\sH, \sR}(h, r, x)\\ 
    & = \paren*{1- \min\curl*{1,\frac{\rho_h(x, \hh(x))}{\rho}}\,p(x,\hh(x))}e^{\alpha r(x)} + c e^{-\alpha r(x)} - 2\sqrt{c\paren*{1 - \max_{y\in \sY}p(x, y)}}\\
    & \geq \paren*{1 - \max_{y\in \sY}p(x, y)}e^{\alpha r(x)} + c e^{-\alpha r(x)} - 2\sqrt{c\paren*{1 - \max_{y\in \sY}p(x, y)}}
    \tag{$1- \min\curl*{1,\frac{\rho_h(x, \hh(x))}{\rho}}\,p(x,\hh(x))\geq 1 - \max_{y\in \sY}p(x, y)$}\\
    & \geq 1 - \max_{y\in \sY}p(x, y)+c - 2\sqrt{c\paren*{1 - \max_{y\in \sY}p(x, y)}}
    \tag{decreasing for $r(x)\leq 0$}\\
    & = \paren*{\sqrt{1 - \max_{y\in \sY}p(x, y)}-\sqrt{c}}^2\\
    & = \paren*{\frac{1 - \max_{y\in \sY}p(x, y)-c}{\sqrt{1 - \max_{y\in \sY}p(x, y)}+\sqrt{c}}}^2\\
    & \geq \paren*{\frac{\max_{y\in \sY} p(x, y)-1+c}{2}}^2
    \tag{$\sqrt{1 - \max_{y\in \sY}p(x, y)}+\sqrt{c}\leq 2$}\\
    & = \frac{\Delta\sC_{\labs,\sH, \sR}(h, r, x)^2}{4}.
    \end{align*}
     Therefore, 
    \begin{align*}
      \sE_{\labs}(h, r) - \sE_{\labs}^*(\sH, \sR) + \sM_{\labs}(\sH, \sR)
      & = \mathbb{E}_{X}\bracket*{\Delta\sC_{\labs,\sH, \sR}(h, r, x)}\\
      & \leq \mathbb{E}_{X}\bracket*{\Gamma_3\paren*{\Delta\sC_{\sfL,\sH, \sR}(h, r, x)}}\\
      & \leq \Gamma_3\paren*{\mathbb{E}_{X}\bracket*{\Delta\sC_{\sfL,\sH, \sR}(h, r, x)}}
      \tag{$\Gamma_3$ is concave}\\
& = \Gamma_3\paren*{\sE_{\sfL}(h, r)-\sE_{\sfL}^*(\sH, \sR) +\sM_{\sfL}(\sH, \sR)}
\end{align*}
where $\Gamma_3(t)=2\sqrt{t}$.

Overall, we obtain
\begin{align*}
\sE_{\labs}(h, r) - \sE_{\labs}^*(\sH, \sR) + \sM_{\labs}(\sH, \sR) \leq \Gamma\paren*{\sE_{\sfL}(h, r)-\sE_{\sfL}^*(\sH, \sR) +\sM_{\sfL}(\sH, \sR)}
\end{align*}
where $\Gamma(t)=\max\curl*{\Gamma_1(t),\Gamma_2(t),\Gamma_3(t)}= \max\curl*{2\sqrt{t},t}$, which completes the proof.

\textbf{Constrained $\rho$-hinge loss: $\ell=\ell_{\rho-\mathrm{hinge}}$.}
When $\ell= \ell_{\rm{\rho-\mathrm{hinge}}}$ with $\Psi(t)=nt$, by
the assumption that $\sH$ is symmetric and complete and $\sR$ is complete, we obtain $\inf_{h \in \sH} \E_y\bracket*{\ell_{\rm{\rho-\mathrm{hinge}}}(h,X, y) \mid X =x}=n\paren*{1 - \max_{y\in \sY}p(x, y)}$ and with the constraint $\sum_{y\in \sY}h(x, y) = 0$,
\begin{align*}
\sC_{\sfL}(h, r, x) & = \sum_{y\in \sY}p(x, y)\sum_{y'\neq y}\max\curl[\big]{0,1 + \frac{h(x, y')}{\rho}}e^{\alpha r(x)} + n c e^{-\alpha r(x)}\\
& = \sum_{y\in \sY} \paren*{1-p(x,y)}\max\curl*{0,1 + \frac{h(x, y)}{\rho}}e^{\alpha r(x)} + n c e^{-\alpha r(x)}\\
\sC^*_{\sfL}(\sH, \sR, x) &=2\sqrt{n^2c\paren*{1 - \max_{y\in \sY}p(x, y)}}.
\end{align*}
Take $h_{\rho}\in \sH$ such that $ h_{\rho}(x,y) = 
\begin{cases}
  h(x, y) & \text{if $y \not \in \curl*{y_{\max}, \hh(x)}$}\\
  -\rho & \text{if $y = \hh(x)$}\\
  h(x, y_{\max})+h(x,\hh(x))+\rho& \text{if $y = y_{\max}$}.
\end{cases}   $
with the constraint $\sum_{y\in \sY}h_{\rho}(x, y) = 0$, where $y_{\max}= \argmax_{y\in \sY}p(x,y)$.
Note that for any $h\in \sH$,
\begin{align*}
& \sum_{y\in \sY} \paren*{1-p(x,y)}\max\curl*{0,1 + \frac{h(x, y)}{\rho}} - n\paren*{1 - \max_{y\in \sY}p(x, y)}\\
& \geq  \sum_{y\in \sY} \paren*{1-p(x,y)}\min\curl*{n,\max\curl*{0,1 + \frac{h(x, y)}{\rho}}} - n\paren*{1 - \max_{y\in \sY}p(x, y)}\\
& \geq \sum_{y\in \sY} \paren*{1-p(x,y)}\min\curl*{n,\max\curl*{0,1 + \frac{h(x, y)}{\rho}}} - \sum_{y\in \sY} \paren*{1-p(x,y)}\min\curl*{n,\max\curl*{0,1 + \frac{ h_{\rho}(x, y)}{\rho}}}\\ 
& \geq \min\curl*{n,1+\frac{h(x,\hh(x))}{\rho}}\paren*{\max_{y\in \sY}p(x, y)-p(x,\hh(x))}
\tag{plug in $h_{\rho}(x,y)$}\\
& \geq \max_{y\in \sY}p(x, y)-p(x, \hh(x). \tag{$h(x,\hh(x))\geq 0$}
\end{align*}
We will then analyze the following four cases.
    \paragraph{(i)} $\max_{y\in \sY}p(x, y) > (1 - c)$ and $r(x)>0$. In this case, by \eqref{eq:calibration_gap_general},
      we have $\sC^*_{\labs}(\sH, \sR, x) =1 - \max_{y\in
        \sY}p(x, y)$ and $\Delta\sC_{\labs,\sH,
        \sR}(h, r, x) = \max_{y\in \sY}p(x, y)-p\paren*{x, \hh(x)}$. For
      the surrogate loss, we have
    \begin{align*}
    &\Delta\sC_{\sfL,\sH, \sR}(h, r, x)\\ 
    & = \sum_{y\in \sY} \paren*{1-p(x,y)}\max\curl*{0,1 + \frac{h(x, y)}{\rho}}e^{\alpha r(x)} + n c e^{-\alpha r(x)} - 2\sqrt{n^2c\paren*{1 - \max_{y\in \sY}p(x, y)}}\\
    & \geq \sum_{y\in \sY} \paren*{1-p(x,y)}\max\curl*{0,1 + \frac{h(x, y)}{\rho}}e^{\alpha r(x)} + n c e^{-\alpha r(x)} - n\paren*{1 - \max_{y\in \sY}p(x, y)}e^{\alpha r(x)} - n c e^{-\alpha r(x)}
    \tag{AM–GM inequality}\\
    & =\sum_{y\in \sY} \paren*{1-p(x,y)}\max\curl*{0,1 + \frac{h(x, y)}{\rho}} - n\paren*{1 - \max_{y\in \sY}p(x, y)}
    \tag{$r(x)>0$}\\
    & \geq \max_{y\in \sY}p(x, y)-p(x, \hh(x)\\
    \tag{$\sum_{y\in \sY} \paren*{1-p(x,y)}\max\curl*{0,1 + \frac{h(x, y)}{\rho}} - n\paren*{1 - \max_{y\in \sY}p(x, y)}\geq \max_{y\in \sY}p(x, y)-p(x, \hh(x)$}\\
    & = \Delta\sC_{\labs,\sH, \sR}(h, r, x).
    \end{align*}
     Therefore, 
    \begin{align*}
\sE_{\labs}(h, r) - \sE_{\labs}^*(\sH, \sR) + \sM_{\labs}(\sH, \sR) & = \mathbb{E}_{X}\bracket*{\Delta\sC_{\labs,\sH, \sR}(h, r, x)}\\
& \leq \mathbb{E}_{X}\bracket*{\Gamma_1\paren*{\Delta\sC_{\sfL,\sH, \sR}(h, r, x)}}\\
& \leq \Gamma_1\paren*{\mathbb{E}_{X}\bracket*{\Delta\sC_{\sfL,\sH, \sR}(h, r, x)}}
\tag{$\Gamma_1$ is concave}\\
& = \Gamma_1\paren*{\sE_{\sfL}(h, r)-\sE_{\sfL}^*(\sH, \sR) +\sM_{\sfL}(\sH, \sR)}
\end{align*}
where $\Gamma_1(t)= t$.
\paragraph{(ii)} $\max_{y\in \sY}p(x, y) \leq (1 - c)$ and $r(x)>0$. In this
      case, by \eqref{eq:calibration_gap_general}, we have $\sC^*_{\labs}(\sH, \sR, x) =c$ and
      $\Delta\sC_{\labs,\sH, \sR}(h, r, x) =1-c-p\paren*{x,
        \hh(x)}$. For the surrogate loss, we have
    \begin{align*}
    & \Delta\sC_{\sfL,\sH, \sR}(h, r, x)\\
    & =  \sum_{y\in \sY} \paren*{1-p(x,y)}\max\curl*{0,1 + \frac{h(x, y)}{\rho}}e^{\alpha r(x)} + n c e^{-\alpha r(x)} - 2\sqrt{n^2c\paren*{1 - \max_{y\in \sY}p(x, y)}}\\
    & \geq \sum_{y\in \sY} \paren*{1-p(x,y)}\max\curl*{0,1 + \frac{h(x, y)}{\rho}}e^{\alpha r(x)} + n c e^{-\alpha r(x)} - 2\sqrt{nc\paren*{\sum_{y\in \sY} \paren*{1-p(x,y)}\max\curl*{0,1 + \frac{h(x, y)}{\rho}}}}
   \tag{$\sum_{y\in \sY} \paren*{1-p(x,y)}\max\curl*{0,1 + \frac{h(x, y)}{\rho}}\geq n\paren*{1 - \max_{y\in \sY}p(x, y)}$}\\
    & \geq \sum_{y\in \sY} \paren*{1-p(x,y)}\max\curl*{0,1 + \frac{h(x, y)}{\rho}} + n c -  2\sqrt{nc\paren*{\sum_{y\in \sY} \paren*{1-p(x,y)}\max\curl*{0,1 + \frac{h(x, y)}{\rho}}}}
    \tag{increasing for $r(x)\geq 0$}\\
    & = \paren*{\sqrt{\sum_{y\in \sY} \paren*{1-p(x,y)}\max\curl*{0,1 + \frac{h(x, y)}{\rho}}}-\sqrt{nc}}^2\\
    &\geq \paren*{\sqrt{\sum_{y\in \sY} \paren*{1-p(x,y)}\min\curl*{n,\max\curl*{0,1 + \frac{h(x, y)}{\rho}}}}-\sqrt{nc}}^2\\
    & = \paren*{\frac{\sum_{y\in \sY} \paren*{1-p(x,y)}\min\curl*{n,\max\curl*{0,1 + \frac{h(x, y)}{\rho}}}-n c}{\sqrt{\sum_{y\in \sY} \paren*{1-p(x,y)}\min\curl*{n,\max\curl*{0,1 + \frac{h(x, y)}{\rho}}}}+\sqrt{nc}}}^2\\
    & \geq  \paren*{\frac{\sum_{y\in \sY} \paren*{1-p(x,y)}\min\curl*{n,\max\curl*{0,1 + \frac{h(x, y)}{\rho}}}-n\paren*{1 - \max_{y\in \sY}p(x, y)}+n\paren*{1 - \max_{y\in \sY}p(x, y)-c}}{2\sqrt{n}}}^2
    \tag{$\sqrt{\sum_{y\in \sY} \paren*{1-p(x,y)}\min\curl*{n,\max\curl*{0,1 + \frac{h(x, y)}{\rho}}}}+\sqrt{nc}\leq 2\sqrt{n}$}\\
    & \geq  \paren*{\frac{\max_{y\in \sY}p(x, y)-p(x, \hh(x)+1 - \max_{y\in \sY}p(x, y)-c}{2\sqrt{n}}}^2
    \tag{$\sum_{y\in \sY} \paren*{1-p(x,y)}\min\curl*{n,\max\curl*{0,1 + \frac{h(x, y)}{\rho}}} - n\paren*{1 - \max_{y\in \sY}p(x, y)}\geq \max_{y\in \sY}p(x, y)-p(x, \hh(x)$}\\
    & = \frac{1}{4n}\paren*{1-c-p\paren*{x,\hh(x)}}^2\\
    & = \frac{\Delta\sC_{\labs,\sH, \sR}(h, r, x)^2}{4n}
    \end{align*}
    Therefore, 
    \begin{align*}
\sE_{\labs}(h, r) - \sE_{\labs}^*(\sH, \sR) + \sM_{\labs}(\sH, \sR) & = \mathbb{E}_{X}\bracket*{\Delta\sC_{\labs,\sH, \sR}(h, r, x)}\\
& \leq \mathbb{E}_{X}\bracket*{\Gamma_2\paren*{\Delta\sC_{\sfL,\sH, \sR}(h, r, x)}}\\
& \leq \Gamma_2\paren*{\mathbb{E}_{X}\bracket*{\Delta\sC_{\sfL,\sH, \sR}(h, r, x)}}
\tag{$\Gamma_2$ is concave}\\
& = \Gamma_2\paren*{\sE_{\sfL}(h, r)-\sE_{\sfL}^*(\sH, \sR) +\sM_{\sfL}(\sH, \sR)}
\end{align*}
where $\Gamma_2(t)=2\sqrt{n t}$.
\paragraph{(iii)} $\max_{y\in \sY}p(x, y) \leq  (1 - c)$ and $r(x)\leq 0$. In this case, by \eqref{eq:calibration_gap_general}, we have $\sC^*_{\labs}(\sH, \sR, x) =c$ and $\Delta\sC_{\labs,\sH, \sR}(h, r, x)=0$, which implies that $\sE_{\labs}(h, r) - \sE_{\labs}^*(\sH, \sR) + \sM_{\labs}(\sH, \sR) = \mathbb{E}_{X}\bracket*{\Delta\sC_{\labs,\sH, \sR}(h, r, x)}=0\leq \Gamma\paren*{\sE_{\sfL}(h, r)-\sE_{\sfL}^*(\sH, \sR) +\sM_{\sfL}(\sH, \sR)}$ for any $\Gamma\geq 0$.
\paragraph{(iv)} $\max_{y\in \sY}p(x, y) > (1 - c)$ and $r(x)\leq 0$. In this
      case, by \eqref{eq:calibration_gap_general}, we have $\sC^*_{\labs}(\sH, \sR, x) =1 - \max_{y\in
        \sY}p(x, y)$ and $\Delta\sC_{\labs,\sH,
        \sR}(h, r, x) =\max_{y\in \sY} p(x, y)-1+c$. For the
      surrogate loss, we have
    \begin{align*}
    &\Delta\sC_{\sfL,\sH, \sR}(h, r, x)\\ 
    & = \sum_{y\in \sY} \paren*{1-p(x,y)}\max\curl*{0,1 + \frac{h(x, y)}{\rho}}e^{\alpha r(x)} + n c e^{-\alpha r(x)} - 2\sqrt{n^2c\paren*{1 - \max_{y\in \sY}p(x, y)}}\\
    & \geq n\paren*{1 - \max_{y\in \sY}p(x, y)}e^{\alpha r(x)} + nc e^{-\alpha r(x)} - 2\sqrt{n^2c\paren*{1 - \max_{y\in \sY}p(x, y)}}
    \tag{$\sum_{y\in \sY} \paren*{1-p(x,y)}\max\curl*{0,1 + \frac{h(x, y)}{\rho}}\geq n\paren*{1 - \max_{y\in \sY}p(x, y)}$}\\
    & \geq n\paren*{1 - \max_{y\in \sY}p(x, y)} + nc - 2\sqrt{n^2c\paren*{1 - \max_{y\in \sY}p(x, y)}}
    \tag{decreasing for $r(x)\leq 0$}\\
    & = n\paren*{\sqrt{1 - \max_{y\in \sY}p(x, y)}-\sqrt{c}}^2\\
    & = n\paren*{\frac{1 - \max_{y\in \sY}p(x, y)-c}{\sqrt{1 - \max_{y\in \sY}p(x, y)}+\sqrt{c}}}^2\\
    & \geq n\paren*{\frac{\max_{y\in \sY} p(x, y)-1+c}{2}}^2
    \tag{$\sqrt{1 - \max_{y\in \sY}p(x, y)}+\sqrt{c}\leq 2$}\\
    & = \frac{n\Delta\sC_{\labs,\sH, \sR}(h, r, x)^2}{4}.
    \end{align*}
     Therefore, 
    \begin{align*}
      \sE_{\labs}(h, r) - \sE_{\labs}^*(\sH, \sR) + \sM_{\labs}(\sH, \sR)
      & = \mathbb{E}_{X}\bracket*{\Delta\sC_{\labs,\sH, \sR}(h, r, x)}\\
      & \leq \mathbb{E}_{X}\bracket*{\Gamma_3\paren*{\Delta\sC_{\sfL,\sH, \sR}(h, r, x)}}\\
      & \leq \Gamma_3\paren*{\mathbb{E}_{X}\bracket*{\Delta\sC_{\sfL,\sH, \sR}(h, r, x)}}
      \tag{$\Gamma_3$ is concave}\\
& = \Gamma_3\paren*{\sE_{\sfL}(h, r)-\sE_{\sfL}^*(\sH, \sR) +\sM_{\sfL}(\sH, \sR)}
\end{align*}
where $\Gamma_3(t)=2\sqrt{t/n}$.

Overall, we obtain
\begin{align*}
\sE_{\labs}(h, r) - \sE_{\labs}^*(\sH, \sR) + \sM_{\labs}(\sH, \sR) \leq \Gamma\paren*{\sE_{\sfL}(h, r)-\sE_{\sfL}^*(\sH, \sR) +\sM_{\sfL}(\sH, \sR)}
\end{align*}
where $\Gamma(t)=\max\curl*{\Gamma_1(t),\Gamma_2(t),\Gamma_3(t)}= \max\curl*{2\sqrt{nt},t}$, which completes the proof.
\end{proof}

\subsection{Proof of \texorpdfstring{$\sR$}{R}-consistency bounds for second-stage surrogates (Theorem~\ref{Thm:bound-general-second-step})}
\label{app:general-positive-second-stage}
\BoundGenralSecondStep*
\begin{proof}
Given any fixed predictor $h$.  For any $r \in \sR$, $x \in \sX$ and $y \in \sY$, the conditional risk of $\ell_{\mathrm{abs}, h}$ and $\ell_{\Phi, h}$ can be written as
\begin{equation}
\label{eq:tsr-cond-error}
\begin{aligned}
\sC_{\ell_{\mathrm{abs}, h}}(r, x) &=  \sum_{y \in \sY} p(x, y) \1_{\hh(x) \neq y} \1_{r(x) > 0} + c \1_{r(x) \leq 0}\\
\sC_{\ell_{\Phi, h}}(r, x) &= \sum_{y \in \sY} p(x, y) \1_{\hh(x) \neq y} \Phi \paren*{-r(x)} + c \Phi \paren*{r(x)}.
\end{aligned}
\end{equation}
Thus, the best-in class conditional risk of $\ell_{\mathrm{abs}, h}$ and $\ell_{\Phi, h}$ can be expressed as
\begin{equation}
\label{eq:tsr-best-cond-error}
\begin{aligned}
\sC^*_{\ell_{\mathrm{abs}, h}}(\sR, x) &= \inf_{r \in \sR}\paren*{ \sum_{y \in \sY} p(x, y) \1_{\hh(x) \neq y} \1_{r(x) > 0} + c \1_{r(x) \leq 0}}\\
\sC^*_{\ell_{\Phi, h}}(\sR, x) &= \inf_{r \in \sR}\paren*{\sum_{y \in \sY} p(x, y) \1_{\hh(x) \neq y} \Phi \paren*{-r(x)} + c \Phi \paren*{r(x)}}.
\end{aligned}
\end{equation}
Let $p_1 = \frac{\sum_{y \in \sY} p(x, y) \1_{\hh(x) \neq y}}{\sum_{y \in \sY} p(x, y) \1_{\hh(x) \neq y} + c}$ and $p_2 = \frac{c}{\sum_{y \in \sY} p(x, y) \1_{\hh(x) \neq y} + c}$. Then, the calibration gap of $\ell_{\mathrm{abs}, h}$ can be written as 
\begin{align*}
& \sC_{\ell_{\Phi, h}}(r, x) - \sC^*_{\ell_{\Phi, h}}(\sR, x)\\
& = \paren*{\sum_{y \in \sY} p(x, y) \1_{\hh(x) \neq y} + c} \bracket*{p_1 \Phi(-r(x)) + p_2 \Phi(r(x)) - \inf_{r \in \sR} \paren*{p_1 \Phi(-r(x)) + p_2 \Phi(r(x))}}.
\end{align*}
By Lemma~\ref{lemma:aux}, we have
\begin{align*}
& \sC_{\ell_{\mathrm{abs}, h}}(r, x) - \sC^*_{\ell_{\mathrm{abs}, h}}(\sR, x)\\
& = p_1 1_{r(x) > 0} + p_2 1_{r(x) \leq 0} - \inf_{r \in \sR} \paren*{p_1 1_{r(x) > 0} +  p_2 1_{r(x) \leq 0}}\\
& \leq \Gamma \paren*{p_1 \Phi(-r(x)) + p_2 \Phi(r(x)) - \inf_{r \in \sR} \paren*{p_1 \Phi(-r(x)) + p_2 \Phi(r(x))}}\\
& = \Gamma\paren*{\frac{\sC_{\ell_{\Phi, h}}(r, x) - \inf_{r\in \sR}\sC_{\ell_{\Phi, h}}(r, x)}{\sum_{y \in \sY} p(x, y) \1_{\hh(x) \neq y} + c}}\\
& \leq \Gamma \paren*{\frac{\sC_{\ell_{\Phi, h}}(r, x) - \inf_{r\in \sR}\sC_{\ell_{\Phi, h}}(r, x)}{c}},
\end{align*}
where we use the fact that $\Gamma$ is non-decreasing and $\sum_{y \in \sY} p(x, y) \1_{\hh(x) \neq y} + c \geq c$ in the last inequality.
Since $\Gamma$ is concave, taking the expectation on both sides and using Jensen's inequality, we obtain 
\begin{equation*}
\E_{X}\bracket*{\sC_{\ell_{\mathrm{abs}, h}}(r, x) - \sC^*_{\ell_{\mathrm{abs}, h}}(\sR, x)} \leq \Gamma \paren*{\frac{\E_{X}\bracket*{\sC_{\ell_{\Phi, h}}(r, x) - \inf_{r\in \sR}\sC_{\ell_{\Phi, h}}(r, x)}}{c}}.
\end{equation*}
Since the term $\E_{X}\bracket*{\sC_{\ell_{\mathrm{abs}, h}}(r, x) - \sC^*_{\ell_{\mathrm{abs}, h}}(\sR, x)}$ and $\E_{X}\bracket*{\sC_{\ell_{\Phi, h}}(r, x) - \inf_{r\in \sR}\sC_{\ell_{\Phi, h}}(r, x)}$ can be expressed as
\begin{align*}
 \E_{X}\bracket*{\sC_{\ell_{\mathrm{abs}, h}}(r, x) - \sC^*_{\ell_{\mathrm{abs}, h}}(\sR, x)} &= \sE_{\ell_{\mathrm{abs}, h} }(r)-\sE_{\ell_{\mathrm{abs}, h} }^*(\sR) +\sM_{\ell_{\mathrm{abs}, h} }(\sR)\\
 \E_{X}\bracket*{\sC_{\ell_{\Phi, h}}(r, x) - \inf_{r\in \sR}\sC_{\ell_{\Phi, h}}(r, x)} &= \sE_{\ell_{\Phi, h} }(r)-\sE_{\ell_{\Phi, h} }^*(\sR) +\sM_{\ell_{\Phi, h}}(\sR),
\end{align*}
we have
\begin{equation*}
\sE_{\ell_{\mathrm{abs}, h} }(r)-\sE_{\ell_{\mathrm{abs}, h} }^*(\sR) +\sM_{\ell_{\mathrm{abs}, h} }(\sR) \leq \Gamma\paren*{\frac{\sE_{\ell_{\Phi, h} }(r)-\sE_{\ell_{\Phi, h} }^*(\sR) +\sM_{\ell_{\Phi, h}}(\sR)}{c}},
\end{equation*}
which completes the proof.
\end{proof}

\subsection{Proof of \texorpdfstring{$(\sH, \sR)$}{HR}-consistency bounds for two-stage surrogates (Theorem~\ref{Thm:bound-general-two-step})}
\label{app:general-positive-two-stage}
\BoundGenralTwoStep*
\begin{proof}
Since $\sR$ is regular, the conditional risk and the best-in-class conditional risk of the abstention loss $\labs$ can be expressed as
\begin{equation}
\label{eq:tshr-cond-error-def}
\begin{aligned}
\sC_{\labs}(h, r, x) &=  \sum_{y \in \sY} p(x, y) \1_{\hh(x) \neq y} \1_{r(x) > 0} + c \1_{r(x) \leq 0}\\
\sC^*_{\labs}(\sH, \sR, x) &= \min\curl*{\inf_{h \in \sH} \sum_{y \in \sY} p(x, y) \1_{\hh(x) \neq y}, c}.
\end{aligned}
\end{equation}
Thus, by introducing the term $\min\curl*{\sum_{y \in \sY} p(x, y) \1_{\hh(x) \neq y}, c}$ and subsequently subtracting it after rearranging, the calibration gap of the abstention loss $\labs$ can be written as follows
\begin{equation}
\label{eq:tshr-cond-reg-def}
\begin{aligned}
& \sC_{\labs}(h, r, x) - \sC^*_{\labs}(\sH, \sR, x)\\
& = \sum_{y \in \sY} p(x, y) \1_{\hh(x) \neq y} \1_{r(x) > 0} + c \1_{r(x) \leq 0} - \min\curl*{\inf_{h \in \sH} \sum_{y \in \sY} p(x, y) \1_{\hh(x) \neq y}, c}\\ 
& =  \sum_{y \in \sY} p(x, y) \1_{\hh(x) \neq y} \1_{r(x) > 0} + c \1_{r(x) \leq 0}  - \min\curl*{\sum_{y \in \sY} p(x, y) \1_{\hh(x) \neq y}, c}\\
& \quad + \min\curl*{\sum_{y \in \sY} p(x, y) \1_{\hh(x) \neq y}, c} - \min\curl*{\inf_{h \in \sH} \sum_{y \in \sY} p(x, y) \1_{\hh(x) \neq y}, c}.
\end{aligned}
\end{equation}
Note that by the property of the minimum, the second term can be upper bounded as 
\begin{align*}
& \min\curl*{\sum_{y \in \sY} p(x, y) \1_{\hh(x) \neq y}, c} - \min\curl*{\inf_{h \in \sH} \sum_{y \in \sY} p(x, y) \1_{\hh(x) \neq y}, c}\\
& \leq \sum_{y \in \sY} p(x, y) \1_{\hh(x) \neq y} - \inf_{h \in \sH} \sum_{y \in \sY} p(x, y) \1_{\hh(x) \neq y}\\
& = \sC_{\ell_{0-1}}(h, x)-\sC^*_{\ell_{0-1}}(\sH,x)\\
& \leq \Gamma_1\paren*{\sC_{\ell}(h, x)-\sC^*_{\ell}(\sH,x)},
\end{align*}
where we use the $\sH$-consistency bound of $\ell$ on the pointwise distribution $\delta_{x}$ that concentrates on a point $x$ in the last inequality.
Next, we will upper bound the first term. Note that the conditional risk and the best-in class conditional risk of $\ell_{\Phi, h}$ can be expressed as
\begin{equation}
\label{eq:tshr-cond-error-sur}
\begin{aligned}
\sC_{\ell_{\Phi, h}}(r, x) &= \sum_{y \in \sY} p(x, y) \1_{\hh(x) \neq y} \Phi \paren*{-r(x)} + c \Phi \paren*{r(x)}\\
\sC^*_{\ell_{\Phi, h}}(\sR, x) &= \inf_{r \in \sR}\paren*{\sum_{y \in \sY} p(x, y) \1_{\hh(x) \neq y} \Phi \paren*{-r(x)} + c \Phi \paren*{r(x)}}.
\end{aligned}
\end{equation}
Let $p_1 = \frac{\sum_{y \in \sY} p(x, y) \1_{\hh(x) \neq y}}{\sum_{y \in \sY} p(x, y) \1_{\hh(x) \neq y} + c}$ and $p_2 = \frac{c}{\sum_{y \in \sY} p(x, y) \1_{\hh(x) \neq y} + c}$. Then, the first term can be rewritten as 
\begin{align*}
& \sum_{y \in \sY} p(x, y) \1_{\hh(x) \neq y} \1_{r(x) > 0} + c \1_{r(x) \leq 0}  - \min\curl*{\sum_{y \in \sY} p(x, y) \1_{\hh(x) \neq y}, c}\\
& = \paren*{\sum_{y \in \sY} p(x, y) \1_{\hh(x) \neq y} + c} \bracket*{p_1 1_{r(x) > 0} + p_2 1_{r(x) \leq 0} - \inf_{r \in \sR} \paren*{p_1 1_{r(x) > 0} +  p_2 1_{r(x) \leq 0}}}
\end{align*}
By Lemma~\ref{lemma:aux}, we have
\begin{align*}
& p_1 1_{r(x) > 0} + p_2 1_{r(x) \leq 0} - \inf_{r \in \sR} \paren*{p_1 1_{r(x) > 0} + p_2 1_{r(x) \leq 0}}\\
& \leq \Gamma_2 \paren*{p_1 \Phi(-r(x)) + p_2 \Phi(r(x)) - \inf_{r \in \sR} \paren*{p_1 \Phi(-r(x)) + p_2 \Phi(r(x))}}\\
& = \Gamma_2 \paren*{\frac{\sC_{\ell_{\Phi, h}}(r, x) - \sC^*_{\ell_{\Phi, h}}(\sR, x)}{\sum_{y \in \sY} p(x, y) \1_{\hh(x) \neq y} + c}}.
\end{align*}
Therefore, the first term can be upper bounded as
\begin{align*}
& \sum_{y \in \sY} p(x, y) \1_{\hh(x) \neq y} \1_{r(x) > 0} + c \1_{r(x) \leq 0}  - \min\curl*{\sum_{y \in \sY} p(x, y) \1_{\hh(x) \neq y}, c}\\
& = \paren*{\sum_{y \in \sY} p(x, y) \1_{\hh(x) \neq y} + c} \bracket*{p_1 1_{r(x) > 0} + p_2 1_{r(x) \leq 0} - \inf_{r \in \sR} \paren*{p_1 1_{r(x) > 0} +  p_2 1_{r(x) \leq 0}}}\\
& \leq \paren*{\sum_{y \in \sY} p(x, y) \1_{\hh(x) \neq y} + c} \Gamma_2 \paren*{\frac{\sC_{\ell_{\Phi, h}}(r, x) - \sC^*_{\ell_{\Phi, h}}(\sR, x)}{\sum_{y \in \sY} p(x, y) \1_{\hh(x) \neq y} + c}} \\
& \leq
\begin{cases}
\Gamma_2\paren*{\sC_{\ell_{\Phi,h}}(r,x)-\sC^*_{\ell_{\Phi,h}}(\sR,x)} & \text{when $\Gamma_2$ is linear}\\
(1+c)\Gamma_2\paren*{\frac {\sC_{\ell_{\Phi,h}}(r,x)-\sC^*_{\ell_{\Phi,h}}(\sR,x)}{c}} & \text{otherwise}
\end{cases}
\end{align*}
where we use the fact that $c\leq \sum_{y\in \sY}p(x, y)\1_{\hh(x)\neq y} + c\leq 1+c$ and $\Gamma_2$ is non-decreasing in the last inequality. After upper bounding the first term and the second term in \eqref{eq:tshr-cond-reg-def} as above, taking the expectation on both sides, using the fact that $\Gamma_1$ and $\Gamma_2$ are concave, we obtain
\begin{align*}
  &\E_{X}\bracket*{\sC_{\labs}(h, r, x) - \sC^*_{\labs}(\sH, \sR, x)}\\
  &\leq 
\begin{cases}
\Gamma_2\paren*{\E_X\bracket*{\sC_{\ell_{\Phi,h}}(r,x)-\sC^*_{\ell_{\Phi,h}}(\sR,x)}} + \Gamma_1\paren*{\E_X\bracket*{\sC_{\ell}(h, x)-\sC^*_{\ell}(\sH,x)}} & \text{$\Gamma_2$ is linear}\\
(1+c)\Gamma_2\paren*{\frac1c\E_X\bracket*{\sC_{\ell_{\Phi,h}}(r,x)-\sC^*_{\ell_{\Phi,h}}(\sR,x))}} + \Gamma_1\paren*{\E_X\bracket*{\sC_{\ell}(h, x)-\sC^*_{\ell}(\sH,x)}} & \text{otherwise}
\end{cases}
\end{align*}
Since the three expected terms can be expressed as
\begin{align*}
\E_{X}\bracket*{\sC_{\labs}(h, r, x) - \sC^*_{\labs}(\sH, \sR, x)} &= \sE_{\labs}(h, r)-\sE^*_{\labs}\paren*{\sH,\sR}+\sM_{\labs}(\sH, \sR)\\
\E_X\bracket*{\sC_{\ell_{\Phi,h}}(r,x)-\sC^*_{\ell_{\Phi,h}}(\sR,x)} &= \sE_{\ell_{\Phi,h}}(r)-\sE_{\ell_{\Phi,h}}^*(\sR) +\sM_{\ell_{\Phi,h}}(\sR)\\
\E_X\bracket*{\sC_{\ell}(h, x)-\sC^*_{\ell}(\sH,x)} &= \sE_{\ell}(h)-\sE_{\ell}^*(\sH) +\sM_{\ell}(\sH),
\end{align*}
we have
\begin{align*}
\sE_{\labs}(h, r) - \sE_{\labs}^*(\sH, \sR) + \sM_{\labs}(\sH, \sR)
& \leq \Gamma_1\paren*{\sE_{\ell}(h)-\sE_{\ell}^*(\sH) +\sM_{\ell}(\sH)}\\
& \quad + (1+c)\Gamma_2\paren*{\frac{\sE_{\ell_{\Phi,h}}(r)-\sE_{\ell_{\Phi,h}}^*(\sR) +\sM_{\ell_{\Phi,h}}(\sR)}{c}},
\end{align*}
where the
constant factors $(1 + c)$ and $\frac{1}{c}$ can be removed 
when $\Gamma_2$ is linear.
\end{proof}

\subsection{Proof of realizable \texorpdfstring{$(\sH, \sR)$}{HR}-consistency bounds for single-stage surrogates (Theorem~\ref{Thm:spcific-loss-bound-realizable})}
\label{app:general-positive-single-stage-realizable}
\SpecificLossBoundRealizable*
\begin{proof}
It is straightforward to see that $\sfL$ serves as an upper bound for $\labs$ when $\ell$ serves as an upper bound for $\ell_{0-1}$ under Assumption~\ref{assumption:phi}.
By definition, for any $(\sH,\sR)$-realizable distribution, there exists $h^*\in \sH$ and $r^*\in \sR$ such that $\sE_{\labs}(h^*,r^*) = \sE_{\labs}^*(\sH, \sR) = 0$. Then, by the assumption that $\sH$ and $\sR$ are closed under scaling, for any $\nu>0$,
\begin{align*}
\sE^*_{\sfL}(\sH,\sR)
&\leq\sE_{\sfL}(\nu h^*,\nu r^*)\\
&=\mathbb{E}\bracket*{\sfL(\nu h^*,\nu r^*,x,y)\mid r^*< 0}\mathbb{P}(r^*< 0) + \mathbb{E}\bracket*{\sfL(\nu h^*,\nu r^*,x,y)\mid r^*>0}\mathbb{P}(r^*> 0)
\end{align*}
Next, we investigate the two terms.
The first term is when $r^*< 0$, then we must have $c=0$ since the data is realizable. By taking the limit, we obtain:
\begin{align*}
&\lim_{\nu\to \plus\infty}\mathbb{E}\bracket*{\sfL(\nu h^*,\nu r^*,x,y)\mid r^*< 0}\mathbb{P}(r^*< 0)\\
&=\lim_{\nu\to \plus\infty}\mathbb{E}\bracket*{\ell(\nu h^*, x, y)\Phi\paren*{-\alpha \nu r^*(x)} + \Psi(c) \Phi\paren*{\beta \nu r^*(x)}\mid r^*< 0}\mathbb{P}(r^*< 0)\\
&=\lim_{\nu\to \plus\infty}\mathbb{E}\bracket*{\ell(\nu h^*, x, y)\Phi\paren*{-\alpha \nu r^*(x)}\mid r^*< 0}\mathbb{P}(r^*< 0) \tag{$c=0$ and $\Psi(0)=0$}\\
&=0. \tag{by the Lebesgue dominated convergence theorem and $\lim_{t\to \plus \infty}\Phi(t)=0$}
\end{align*}
The second term is when $r^*> 0$, then we must have $h^*(x,y) - \max_{y'\neq y}h^*(x,y') > 0$ since the data is realizable. Thus, using the fact that $\lim_{\nu\to \plus\infty}\ell(\nu h^*,x,y)=0$ and taking the limit, we obtain
\begin{align*}
&\lim_{\nu\to \plus\infty}\mathbb{E}\bracket*{\sfL(\nu h^*,\nu r^*,x,y)\mid r^*< 0}\mathbb{P}(r^*< 0)\\
&=\lim_{\nu\to \plus\infty}\mathbb{E}\bracket*{\ell(\nu h^*, x, y)\Phi\paren*{-\alpha \nu r^*(x)} + \Psi(c) \Phi\paren*{\beta \nu r^*(x)}\mid r^*< 0}\mathbb{P}(r^*< 0)\\
&=0. \tag{by the Lebesgue dominated convergence theorem, $\lim_{t\to \plus \infty}\Phi(t)=0$, $\lim_{\nu\to \plus\infty}\ell(\nu h^*,x,y)=0$}
\end{align*}
Therefore, by combining the above two analysis, we obtain
\begin{align*}
\sE^*_{\sfL}(\sH,\sR)\leq \lim_{\nu\to \plus\infty}\sE_{\sfL}(\nu h^*,\nu r^*)=0.
\end{align*}
By using the fact that $\sfL$ serves as an upper bound for $\labs$ and $\sE_{\labs}^*(\sH, \sR)=0$, we conclude that
\begin{equation*}
\sE_{\labs}(h, r) - \sE_{\labs}^*(\sH, \sR)
\leq \sE_{\sfL}(h, r)-\sE_{\sfL}^*(\sH, \sR).
\end{equation*}
\end{proof}
\subsection{Proof of realizable \texorpdfstring{$(\sH, \sR)$}{HR}-consistency for two-stage surrogates (Theorem~\ref{Thm:bound-general-two-step-realizable})}
\label{app:general-positive-two-stage-realizable}
\BoundGenralTwoStepRealizable*
\begin{proof}
It is straightforward to see that $\ell_{\Phi, h}$ upper bounds the abstention loss $\labs$ under Assumption~\ref{assumption:phi}. By definition, for any $(\sH,\sR)$-realizable distribution, there exists $h^*\in \sH$ and $r^*\in \sR$ such that $\sE_{\labs}(h^*,r^*)=0$. Let $\hat h$ be the minimizer of $\sE_{\ell}$ and $\hat r$ be the minimizer of $\sE_{\ell_{\Phi, \hat h}}$.
Then, using the fact that $\ell_{\Phi, h}$ upper bounds the abstention loss $\sE_{\labs}$, we have $\sE_{\sE_{\labs}}(\hat h, \hat r)\leq \sE_{\ell_{\Phi, \hat h}}(\hat r)$.

Next, we analyze two cases. If for a point $x$, abstention happens, that is $r^*(x) < 0$, then we must have $c = 0$ since the data is realizable. Therefore, there exists an optimal $r^{**}$ abstaining all the points with zero cost: $r^{**}(x) < 0$ for all $x \in \sX$. Then, by the assumption that $\sR$ is closed under scaling and the Lebesgue dominated convergence theorem, using the fact that $\lim_{t \to \plus
  \infty} \Phi(t) = 0$, we obtain
\begin{align*}
\sE_{\labs}(\hat h, \hat r) 
&\leq \sE_{\ell_{\Phi, \hat h}}(\hat r)\\
&\leq \lim_{\nu \to +\infty} \sE_{\ell_{\Phi, \hat h}}(\nu r^{**}) \tag{$\hat r$ is the minimizer of $\sE_{\ell_{\Phi, \hat h}}$}\\
&= \lim_{\nu\to \plus\infty}\mathbb{E}\bracket*{\1_{\hat \hh(x) \neq y} \Phi\paren*{-\nu r^{**}(x)} + c \Phi\paren*{\nu r^{**}(x)}} \tag{By \eqref{eq:ell-Phi-h}}\\
&= \lim_{\nu\to \plus\infty}\mathbb{E}\bracket*{\1_{\hat \hh(x) \neq y} \Phi\paren*{-\nu r^{**}(x)}} \tag{$c=0$}\\
&= 0.  \tag{by the Lebesgue dominated convergence theorem and $\lim_{t\to \plus \infty}\Phi(t)=0$}
\end{align*}
On the other hand, if no abstention occurs for any point, that is $r^*(x) > 0$ for any $x \in \sX$, then we must have $\1_{\mathsf h^*(x) \neq y} = 0$ for all $(x,y ) \in \sX \times \sY$ since the data is realizable. Using the fact that $\ell$ is realizable $\sH$-consistent with respect to $\ell_{0-1}$ when $\sH$ is closed under scaling, we obtain $\1_{\hat{\mathsf h}(x) \neq y} = 0$ for all $(x,y) \in \sX \times \sY$. Then, by the assumption that $\sR$ is closed under scaling and the Lebesgue dominated convergence theorem, using the fact that $\lim_{t \to \plus \infty} \Phi(t) = 0$, we obtain
\begin{align*}
\sE_{\labs}(\hat h, \hat r) 
&\leq \sE_{\ell_{\Phi, \hat h}}(\hat r)\\
&\leq \lim_{\nu \to +\infty} \sE_{\ell_{\Phi, \hat h}}(\nu r^{*}) \tag{$\hat r$ is the minimizer of $\sE_{\ell_{\Phi, \hat h}}$}\\
&= \lim_{\nu\to \plus\infty}\mathbb{E}\bracket*{\1_{\hat \hh(x) \neq y} \Phi\paren*{-\nu r^{*}(x)} + c \Phi\paren*{\nu r^{*}(x)}} \tag{By \eqref{eq:ell-Phi-h}}\\
&= \lim_{\nu\to \plus\infty}\mathbb{E}\bracket*{c \Phi\paren*{\nu r^{*}(x)}} \tag{$\1_{\hat{\mathsf h}(x) \neq y} = 0$}\\
&= 0.  \tag{by the Lebesgue dominated convergence theorem and $\lim_{t\to \plus \infty}\Phi(t)=0$}
\end{align*}
This completes the proof.
\end{proof}
\end{document}